\documentclass[twoside,11pt]{article}

\usepackage{jmlr2e}
\usepackage{theorem}
\usepackage{stmaryrd} %
\usepackage{epsfig} %
\usepackage{misc}
\usepackage{amsmath,amssymb,amsbsy} 
\usepackage{bm}
\usepackage{bbm}
\usepackage{proof}
\usepackage{setspace}

\usepackage{boxedminipage}
\usepackage{pifont}
\usepackage{graphicx}
\usepackage{caption}
\usepackage{multirow}
\usepackage{xspace} %
\usepackage{color}
\usepackage{url}
\usepackage[svgnames]{xcolor}
\usepackage{tikz}
\usetikzlibrary{decorations.markings}
\usetikzlibrary{shapes.geometric}
\pgfdeclarelayer{edgelayer}
\pgfdeclarelayer{nodelayer}
\pgfsetlayers{edgelayer,nodelayer,main}

\tikzstyle{none}=[inner sep=0pt]

\tikzstyle{rn}=[circle,fill=Red,draw=Black,line width=0.8 pt]
\tikzstyle{gn}=[circle,fill=White,draw=Black,line width=0.8 pt]
\tikzstyle{yn}=[circle,fill=Yellow,draw=Black,line width=0.8 pt]
\tikzstyle{simple}=[circle,fill=White,draw=Black]
\tikzstyle{newstyle1}=[circle,fill=Black,draw=Black,line width=0.3 pt,inner sep=0pt]
\tikzstyle{simple2}=[-,dashed,draw=Black]
\tikzstyle{simpledotted}=[-,dotted,draw=Black]
\tikzstyle{simple}=[-,draw=Black,line width=2.000]
\tikzstyle{arrow}=[-,draw=Black,postaction={decorate},decoration={markings,mark=at position .5 with {\arrow{>}}},line width=2.000]
\tikzstyle{tick}=[-,draw=Black,postaction={decorate},decoration={markings,mark=at position .5 with {\draw (0,-0.1) -- (0,0.1);}},line width=2.000]
\tikzstyle{newstyle2}=[-latex,draw=Black]
\tikzstyle{newstyle3}=[->,dotted,draw=Black]
\tikzstyle{newstyle6}=[->,dotted,draw=Black]

\usepackage{enumerate,caption}
\usepackage{enumitem}
\usepackage{algorithm}
\usepackage{algpseudocode}
\usepackage{scalefnt}

\theoremstyle{plain}
\newtheorem{theoremaux}{Theorem}%
\renewenvironment{theorem}{\begin{theoremaux}}{\end{theoremaux}\parfillskip=0pt\finalhyphendemerits=0\goodbreak\noindent\ignorespacesafterend}
\newtheorem{lemmaaux}[theoremaux]{Lemma}
\renewenvironment{lemma}{\begin{lemmaaux}}{\end{lemmaaux}\parfillskip=0pt\finalhyphendemerits=0\goodbreak\noindent\ignorespacesafterend}
\newtheorem{corollaryaux}[theoremaux]{Corollary}

\newtheorem{propositionaux}[theoremaux]{Proposition}

\newtheorem{exampleaux}[theoremaux]{Example}
\renewenvironment{example}{\begin{exampleaux}}{\end{exampleaux}\parfillskip=0pt\finalhyphendemerits=0\goodbreak\noindent\ignorespacesafterend}
\newtheorem{observationaux}[theoremaux]{Observation}

\newtheorem{claimaux}[theoremaux]{Claim}

\newtheorem{factaux}[theoremaux]{Fact}

\newtheorem{conjectureaux}[theoremaux]{Conjecture}

\theoremstyle{definition}
\newtheorem{definitionaux}[theoremaux]{Definition}
\renewenvironment{definition}{\begin{definitionaux}}{\end{definitionaux}\parfillskip=0pt\finalhyphendemerits=0\goodbreak\noindent\ignorespacesafterend}
\newcommand{\sma}{{\ensuremath{{\boldsymbol{\sigma}}}}}

\newcommand{\esiblingrule}{`Sibling merging for \Tmc'\xspace} 

\newcommand{\eleftcondition}{left saturated for \Hmc} 
\newcommand{\erightcondition}{concept saturated for \Tmc} 
\newcommand{\erolecondition}{role saturated for \Tmc}

\DeclareMathAlphabet\mathbfcal{OMS}{cmsy}{b}{n}

\usepackage{tikzsymbols}
\newcommand{\symbola}{d_2}
\newcommand{\symbolaa}{e_3}
\newcommand{\symbolb}{d_4}
\newcommand{\symbolbb}{e_2}
\newcommand{\symbolcc}{e_1}
\newcommand{\symbolc}{d_5}
\newcommand{\symbold}{d_3}
\newcommand{\symboldd}{e_4}
\newcommand{\symbole}{d_0}
\newcommand{\symbolee}{e_0}
\newcommand{\symbolf}{d_1}
\newcommand{\symbolg}{e_5}

\newcommand{\lhs}{{\sf lhs}}

\ShortHeadings{Exact Learning of Lightweight Description Logic Ontologies}{Boris Konev, Carsten Lutz, Ana Ozaki and Frank Wolter}
\firstpageno{1}

\begin{document}

\title{Exact Learning of Lightweight Description Logic Ontologies}

\author{\name Boris Konev \email konev@liverpool.ac.uk \\
       \addr Department of Computer Science\\
       University of Liverpool, United Kingdom
       \AND
       \name Carsten Lutz \email clu@informatik.uni-bremen.de \\
       \addr Department of Computer Science\\
       University of Bremen, Germany
       \AND
       \name Ana Ozaki \email anaozaki@liverpool.ac.uk \\
       \addr Department of Computer Science\\
       University of Liverpool, United Kingdom
       \AND
       \name Frank Wolter \email wolter@liverpool.ac.uk \\
       \addr Department of Computer Science\\
       University of Liverpool, United Kingdom
       }

\editor{}

\maketitle

\begin{abstract} %
  We study the problem of learning description logic (DL) ontologies
  in Angluin et al.'s framework of exact learning via queries. We
  admit membership queries (``is a given subsumption entailed by the
  target ontology?'') and equivalence queries (``is a given ontology
  equivalent to the target ontology?''). %
  We present three main results: (1)~ontologies formulated in (two
  relevant versions of) the description logic DL-Lite can be learned
  with polynomially many queries of polynomial size; (2)~this is not
  the case for ontologies formulated in the description logic \EL,
  even when only acyclic ontologies are admitted; and (3)~ontologies
  formulated in a fragment of \EL related to the web ontology language
  OWL 2 RL can be learned in polynomial time. We also show that
  neither membership nor equivalence queries alone are sufficient in
  cases (1) and (3).
\end{abstract}


\section{Introduction}

In many subfields of artificial intelligence, ontologies are used to
provide a common vocabulary for the application domain of interest and
to give a meaning to the terms in the vocabulary, and to describe the
relations between them.  Description logics (DLs) are a prominent
family of ontology languages with a long history that goes back to
Brachman's famous knowledge representation system KL-ONE in the early
1980s \citep{DBLP:journals/cogsci/BrachmanS85}. Today, there are
several widely used families of DLs that differ in expressive power,
computational complexity, and intended application.  The most
important ones are the \ALC family which aims at high expressive
power, the \EL family \citep{BaBrLu-IJCAI-05} which aims to provide
scalable reasoning, and the DL-Lite family
\citep{CDLLR07,DBLP:journals/jair/ArtaleCKZ09} which is tailored
specifically towards applications in data access. In 2004, the World
Wide Web Committee (W3C) has standardised a DL of the \ALC family as
an ontology language for the web, called OWL. The standard was updated
to OWL~2 in 2009, and since then comprises a family of five languages
including the OWL~2 profiles OWL~2 EL, OWL~2 QL, and OWL~2 RL.  While
OWL~2 EL is based on $\mathcal{EL}$ and OWL~2 QL on DL-Lite, OWL~2 RL
is closely related to the fragment of \EL that is obtained by allowing
only concept names on the right-hand side of concept inclusions.  In
this paper we study DLs from the $\mathcal{EL}$ and DL-Lite
families. %

Designing an ontology for an application domain is a subtle,
error-prone, and time consuming task.
From its beginnings, DL research was driven by the aim to provide
various forms of support for ontology engineers, assisting them in the
design of high-quality ontologies; examples include the ubiquitous
task of ontology classification~\citep{Textbook}, reasoning
support for debugging ontologies~\citep{wang2005debugging,schlobach2007debugging}, support for modular
ontology design~\citep{DBLP:series/lncs/5445}, and
checking the completeness of the modelling in a systematic
way~\citep{DBLP:conf/ijcai/BaaderGSS07}. The same aim is pursued by
the field of ontology learning, where the goal is to use machine
learning techniques
for various ontology engineering tasks such as to identify the
relevant vocabulary of the application
domain~\citep{DBLP:reference/nlp/CimianoVB10,Wong:2012:OLT:2333112.2333115},
to learn an initial version of the ontology that is then refined
manually~\citep{BorDi11,DBLP:conf/aime/MaD13,bootstrapping}, and to
learn concept expressions as building blocks of an
ontology~\citep{DBLP:journals/ml/LehmannH10}; see the recent
collection~\citep{lehmann2014perspectives} and the related work
section at the end of the paper for details.

In this paper we concentrate on learning the full logical structure
of a description logic ontology.
Our starting point is the observation that building
a high-quality ontology relies on the successful communication between
an ontology engineer and a domain expert because the former is
typically not sufficiently familiar with the domain and the latter is
rarely an expert in ontology engineering. We study the foundations of
this communication process in terms of a simple communication model
and analyse, within this model, the complexity of constructing a
correct and complete domain ontology.  Our model rests on the
following assumptions:
\begin{enumerate}
\item The domain expert has perfect knowledge of the domain, but is
  not able to formalise or communicate the target ontology \Omc to be
  constructed.
\item The domain expert is able to communicate the vocabulary
  (predicate symbols, which in the case of DLs take the form of
  concept and role names) of \Omc and shares it with the ontology
  engineer. %
  The ontology engineer knows nothing else
  about the domain.
\item The ontology engineer can pose queries to the domain expert
  which the domain expert answers truthfully. The main queries posed
  by the ontology engineer are of the form 
  \begin{quote}
``Is the concept inclusion
  $C\sqsubseteq D$ entailed by \Omc?''
  \end{quote}

\item In addition, the ontology engineer needs a way to find out
  whether the ontology \Hmc  that has been constructed so far, called the
  hypothesis ontology, is complete. If not, he
  requests an example illustrating the incompleteness. The engineer can
  thus ask:
  \begin{quote}
    ``Is the ontology \Hmc complete? If not, then return a
    concept inclusion $C\sqsubseteq D$ entailed by \Omc but not by \Hmc.''
  \end{quote}

\end{enumerate}
We are then interested in whether the target ontology \Omc can be
constructed with only polynomially many queries of polynomial size
(polynomial query learnability) or, even better, with overall
polynomial time (polynomial time learnability). In both cases, the
polynomial is in the size of the ontology to be
constructed %
plus the size of the counterexamples returned by the domain
expert. Without taking into account the latter, one can never expect
to achieve polynomial time learnability because the domain expert
could %
provide unnecessarily large counterexamples. Note that polynomial time
learnability implies polynomial query learnability, but that the converse is
false because polynomial query learnability allows the ontology engineer to run
computationally costly procedures between posing queries.

The above model is an instance of Angluin et al.'s framework of exact
learning via queries~\citep{DBLP:journals/ml/Angluin87}.  In
this context, the 
queries mentioned in Point~3 above are called \emph{membership
  queries}. The queries in Point~4 are a form of \emph{equivalence
  queries}. In Angluin's framework, however, such queries are slightly more
general:
  \begin{quote}
    ``Is the hypothesis ontology \Hmc equivalent to the target
    ontology \Omc? If not, then return a concept inclusion $C\sqsubseteq D$
    entailed by \Omc but not by \Hmc  (a \emph{positive
       counterexample}) or vice versa   (a \emph{negative counterexample}).''
  \end{quote}
  In our upper bounds (that is, polynomial learnability results), we
  admit only queries of the more restricted form in Point~4 above: the
  learning algorithm is designed in a way so that the hypothesis
  ontology \Hmc is a consequence of the target ontology \Omc at all
  times, and thus the only meaningful equivalence query is a query of
  the form ``Is \Hmc already complete?''. Our lower bounds (results
  saying that polynomial learnability is impossible), in contrast,
  apply to unrestricted equivalence queries, that is, they do not
  assume that the hypothesis is implied by the target. In this way, we
  achieve maximum generality.

\smallskip

Within the setup outlined above, we study the following description logics: 
\begin{enumerate}

\item[(a)] $\ourDLLite$, which is a member of the DL-Lite family that
  admits role inclusions and allows nested existential quantification
  on the right-hand side of concept inclusions;

\item[(b)] the extension $\ourDLLitehorn$ of $\ourDLLite$ with
  conjunction on the left-hand side of concept inclusions;

\item[(c)] the basic member \EL of the \EL family;

\item[(d)] the fragment $\mathcal{EL}_{\mn{lhs}}$ of $\mathcal{EL}$
  where only concept names (but no compound concept expressions) are
  admitted on the right-hand side of concept inclusions.

\end{enumerate}
We remark that $\ourDLLite$ is closely related to OWL~2 QL, which is
based on the fragment of $\ourDLLite$ that does not allow nested
existential quantification on the right-hand side of concept
inclusions. In this more restricted case, though, polynomial
learnability is uninteresting. In fact, the number of concept
inclusions formulated in a fixed finite vocabulary $\Sigma$ is bounded
polynomially in the size of~$\Sigma$ instead of being infinite as in
the description logics studied in this paper; consequently, TBoxes are
trivially learnable in polynomial time, even when only membership
queries (but no equivalence queries) are available or vice versa.  The
extension $\ourDLLitehorn$ of $\ourDLLite$ is not part of the OWL~2 QL
standard, but admitting conjunctions on the left-hand side of concept
inclusions is a useful and widely considered extension of basic
DL-Lite dialects, see for example
\citep{DBLP:journals/jair/ArtaleCKZ09}. $\mathcal{EL}_{\mn{lhs}}$ is a
significant part of the OWL~2 RL language and can be viewed as a
natural fragment of Datalog. An even better approximation of OWL~2 RL
would be the extension of $\mathcal{EL}_{\mn{lhs}}$ with inverse
roles, but polynomial learnability in that language remains an open
problem. And finally, unrestricted \EL can be viewed as a logical core
of the OWL~2 EL language.

After introducing preliminaries in Section~\ref{sect:prelim}, we study
exact learning of $\ourDLLite$ ontologies in Section~\ref{sec:dllite},
establishing polynomial query learnability. We strengthen this result
to \ourDLLitehorn in Section~\ref{sec:elih-rhs-horn-subsumption},
using a significantly more subtle algorithm. It remains open whether
$\ourDLLite$ and $\ourDLLitehorn$ admit polynomial time learnability.
Our algorithms do not yield such a stronger result since they use
subsumption checks to analyse counterexamples provided by the oracle
and to integrate them into the current hypothesis ontology, and
subsumption is {\sc NP}-complete in these
DLs~\citep{kikot2011tractability}.
In Section~\ref{sec:final}, we show that $\EL_{\mn{lhs}}$ ontologies
are learnable in polynomial time, a result that extends the known
polynomial time learnability of propositional Horn formulas
\citep{DBLP:journals/ml/AngluinFP92}, which correspond to
$\mathcal{EL}$ ontologies without existential restrictions.
In fact, our algorithms take inspiration from learning algorithms for
propositional Horn formulas and combine the underlying ideas with
modern concepts from DL such as canonical models, simulations, and
products.  The algorithm for $\EL_{\mn{lhs}}$ also uses subsumption
checks, which in this case does not get in the way of polynomial time
learnability since subsumption in $\EL_{\mn{lhs}}$ can be decided in
polynomial time.

In Section~\ref{sec:EL_not_learnable}, we then establish that \EL
ontologies are not polynomial query learnable.  Note that the fragment
$\mathcal{EL}_{\mn{rhs}}$ of \EL, which is symmetric to
$\mathcal{EL}_{\mn{lhs}}$ and only admits concept names on the
\emph{left}-hand side of concept inclusions is a fragment of
$\ourDLLite$. Together, our upper bounds for $\ourDLLite$ and
$\mathcal{EL}_{\mn{lhs}}$ thus establish that failure of polynomial
query learnability of $\mathcal{EL}$ ontologies is caused by the
interaction between existential restrictions on the left- and
right-hand sides of concept inclusions. Interestingly, our result
already applies to acyclic \EL TBoxes, which disallow recursive
definitions of concepts and are of a rather restricted syntactic
form. However, the result does rely on concept inclusions as
counterexamples that are of a form not allowed in acyclic TBoxes. We
also show that ontologies formulated in
$\text{DL-Lite}^{\exists}_{\mathcal{R},{\sf horn}}$ and in
$\EL_{\mn{lhs}}$ are neither polynomial query learnable with membership
queries alone nor with equivalence queries alone; corresponding results
for propositional Horn formulas can be found
in~\citep{DBLP:conf/icml/FrazierP93,DBLP:journals/ml/AngluinFP92,Ang},
see also \citep{DBLP:journals/ml/AriasB11}.
\begin{figure}[t]
 \centering
\begin{tikzpicture}
	\begin{pgfonlayer}{nodelayer}
		\node [style=none] (0) at (1.5, 0.75) {$\mathcal{EL}_{\sf rhs}$};
		\node [style=none] (1) at (3, 0.75) {$\mathcal{EL}_{\sf lhs}$};
		\node [style=none] (2) at (-0.25, 1.25) {};
		\node [style=none] (3) at (0.5, 1.25) {};
		\node [style=none] (4) at (1.25, 1.25) {};
		\node [style=none] (5) at (3.75, 1.25) {};
		\node [style=none] (6) at (-0.25, 2.25) {};
		\node [style=none] (7) at (0, 2.25) {};
		\node [style=none] (8) at (1, 2.25) {};
		\node [style=none] (9) at (3.75, 2.25) {};
		\node [style=none] (10) at (3.5, 1.75) {DL-Lite$^\exists_{\Rmc,{\sf horn}}$};
		\node [style=none] (11) at (1.25, 1.75) {DL-Lite$^\exists_{\Rmc}$};
		\node [style=none] (12) at (2.25, 2.75) {$\mathcal{EL}$};
		\node [style=none] (13) at (-4, 1.75) {\small {Polynomial query learnable}};
		\node [style=none] (14) at (-4, 0.75) {\small {Polynomial time learnable}};
		\node [style=none] (15) at (5.25, 1.25) {};
		\node [style=none] (16) at (4.75, 1.25) {};
		\node [style=none] (17) at (5, 2.25) {};
		\node [style=none] (18) at (5.25, 2.25) {};
		\node [style=none] (19) at (-4, 2.75) {\small {Not polynomial query learnable}};
		\node [style=none] (20) at (-0.75, 0.75) {};
		\node [style=none] (21) at (-1.25, 0.75) {};
		\node [style=none] (22) at (-0.75, 1.75) {};
		\node [style=none] (23) at (-1.25, 1.75) {};
		\node [style=none] (24) at (-1.25, 2.75) {};
		\node [style=none] (25) at (-0.75, 2.75) {};
	\end{pgfonlayer}
	\begin{pgfonlayer}{edgelayer}
		\draw [style=simpledotted] (2.center) to (3.center);
		\draw [style=simpledotted] (3.center) to (4.center);
		\draw [style=simpledotted] (4.center) to (5.center);
		\draw [style=simpledotted] (6.center) to (7.center);
		\draw [style=simpledotted] (7.center) to (8.center);
		\draw [style=simpledotted] (8.center) to (9.center);
		\draw [style=simpledotted] (15.center) to (16.center);
		\draw [style=simpledotted] (16.center) to (5.center);
		\draw [style=simpledotted] (18.center) to (17.center);
		\draw [style=simpledotted] (9.center) to (17.center);
		\draw [style=newstyle6] (21.center) to (20.center);
		\draw [style=newstyle6] (23.center) to (22.center);
		\draw [style=newstyle6] (24.center) to (25.center);
	\end{pgfonlayer}
\end{tikzpicture}  %
 \caption{Summary of main results}
 \label{fig:complexity}
\end{figure}
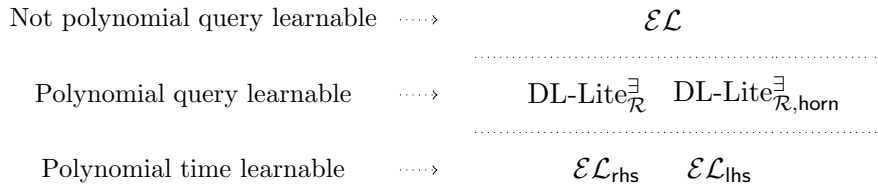
Figure \ref{fig:complexity} summarises  the main results obtained in this paper.

In Section~\ref{sect:related} we provide an extensive discussion of
related work on the exact learning of logical formulas and theories,
and we close the paper with a discussion of open problems. A small
number of proofs are deferred to an appendix. %
 \section{Preliminaries}
\label{sect:prelim}

We introduce the description logics studied in this paper, then
consider a representation of concept expressions in terms of labelled
trees and show how important semantic notions such as subsumption
between concept expressions can be characterised by homomorphisms
between the corresponding trees. This also involves introducing
canonical models, which are an important tool throughout the
paper. Finally, we formally introduce the framework of exact learning.

\subsection{Description Logics}
Let \NC be a countably infinite set of \emph{concept names} (denoted by upper
case letters $A$, $B$, etc) and let \NR be a countably infinite set of
\emph{role names} disjoint from \NC (denoted by lower case letters $r$, $s$, etc). Concept and role names can
be regarded as unary and binary predicates, respectively. 
In description logic, constructors are used to define compound concept
and role expressions from concept and role names. In this paper, the
only role constructor is the inverse role constructor: for $r \in
\NR$, the expression $r^-$ is the \emph{inverse role} of
$r$. Semantically, $r^-$ represents the converse of the binary
relation $r$. A \emph{role expression} is a role name or an inverse role. 
We set
$r^{-}:=s$ if $r=s^{-}$ for a role name $s$. For brevity, we will
typically
speak of \emph{roles} rather than of role expressions.
The concept constructors used in this paper are $\top$
(\emph{everything}), $\sqcap$ (\emph{conjunction}), and $\exists r.C$
(\emph{qualified existential restriction}). Formally,
\emph{concept expressions} $C$
are defined according to the following syntactic rule: 
$$
C,D \quad :=  \quad \top \quad | \quad A \quad | \quad C \sqcap D \quad | \quad \exists r.C
$$
where $A$ is a concept name and $r$ is a role.
For example, $\exists \mathsf{hasChild}.\top \sqcap \exists \mathsf{gender}.\mathsf{Male}$ denotes the class of individuals 
who have a child and whose gender is male.

\medskip

Terminological knowledge is captured by finite sets of
\emph{inclusions} between concept expressions or roles. Specifically,
\begin{itemize}

\item a \emph{concept inclusion (CI)} takes the form $C 
  \sqsubseteq D$, where $C$ and $D$ are concept expressions, and

\item a \emph{role inclusion (RI)} takes the form $r
  \sqsubseteq s$, where $r$ and $s$ are roles.

\end{itemize}
An \emph{ontology} or \emph{TBox} is a finite set of CIs and
RIs.\footnote{In the description logic literature, CIs of the form
  introduced here are often called $\mathcal{ELI}$ CIs to distinguish
  them from CIs that use concept expressions formulated in other
  description logics.  The TBoxes are called $\mathcal{ELIH}$ TBoxes
  (TBoxes that consist of $\mathcal{ELI}$ CIs and RIs).} We use $C
\equiv D$ as an abbreviations for the two CIs $C \sqsubseteq D$ and $D
\sqsubseteq C$ and likewise for $r\equiv s$; we speak of \emph{concept
equivalences (CEs)} and \emph{role equivalences (REs)}, respectively.

\begin{example}\label{ex:1}
Consider the following  TBox: %
\begin{eqnarray}
\label{eq:1} \mn{Prof}            & \sqsubseteq & \exists \mn{supervisor\_of}.\mn{Student}\sqcap \exists \mn{conduct\_research}.\top\\
\label{eq:2} \mn{Graduate}        & \equiv      & \exists \mn{has\_degree}.\top \\
\label{eq:3} \mn{GraduateStudent} & \equiv      & \mn{Student} \sqcap \mn{Graduate}\\
\label{eq:4} \mn{GraduateStudent} & \sqsubseteq & \exists \mn{supervisor\_of}^{-}.\mn{Prof}\\
\label{eq:5} \mn{supervisor\_of}  & \sqsubseteq & \mn{advisor\_of}\\
\label{eq:6} \mn{CS\_Graduate}    & \equiv      & \exists \mn{has\_degree}.\mn{CS\_Degree} 
\end{eqnarray}
The CI in Line~\ref{eq:1} states that every professor supervises
students and conducts research. Notice that we do not specify the
specific area of research, hence we use an unqualified existential
restriction of the form $\exists r.\top$.  The CE in Line~\ref{eq:2}
defines a graduate as anyone who has a degree. The CE in
Line~\ref{eq:3} defines a graduate student as a student who is a
graduate.  The CI in Line~\ref{eq:4} states that graduate students are
supervised by professors. Notice that we use the inverse role of
$\mn{supervisor\_of}$ here.  Line~\ref{eq:5} shows an RI which states
that every supervisor is an advisor.  The CE in the last line defines
a computer science graduate as someone with a degree in computer
science.
\end{example}
A \emph{signature} is a set of concept and role names and we use
$\Sigma_\Tmc$ to denote the signature of the TBox $\Tmc$, that is, 
the set of concept and role names that occur in it.
The \emph{size} $|C|$ of a concept expression $C$ is the length of the string that represents
$C$, where concept names and role names are considered to be of length one.  The
\emph{size} $|\Tmc|$ of a TBox \Tmc is defined as $\sum_{C \sqsubseteq D\in \Tmc} |C|+|D|$.

\medskip

The semantics of concept expressions and TBoxes is defined as follows
\citep{Textbook}.  
An \emph{interpretation} $\Imc= (\Delta^{\Imc},\cdot^{\Imc})$ is given by a
non-empty set $\Delta^{\Imc}$ (the \emph{domain} of $\Imc$) and a mapping
$\cdot^{\Imc}$ that maps every concept name $A$ to a subset $A^{\Imc}$ of
$\Delta^{\Imc}$ and every role name $r$ to a subset $r^{\Imc}$ of
$\Delta^{\Imc} \times \Delta^{\Imc}$. The \emph{interpretation} $r^{\Imc}$ of
an inverse role $r=s^{-}$ is given by $r^{\Imc}= \{(d,d')\mid (d',d)\in
s^{\Imc}\}$ and the \emph{interpretation} $C^{\Imc}$ of a 
concept expression $C$ is defined inductively by
\begin{eqnarray*}
\top^{\Imc} & = & \Delta^{\Imc}\\
(C_{1} \sqcap C_{2})^{\Imc} & = & C_{1}^{\Imc} \cap C_{2}^{\Imc}\\
(\exists r.C)^{\Imc} & = & \{ d\in \Delta^{\Imc} \mid \text{there exists $d'\in C^{\Imc}$ with $(d,d')\in r^{\Imc}$}\}.
\end{eqnarray*}
An interpretation $\Imc$ \emph{satisfies} a concept expression $C$ if
$C^\Imc$ is not empty.  It satisfies the CI $C\sqsubseteq D$ if
$C^{\Imc} \subseteq D^{\Imc}$, written as $\Imc\models C \sqsubseteq
D$.  Similarly, \Imc satisfies RI $r \sqsubseteq s$ if $r^{\Imc}
\subseteq s^{\Imc}$, written as $\Imc\models r \sqsubseteq s$.  \Imc
is a \emph{model of a TBox $\Tmc$} if it satisfies all CIs and RIs in
$\Tmc$. A TBox $\Tmc$ \emph{entails} a CI or RI $\alpha$, in symbols
$\Tmc\models \alpha$, if $\alpha$ is satisfied in every model of \Tmc.
Concept expressions $C$ and $D$ are \emph{equivalent w.r.t.~$\Tmc$}, written
$\Tmc \models C\equiv
D$, if $\Tmc\models C\sqsubseteq D$ and \mbox{$\Tmc\models D\sqsubseteq
  C$}; equivalence of roles $r$ and $s$ is defined accordingly,
written $\Tmc\models s \equiv r$. TBoxes $\Tmc$ and $\Tmc'$ are \emph{logically equivalent},
in symbols $\Tmc\equiv \Tmc'$, if $\Tmc\models \alpha$ for all $\alpha\in \Tmc'$ and vice versa. 

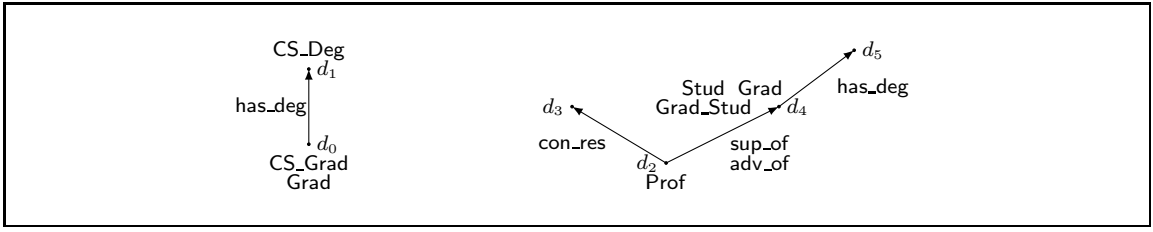
\begin{figure}
\begin{boxedminipage}[h]{\columnwidth}
\vspace*{1em}
\center
\begin{tikzpicture}
	\begin{pgfonlayer}{nodelayer}
		\node [style=none] (0) at (-0.75, -1.5) {};
		\node [style=none] (1) at (0.75, -0.75) {};
		\node [style=none] (2) at (-5.5, -1.25) {};
		\node [style=none] (3) at (-5.5, -0.25) {};
		\node [style=none] (4) at (-0.75, -1.75) {\scalefont{0.7}{$\mn{Prof}$}};
		\node [style=none] (5) at (-0.25, -0.5) {\scalefont{0.7}{$\mn{Stud}$}};
		\node [style=none] (6) at (0.5, -0.5) {\scalefont{0.7}{$\mn{Grad}$}};
		\node [style=none] (7) at (-0.25, -0.75) {\scalefont{0.7}{$\mn{Grad\_Stud}$}};
		\node [style=none] (8) at (-5.5, -1.75) {\scalefont{0.7}{$\mn{Grad}$}};
		\node [style=none] (9) at (-5.5, -1.5) {\scalefont{0.7}{$\mn{CS\_Grad}$}};
		\node [style=none] (10) at (-5.5, -0) {\scalefont{0.7}{$\mn{CS\_Deg}$}};
		\node [style=none] (11) at (2, -0.5) {\scalefont{0.7}{$\mn{has\_deg}$}};
		\node [style=none] (12) at (0.5, -1.25) {\scalefont{0.7}{$\mn{sup\_of}$}};
		\node [style=none] (13) at (0.5, -1.5) {\scalefont{0.7}{$\mn{adv\_of}$}};
		\node [style=none] (14) at (-6, -0.75) {\scalefont{0.7}{$\mn{has\_deg}$}};
		\node [style=none] (15) at (1.75, -0) { };
		\node [style=none] (16) at (-2, -0.75) {};
		\node [style=none] (17) at (-2, -1.25) {\scalefont{0.7}{$\mn{con\_res}$}};
		\node [style=newstyle1] (18) at (-2, -0.75) {};
		\node [style=newstyle1] (19) at (-0.75, -1.5) {};
		\node [style=newstyle1] (20) at (1.75, -0) {};
		\node [style=newstyle1] (21) at (-5.5, -0.25) {};
		\node [style=newstyle1] (22) at (-5.5, -1.25) {};
		\node [style=newstyle1] (23) at (0.75, -0.75) {};
		\node [style=none] (24) at (-2.25, -0.75) {\scalefont{0.7}{$\symbold$}};
		\node [style=none] (25) at (-1, -1.5) {\scalefont{0.7}{$\symbola$}};
		\node [style=none] (26) at (2, -0) {\scalefont{0.7}{$\symbolc$}};
		\node [style=none] (27) at (-5.25, -0.25) {\scalefont{0.7}{$\symbolf$}};
		\node [style=none] (28) at (-5.25, -1.25) {\scalefont{0.7}{$\symbole$}};
		\node [style=none] (29) at (1, -0.75) {\scalefont{0.7}{$\symbolb$}};
	\end{pgfonlayer}
	\begin{pgfonlayer}{edgelayer}
		\draw [style=newstyle2] (0.center) to (1.center);
		\draw [style=newstyle2] (2.center) to (3.center);
		\draw [style=newstyle2] (1.center) to (15.center);
		\draw [style=newstyle2] (0.center) to (16.center);
	\end{pgfonlayer}
\end{tikzpicture} \vspace*{1em}
\end{boxedminipage}
\caption{Illustration to Example \ref{ex:model}. \label{fig:model}}
\end{figure}

\begin{example}\label{ex:model}
Consider  the TBox \Tmc from Example~\ref{ex:1} and
the interpretation $\Imc$ that is illustrated
in Figure~\ref{fig:model} and defined by setting
$\Delta^\Imc = \{d_{0},\ldots,d_{5}\}$ and 
\[
\begin{array}{rcl@{\qquad}rcl}
 \mn{Prof}^\Imc                 & = &  \{\symbola\}      ,    &
 \mn{conduct\_research}^\Imc  & = &  \{(\symbola,\symbold)\}                 \\
 \mn{Student}^\Imc              & = &  \{\symbolb\}      ,    &
 \mn{supervisor\_of}^\Imc     & = &  \{(\symbola,\symbolb)\},          \\
 \mn{Graduate\_Student}^\Imc    & = &  \{\symbolb\}      ,    &
 \mn{advisor\_of}^\Imc        & = &  \{(\symbola,\symbolb)\},          \\ 
 \mn{Graduate}^\Imc           & = &  \{\symbolb,\symbole\}  ,    &
 \mn{has\_degree}^\Imc        & = &
 \{(\symbolb,\symbolc),(\symbole,\symbolf)\},\\
 \mn{CS\_Graduate}^\Imc         & = &  \{\symbole\}      ,    &
 \mn{CS\_Degree}^\Imc         & = &  \{\symbolf\}.         \\
\end{array}
\]
It is easy to see that $\Imc$ is a model of $\Tmc$.
Moreover, $\Imc\not\models \mn{Graduate}\sqsubseteq \mn{CS\_Graduate}$  as
$\mn{Graduate}^\Imc = \{\symbolb, \symbole\}$ but
$\mn{CS\_Graduate}^\Imc = \{\symbole\}$, thus 
$\Tmc\not\models\mn{Graduate} \sqsubseteq\mn{CS\_Graduate}$.
It can be shown that 
$\Tmc\models\mn{CS\_Graduate} \sqsubseteq\mn{Graduate}$.
\end{example}
It is \ExpTime-complete to decide, given a TBox \Tmc and a concept
inclusion $C \sqsubseteq D$, whether $\Tmc\models C \sqsubseteq D$
\citep{DBLP:conf/owled/BaaderLB08}; this reasoning problem is known as
\emph{subsumption}. Because of this high complexity, the profiles of
OWL~2 are based on syntactically more restricted description logics in
which subsumption is less complex. We next introduce a few relevant
such logics.
A \emph{basic concept} is a concept name or a concept expression of
the form $\exists r.\top$, where $r$ is a role. For example, $\exists
\mathsf{hasChild}^{-}.\top$ is a basic concept, but $\exists
\mathsf{hasChild}^{-}.\mathsf{Graduate}$ is not.

\medskip
\noindent
{\bf \ourDLLite.} A \emph{\ourDLLite CI} takes the form
$$
B \sqsubseteq C
$$
where $B$ is a basic concept and $C$ is a concept expression. A
\emph{\ourDLLite inclusion} is a \ourDLLite CI or an RI. A \emph{\ourDLLite
TBox} is a finite set of \ourDLLite inclusions.
\begin{example}\label{ex:1.1}
  Lines $(\ref{eq:1})$, $(\ref{eq:4})$, and $(\ref{eq:5})$ of
  Example~\ref{ex:1} are \ourDLLite inclusions and Line~$(\ref{eq:2})$
  abbreviates the two \ourDLLite CIs $\mn{Graduate} \sqsubseteq
  \mn{has\_degree}.\top$ and $\mn{has\_degree}.\top \sqsubseteq
  \mn{Graduate}$. Lines  $(\ref{eq:3})$ and $(\ref{eq:6})$ do not fall
  within   \ourDLLite.
\end{example}

\medskip
\noindent
{\bf \ourDLLitehorn.} 
In the extension \ourDLLitehorn of \ourDLLite, CIs take the form
$$
B_{1} \sqcap \cdots \sqcap B_{n} \sqsubseteq C
$$
where $B_{1},\ldots,B_{n}$ are basic concepts and $C$ is a concept
expression.  A \emph{\ourDLLitehorn TBox} $\Tmc$ is a finite set of
\ourDLLitehorn CIs and RIs. Both \ourDLLite and \ourDLLitehorn have
been investigated in detail in \citep{DBLP:journals/jair/ArtaleCKZ09}.
\begin{example}\label{ex:1.2}
  As Lines $(\ref{eq:1})$, $(\ref{eq:2})$, $(\ref{eq:4})$ and
  $(\ref{eq:5})$ from Example~\ref{ex:1} fall within \ourDLLite, they
  also fall within \ourDLLitehorn. Line $(\ref{eq:3})$ falls within
  \ourDLLitehorn. Line $(\ref{eq:6})$ is not in \ourDLLitehorn.
\end{example}

\medskip
\noindent
{\bf $\mathbfcal{EL}$.} An \EL \emph{concept expression} is a concept expression that does not use inverse roles.
An \emph{$\mathcal{EL}$ concept inclusion} is a CI of the form 
$$
C \sqsubseteq D
$$
where $C$ and $D$ are $\mathcal{EL}$ concept expressions. An
\emph{$\mathcal{EL}$ TBox} is a finite set of $\mathcal{EL}$
CIs. Thus, $\mathcal{EL}$ does neither admit role inclusions nor
inverse roles. In contrast to \ourDLLitehorn, however, it allows
existential restrictions $\exists r.C$ with $C\not=\top$ on the
left-hand side of CIs.
\begin{example}\label{ex:1.3}
Inclusions $(\ref{eq:1})$, $(\ref{eq:2})$, $(\ref{eq:3})$ and $(\ref{eq:6})$ from Example~\ref{ex:1} are
\EL inclusions. Inclusion $(\ref{eq:5})$ is not an \EL inclusion.
\end{example}
Subsumption is {\sc NP}-complete in \ourDLLite and in \ourDLLitehorn,
see~\citep{kikot2011tractability} for the lower bound
and~\citep{CDLLR07,DBLP:journals/jair/ArtaleCKZ09} for the upper
bound. Subsumption in \EL is in {\sc PTime}~\citep{BaBrLu-IJCAI-05}
and this is still true if RIs that do not use inverse roles are
admitted in the TBox. Given a TBox \Tmc and an RI $r \sqsubseteq s$,
deciding whether $\Tmc \models r \sqsubseteq s$ is possible in {\sc
  PTime} in all description logics considered in this paper. In fact,
$\Tmc\models r \sqsubseteq s$ if, and only if, there exists a sequence
$r_{0},\ldots,r_{n}$ of roles such that $r=r_{0}$, $s= r_{n}$, and for
every $i<n$ either $r_{i}\sqsubseteq r_{i+1}\in \Tmc$ or
$r_{i}^{-}\sqsubseteq r_{i+1}^{-}\in \Tmc$. Our learning algorithms
will carry out various subsumption checks as a subprocedure, as
detailed later on.

\subsection{Tree representation of concept expressions}

To achieve an elegant and succinct exposition of our learning
algorithms, it will be convenient to represent concept expressions $C$
as a finite directed tree $T_C$ whose nodes are labelled with sets of
concept names and whose edges are labelled with roles, and to describe
manipulations of concept expressions in terms of manipulations of the
corresponding tree such as merging nodes, replacing subgraphs,
modifying node and edge labels, etc.  We generally use $\rho_C$ to
denote the root node of the tree $T_C$.  In detail, $T_C$ is defined
as follows. For $C=\top$, the tree $T_C$ has a single node $d$ with
label $l(d)=\emptyset$; if $C = A$, where $A$ is a concept name, then
$T_C$ has a single node $d$ with $l(d)=\{A\}$; if $C=\exists r . D$,
then $T_C$ is obtained from $T_D$ by adding a new root $d_0$ and an
edge from $d_0$ to the root $d$ of $T_D$ with label $l(d_0,d)=r$ (we
then call $d$ an \emph{$r$-successor} of $d_0$); if $C=D_1 \sqcap
D_2$, then $T_C$ is obtained by identifying the roots of $T_{D_1}$ and
$T_{D_2}$.

\begin{example}\label{ex-tree1}
For $C= \mn{Student} \sqcap \exists \mn{has\_degree}.\exists \mn{has\_degree}^{-}.\mn{Graduate\_Student}$, $T_{C}$ has three nodes,
$\symbolee,\symbolcc,\symbolbb$, where
$\symbolee$ is the root $\rho_C$  of $T_C$,
$\symbolcc$ is a successor of $\symbolee$ and $\symbolbb$ is a
successor of $\symbolcc$, the labelling of the nodes is given by
$l(\symbolee)=\{\mn{Student}\}$,
$l(\symbolcc)=\emptyset$, and $l(\symbolbb)=\{\mn{Graduate\_Student}\}$, and the labelling of the edges is
given by $l(\symbolee,\symbolcc)=\mn{has\_degree}$ and
$l(\symbolcc,\symbolbb)=\mn{has\_degree}^{-}$, see
Figure~\ref{fig:concept-tree}~(left). 
\end{example}
Conversely, every labelled finite directed tree $T$ of the described form gives
rise to a concept expression $C_T$ in the following way: if
$T$ has a single node $d$ labelled by $\{A_{1},\ldots,A_{n}\}$, then
$C_{T}=A_{1}\sqcap \cdots \sqcap A_{n}$ (we treat $\top$ as the empty
conjunction here, so if $l(d)=\emptyset$ then $C_{T}=\top$).  
Inductively, let $d$ be the root of $T$ labelled with
$l(d)=\{A_{1},\ldots,A_{n}\}$, let $d_{1},\ldots,d_{m}$ be
the successors of $d$, and let $l(d,d_{1})=r_{1},\ldots, l(d,d_{m})=r_{m}$.
Assume $C_{d_{1}},\ldots,C_{d_{m}}$ are the concept expressions
corresponding to the subtrees of $T$ with roots $d_{1},\ldots,d_{m}$,
respectively. Then $C_{T}=A_{1}\sqcap\cdots \sqcap A_{n} \sqcap \exists
r_{1}.C_{d_{1}} \sqcap \cdots \sqcap \exists r_{m}.C_{d_{m}}$.

\begin{example}\label{ex-tree2}
Let $T$ be the tree with root $\symbolaa$ labelled by $\{\mn{Prof}\}$
and successors $\symboldd,\symbolg$
labelled by $\emptyset$ and $\{\mn{Graduate}\}$, respectively, and with edge labelling given
by $l(\symbolaa,\symboldd)=\mn{conduct\_research}$ and
$l(\symbolaa,\symbolg)=\mn{supervisor\_of}$. Then
$C_{T}=\mn{Prof} \sqcap \exists \mn{conduct\_research}.\top \sqcap
\exists \mn{supervisor\_of}.\mn{Graduate}$;
see Figure~\ref{fig:concept-tree}~(right).
\end{example}

\begin{figure}
\begin{boxedminipage}[h]{\columnwidth}
\vspace*{1em}
\center
\begin{tikzpicture}
	\begin{pgfonlayer}{nodelayer}
		\node [style=none] (0) at (-3, 0.5) {};
		\node [style=none] (1) at (-3, 1.5) {};
		\node [style=none] (2) at (1, 1) {};
		\node [style=none] (3) at (2, 1.75) {};
		\node [style=none] (4) at (-2.5, 2) {\scalefont{0.7}{$\mn{\ \ \ has\_deg}^-$}};
		\node [style=none] (5) at (-2.5, 1) {\scalefont{0.7}{$\mn{has\_deg}$}};
		\node [style=none] (6) at (2, 1.25) {\scalefont{0.7}{$\mn{sup\_of}$}};
		\node [style=none] (7) at (-3, 2.5) {};
		\node [style=none] (8) at (0, 1.75) {};
		\node [style=none] (9) at (0, 1.25) {\scalefont{0.7}{$\mn{con\_res}$}};
		\node [style=none] (10) at (-3, 0.25) {\scalefont{0.7}{$\mn{Stud}$}};
		\node [style=none] (11) at (-3, 2.75) {\scalefont{0.7}{$\mn{Grad\_Stud}$}};
		\node [style=none] (12) at (1, 0.75) {\scalefont{0.7}{$\mn{Prof}$}};
		\node [style=none] (13) at (2, 2) {\scalefont{0.7}{$\mn{Grad}$}};
		\node [style=newstyle1] (14) at (-3, 0.5) {};
		\node [style=newstyle1] (15) at (-3, 1.5) {};
		\node [style=newstyle1] (16) at (-3, 2.5) {};
		\node [style=none] (17) at (-3.25, 0.5) {\scalefont{0.7}{$\symbolee$}};
		\node [style=none] (18) at (-3.25, 1.5) {\scalefont{0.7}{$\symbolcc$}};
		\node [style=none] (19) at (-3.25, 2.5) {\scalefont{0.7}{$\symbolbb$}};
		\node [style=newstyle1] (20) at (0, 1.75) {};
		\node [style=newstyle1] (21) at (1, 1) {};
		\node [style=newstyle1] (22) at (2, 1.75) {};
		\node [style=none] (23) at (0.75, 1) {\scalefont{0.7}{$\symbolaa$}};
		\node [style=none] (24) at (-0.25, 1.75) {\scalefont{0.7}{$\symboldd$}};
		\node [style=none] (25) at (2.25, 1.75) {\scalefont{0.7}{$\symbolg$}};
	\end{pgfonlayer}
	\begin{pgfonlayer}{edgelayer}
		\draw [style=newstyle2] (0.center) to (1.center);
		\draw [style=newstyle2] (2.center) to (3.center);
		\draw [style=newstyle2] (1.center) to (7.center);
		\draw [style=newstyle2] (2.center) to (8.center);
	\end{pgfonlayer}
\end{tikzpicture} \vspace*{1em}
\end{boxedminipage}
\caption{Illustration to Examples \ref{ex-tree1} (left) and \ref{ex-tree2} (right). \label{fig:concept-tree}}
\end{figure}
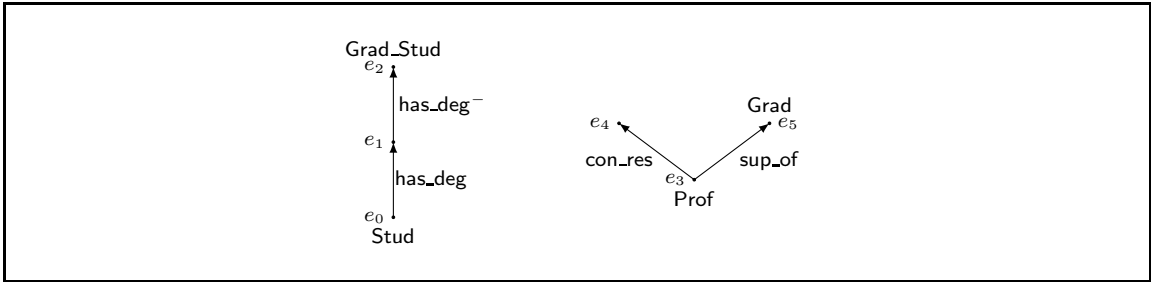

In what follows, we will not always distinguish explicitly between $C$ and its tree representation
$T_C$ which allows us to speak, for example, about the nodes and subtrees of a concept expression.

One important use of the tree representation of concept expressions is
that both the truth relation `$d\in C^{\Imc}$' and the entailment
`$\Tmc\models C \sqsubseteq D$' can be characterised in terms of
homomorphisms between labelled trees and interpretations. A mapping
$h$ from a tree $T_{C}$ corresponding to a concept expression $C$ to
an interpretation $\Imc$ is a \emph{homomorphism} if $A\in l(d)$
implies $h(d)\in A^{\Imc}$ for every concept name $A$ and $r=l(d,d')$
implies $(h(d),h(d'))\in r^{\Imc}$ for all role names $r$. The
following characterisation of the truth relation $d\in C^{\Imc}$ by
means of homomorphisms is well-known. 
\begin{lemma}\label{hom}
Let $\Imc$ be an interpretation, $d\in \Delta^{\Imc}$, and $C$ a concept expression.
Then $d\in C^{\Imc}$ if, and only if, there is a homomorphism from $T_{C}$ to $\Imc$ mapping $\rho_{C}$
to $d$. 
\end{lemma}
The proof is by a straightforward induction on the structure of $C$,
see for example~\cite{DBLP:conf/ijcai/BaaderKM99} for details.

\begin{example}\label{ex:homom}
Consider the interpretation $\Imc$ from Example~\ref{ex:model} and
the tree representations of the concept expressions given in Figure~\ref{fig:concept-tree}.
It can be seen that functions $g$ and $h$ defined as
$g(\symbolee) =\symbolb$, $g(\symbolcc) = \symbolc$, $g(\symbolbb) = \symbolb$ and 
$h(\symbolaa) = \symbola$, $h(\symboldd) = \symbold$, $h(\symbolg) = \symbolb$ are homomorphisms and 
so, by Lemma~\ref{hom}, 
$\symbolb\in (\mn{Student}\sqcap \exists
\mn{has\_degree}.\exists\mn{has\_degree}^{-}.\mn{Graduate\_Student})^\Imc$ and
$\symbola\in(\mn{Prof}\sqcap \exists\mn{conduct\_research}.\top\sqcap
\exists\mn{supervisor\_of}.\mn{Graduate})^\Imc$.
\end{example}
It is also standard to characterise the subsumption relation $\emptyset\models
C \sqsubseteq D$ (that is, subsumption relative to the empty TBox) by
means of homomorphisms between the tree representations $T_{D}$ and
$T_{C}$. A \emph{homomorphism} $h$ from labelled tree $T_{1}$ to
labelled tree $T_{2}$ is a mapping from the nodes of $T_{1}$ to the
nodes of $T_{2}$ such that $A\in l(d)$ implies $A \in l(h(d))$ for
every concept name $A$ and $r=l(d,d')$ implies $r=l(h(d),h(d'))$ for
every role $r$. 
\begin{lemma}\label{lem:hom1}
  Let $C$ and $D$ be concept expressions. Then $\emptyset\models C \sqsubseteq
  D$ if, and only if, there is a homomorphism from $T_{D}$ to $T_{C}$
  that maps $\rho_{D}$ to $\rho_{C}$.
\end{lemma}
The `if' direction is essentially a consequence of
Lemma~\ref{hom} and the fact that the composition of two
homomorphisms is again a homomorphism. For the `only if' direction,
one can consider $T_C$ as an interpretation \Imc and apply
Lemma~\ref{hom}.  We again refer to \cite{DBLP:conf/ijcai/BaaderKM99} for
details.

Next, we characterise subsumption in the presence of TBoxes
in terms of homomorphisms.  To achieve this, we make use of the
\emph{canonical model} $\Imc_{C_0,\Tmc}$ of a concept expression~$C_0$ and
a TBox $\Tmc$. %
If $\Tmc=\emptyset$, then we want $\Imc_{C_0,\Tmc}$ to be $T_{C_0}$ viewed
as a tree-shaped interpretation which we denote by $\Imc_{{C_0}}$ rather than by
$\Imc_{{C_0},\Tmc}$. More precisely, the domain of $\Imc_{C_0}$ is the set of
nodes of $T_{{C_0}}$ and
\begin{eqnarray*}
d\in A^{\Imc_{{C_0}}} & \text { iff } & A\in l(d), \text{ for all $d\in \Delta^{\Imc_{{C_0}}}$ and concept names $A$}\\
(d,d')\in r^{\Imc_{{C_0}}} &  \text { iff } & r= l(d,d'), \text{ for all $d,d'\in \Delta^{\Imc_{{C_0}}}$ and roles names $r$}
\end{eqnarray*}
We call the root $\rho_{{C_0}}$ of $T_{{C_0}}$ the \emph{root of $\Imc_{{C_0}}$}.
If $\Tmc\not=\emptyset$, then $\Imc_{{C_0},\Tmc}$ is obtained by extending
$\Imc_{{C_0}}$ so that the CIs in $\Tmc$ are satisfied.  For example, if
$\Tmc=\{A\sqsubseteq \exists r.B\}$ and ${C_0}=A$, then $\Imc_{{C_0}}$ is a
single node $\rho_{{C_0}}$ with $A^{\Imc_{{C_0}}}=\{\rho_{{C_0}}\}$ and
$X^{\Imc_{{C_0}}}=\emptyset$ for all concept and role names $X$ distinct from
$A$. To define $\Imc_{{C_0},\Tmc}$ we add a node $d$ to
$\Delta^{\Imc_{{C_0}}}$ and set $B^{\Imc_{{C_0},\Tmc}}=\{d\}$ and
$r^{\Imc_{{C_0},\Tmc}}=\{(\rho_{{C_0}},d)\}$.  In general, $\Imc_{{C_0},\Tmc}$ is
defined as the limit of a sequence $\Imc_{0}, \Imc_{1},\ldots$
of interpretations, where $\Imc_{0}= \Imc_{{C_0}}$.
For the inductive definition of the sequence, assume that $\Imc_{n}$ has been defined. 
Then obtain $\Imc_{n+1}$ by applying one of the following rules once:
\begin{enumerate}
\item if $C\sqsubseteq D\in \Tmc$ and 
$d\in C^{\Imc_{n}}$ but $d\not\in D^{\Imc_{n}}$, then take the
interpretation $\Imc_{D}$ and add it to $\Imc_{n}$ 
by identifying its root $\rho_{C}$ with $d$. In more detail, assume that 
$\Delta^{\Imc_{n}}\cap \Delta^{\Imc_{C}} =\{ d\}$ and $d=\rho_{C}$ and define $\Imc_{n+1}$ by setting, 
for all concept names $A$ and role names $r$:
$$
\Delta^{\Imc_{n+1}} =  \Delta^{\Imc_{n}}\cup \Delta^{\Imc_{C}}, \quad
A^{\Imc_{n+1}}  =  A^{\Imc_{n}} \cup A^{\Imc_{C}}, \quad
r^{\Imc_{n+1}}  =  r^{\Imc_{n}} \cup r^{\Imc_{C}};
$$

\item if $r \sqsubseteq s\in \Tmc$ and $(d,d')\in r^{\Imc_{n}}$
but $(d,d')\not\in s^{\Imc_{n}}$, then define $\Imc_{n+1}$ as $\Imc_{n}$ except that 
$s^{\Imc_{n+1}}:= s^{\Imc_{n}}\cup \{(d,d')\}$ if $s$ is a role name; otherwise there is a role name 
$s_{0}$ with $s=s_{0}^{-}$ and we define $\Imc_{n+1}$ as $\Imc_{n}$ except that $s_{0}^{\Imc_{n+1}}= s_{0}^{\Imc_{n}}\cup \{(d',d)\}$.

\end{enumerate}
We assume that rule application is fair, that is, if a rule is
applicable in a certain place, then it will indeed eventually be
applied there.  If for some $n>0$ no rule is applicable then we set
$\Imc_{n+1}=\Imc_{n}$. We obtain $\Imc_{C_{0},\Tmc}$ by setting
for all concept names $A$ and role names $r$:
$$
\Delta^{\Imc_{C_{0},\Tmc}} =  \bigcup_{n\geq 0} \Delta^{\Imc_{n}}, \quad
A^{\Imc_{C_{0},\Tmc}}  =  \bigcup_{n\geq 0} A^{\Imc_{n}}, \quad
r^{\Imc_{C_{0},\Tmc}}  =  \bigcup_{n\geq 0}r^{\Imc_{n}}.
$$
Note that the interpretation $\Imc_{C_0,\Tmc}$ obtained in the limit
is tree-shaped and might be infinite.\footnote{The exact shape of
  $\Imc_{C_0,\Tmc}$ depends on the order of rule
  applications. However, all possible resulting interpretations
  $\Imc_{C_0,\Tmc}$ are homomorphically equivalent and, as a
  consequence, the order of rule application is not important 
  for our purposes.} The following example illustrates the definition
of $\Imc_{C_0,\Tmc}$.

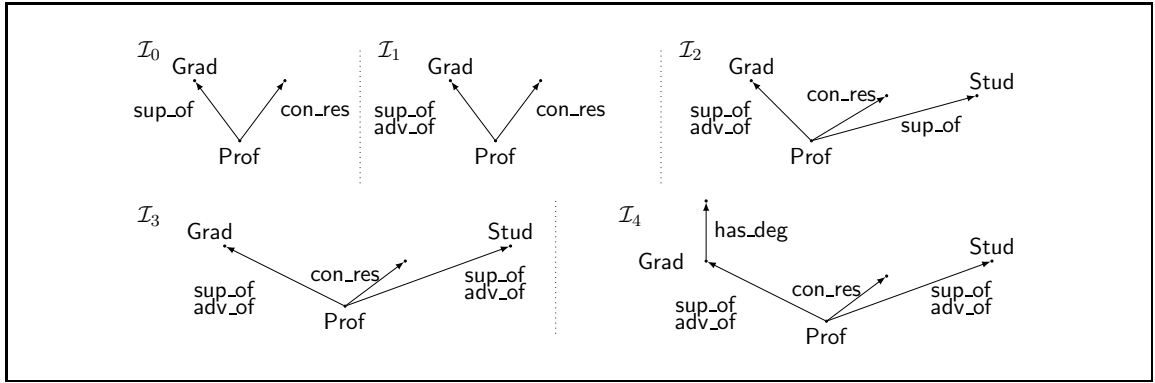
\begin{figure}
\begin{boxedminipage}[h]{\columnwidth}
\vspace*{1em}
\center
\scalebox{.8}{
\begin{tikzpicture}
	\begin{pgfonlayer}{nodelayer}
		\node [style=newstyle1] (0) at (-5.25, 1.5) {};
		\node [style=newstyle1] (1) at (-6, 2.5) {};
		\node [style=newstyle1] (2) at (-4.5, 2.5) {};
		\node [style=none] (3) at (-6.5, 2) {${\sf sup\_of}$};
		\node [style=none] (4) at (-4, 2) {${\sf con\_res}$};
		\node [style=none] (5) at (-5.25, 1.25) {${\sf Prof}$};
		\node [style=none] (6) at (-6.75, 3) {$\Imc_0$};
		\node [style=newstyle1] (7) at (-1, 1.5) {};
		\node [style=newstyle1] (8) at (-1.75, 2.5) {};
		\node [style=none] (9) at (-1, 1.25) {${\sf Prof}$};
		\node [style=none] (10) at (-2.75, 3) {$\Imc_1$};
		\node [style=none] (11) at (0.25, 2) {${\sf con\_res}$};
		\node [style=newstyle1] (12) at (-0.25, 2.5) {};
		\node [style=none] (13) at (-2.5, 2) {${\sf sup\_of}$};
		\node [style=none] (14) at (-2.5, 1.75) {${\sf adv\_of}$};
		\node [style=none] (15) at (4.25, 1.25) {${\sf Prof}$};
		\node [style=none] (16) at (4.75, 2.25) {${\sf con\_res}$};
		\node [style=newstyle1] (17) at (5.5, 2.25) {};
		\node [style=newstyle1] (18) at (4.25, 1.5) {};
		\node [style=none] (19) at (2.75, 2) {${\sf sup\_of}$};
		\node [style=newstyle1] (20) at (3.25, 2.5) {};
		\node [style=none] (21) at (2.25, 3) {$\Imc_2$};
		\node [style=none] (22) at (2.75, 1.75) {${\sf adv\_of}$};
		\node [style=newstyle1] (23) at (7, 2.25) {};
		\node [style=none] (24) at (6.25, 1.75) {${\sf sup\_of}$};
		\node [style=none] (25) at (7.25, 2.5) {${\sf Stud}$};
		\node [style=none] (26) at (1.75, 3) {};
		\node [style=none] (27) at (1.75, 0.75) {};
		\node [style=none] (28) at (-3.25, 3) {};
		\node [style=none] (29) at (-3.25, 0.75) {};
		\node [style=none] (30) at (-6.75, 0.25) {$\Imc_3$};
		\node [style=none] (31) at (0, 0.5) {};
		\node [style=none] (32) at (-0.75, 0) {${\sf Stud}$};
		\node [style=none] (33) at (0, -1.75) {};
		\node [style=none] (34) at (-1, -0.75) {${\sf sup\_of}$};
		\node [style=none] (35) at (-5.5, -1) {${\sf sup\_of}$};
		\node [style=newstyle1] (36) at (-3.5, -1.25) {};
		\node [style=none] (37) at (-3.5, -1.5) {${\sf Prof}$};
		\node [style=newstyle1] (38) at (-0.75, -0.25) {};
		\node [style=newstyle1] (39) at (-2.5, -0.5) {};
		\node [style=none] (40) at (-5.5, -1.25) {${\sf adv\_of}$};
		\node [style=none] (41) at (-3.5, -0.75) {${\sf con\_res}$};
		\node [style=newstyle1] (42) at (-5.5, -0.25) {};
		\node [style=none] (43) at (-1, -1) {${\sf adv\_of}$};
		\node [style=none] (44) at (-6, 2.75) {${\sf Grad}$};
		\node [style=none] (45) at (3.25, 2.75) {${\sf Grad}$};
		\node [style=none] (46) at (-5.75, 0) {${\sf Grad}$};
		\node [style=none] (47) at (-1.75, 2.75) {${\sf Grad}$};
		\node [style=none] (48) at (1.25, 0.25) {$\Imc_4$};
		\node [style=newstyle1] (49) at (4.5, -1.5) {};
		\node [style=newstyle1] (50) at (7.25, -0.5) {};
		\node [style=none] (51) at (6.75, -1.25) {${\sf adv\_of}$};
		\node [style=none] (52) at (1.75, -0.5) {${\sf Grad}$};
		\node [style=none] (53) at (6.75, -1) {${\sf sup\_of}$};
		\node [style=none] (54) at (4.5, -1) {${\sf con\_res}$};
		\node [style=newstyle1] (55) at (2.5, -0.5) {};
		\node [style=none] (56) at (2.5, -1.5) {${\sf adv\_of}$};
		\node [style=none] (57) at (4.5, -1.75) {${\sf Prof}$};
		\node [style=newstyle1] (58) at (5.5, -0.75) {};
		\node [style=none] (59) at (7.25, -0.25) {${\sf Stud}$};
		\node [style=none] (60) at (2.5, -1.25) {${\sf sup\_of}$};
		\node [style=newstyle1] (61) at (2.5, 0.5) {};
		\node [style=none] (62) at (3.25, 0) {${\sf has\_deg}$};
	\end{pgfonlayer}
	\begin{pgfonlayer}{edgelayer}
		\draw [style=newstyle2] (0) to (2);
		\draw [style=newstyle2] (0) to (1);
		\draw [style=newstyle2] (7) to (12);
		\draw [style=newstyle2] (7) to (8);
		\draw [style=newstyle2] (18) to (17);
		\draw [style=newstyle2] (18) to (20);
		\draw [style=newstyle2] (18) to (23);
		\draw [style=simpledotted] (26.center) to (27.center);
		\draw [style=simpledotted] (28.center) to (29.center);
		\draw [style=simpledotted] (31.center) to (33.center);
		\draw [style=newstyle2] (36) to (39);
		\draw [style=newstyle2] (36) to (42);
		\draw [style=newstyle2] (36) to (38);
		\draw [style=newstyle2] (49) to (58);
		\draw [style=newstyle2] (49) to (55);
		\draw [style=newstyle2] (49) to (50);
		\draw [style=newstyle2] (55) to (61);
	\end{pgfonlayer}
\end{tikzpicture}
 }
\vspace*{1em}
\end{boxedminipage}
\caption{Canonical model construction for Example~\ref{ex:canonical}. \label{fig:canonical2}}
\end{figure}
\begin{example}\label{ex:canonical}
Consider the following TBox $\Tmc$:
\begin{eqnarray}
\label{eq:7}\mn{Prof}                     & \sqsubseteq & \exists \mn{supervisor\_of}.\mn{Student}\\
\label{eq:8}\mn{Prof}                     & \sqsubseteq & \exists \mn{conduct\_research}.\top\\
\label{eq:9}\mn{Graduate}                 & \sqsubseteq & \exists \mn{has\_degree}.\top \\
\label{eq:10}\exists \mn{has\_degree}.\top & \sqsubseteq & \mn{Graduate} \\
\label{eq:11}\mn{supervisor\_of}           & \sqsubseteq & \mn{advisor\_of}
\end{eqnarray}
and the concept expression
$$
C_0 = \mn{Prof}\sqcap \exists\mn{conduct\_research}.\top\sqcap \exists\mn{supervisor\_of}.\mn{Graduate}.
$$
Figure~\ref{fig:canonical2} illustrates the steps of the canonical model construction with
$\Imc_0$ being $\Imc_{C_0}$ and 
$\Imc_4$ being the canonical model $\Imc_{C_0,\Tmc}$.
\end{example}
The following lemma provides the announced characterisation of
subsumption in the presence of TBoxes.
\begin{lemma}\label{lem:can1}
  Let $\Tmc$ be a TBox and $C$ a concept expression.  Then
  $\Imc_{C,\Tmc}$ is a model of $\Tmc$ and the following conditions
  are equivalent, for every concept expression $D$:
\begin{enumerate}
\item $\Tmc\models C \sqsubseteq D$;
\item $\rho_{C}\in D^{\Imc_{C,\Tmc}}$;
\item there is a homomorphism from $T_{D}$ to $\Imc_{C,\Tmc}$ that maps $\rho_{D}$ to $\rho_{C}$.
\end{enumerate} 
\end{lemma}
The proof is completely standard (see, for example,
\citep{DBLP:conf/rweb/Krotzsch12}), we only give a high-level overview. Using the
construction of $\Imc_{C,\Tmc}$, it is not hard to show that
$\Imc_{C,\Tmc}$ is a model of \Tmc and that $\rho_C \in
C^{\Imc_{C,\Tmc}}$. This implies `1 $\Rightarrow$ 2' and `2
$\Rightarrow$ 3' follows from Lemma~\ref{hom}. For `3
$\Rightarrow$ 1', one can show that for any model \Imc of \Tmc and any
$d \in C^\Imc$, there is a homomorphism from $\Imc_{C,\Tmc}$ to \Imc
that maps $\rho_C$ to $d$. In fact, one constructs a homomorphism to
\Imc from each of the interpretations $\Imc_0,\Imc_1,\dots$ built
during the construction of $\Imc_{C,\Tmc}$, which is not hard by
analysing the rules applied during that construction. The homomorphism
built for each $\Imc_{n+1}$ extends that for $\Imc_n$ and thus we can
take the unions of all those homomorphisms to obtain a homomorphism
from $\Imc_{C,\Tmc}$ to \Imc. It remains to compose homomorphisms and
apply Lemma~\ref{hom}.

\smallskip

We are only going to use canonical models and Lemma~\ref{lem:can1} in
the context of \ourDLLite and \ourDLLitehorn. Next, we identify a more
subtle property of canonical models of \ourDLLite and \ourDLLitehorn TBoxes that we
need later on. Very roughly speaking, it states a form of locality
which is due to the fact that existential restrictions on the left-hand side of CIs in
\ourDLLitehorn are unqualified. Assume that $d\in (\exists r.D)^{\Imc_{C,\Tmc}}$
for the canonical model $\Imc_{C,\Tmc}$ of a concept expression $C$ and TBox $\Tmc$.
Assume $d\in \Delta^{{\Imc}_{C}}$. We know that there
exists a homomorphism $h$ from $T_{\exists r.D}$ to $\Imc_{C,\Tmc}$ mapping $\rho_{\exists r.D}$ to $d$.
Then either $h$ maps some elements of $T_{D}$ into $\Delta^{\Imc_{C}}$ or it maps the whole tree
$T_{D}$ into $\Delta^{\Imc_{C,\Tmc}}\setminus \Delta^{\Imc_{C}}$. We are interested in the latter case.
The following lemma states that if $\Tmc$ is a \ourDLLite TBox, then there is a basic concept $B$ with 
$d\in B^{\Imc_{C}}$ such that $\Tmc\models B\sqsubseteq \exists r.D$. Thus, the question whether 
$d\in (\exists r.D)^{\Imc_{C,\Tmc}}$ only depends on the concept names $A$ with $d\in A^{\Imc_{C}}$ 
and the roles $r$ with $d\in (\exists r.\top)^{\Imc_{C}}$. If $\Tmc$ is a \ourDLLitehorn TBox, it might
not be sufficient to take a single basic concept but at least a set of basic concepts 
suffices (corresponding to the
fact that \ourDLLitehorn admits conjunctions on the left-hand side of CIs). This observation does not hold for
$\mathcal{EL}$ TBoxes and is ultimately the reason for the fact that one cannot polynomially
learn $\mathcal{EL}$ TBoxes. The following example illustrates this observation.

\begin{example}
Consider the $\mathcal{EL}$ TBox 
$\Tmc = \{\exists r.A \sqsubseteq A, A \sqsubseteq \exists r.B\}$ and let $C = \exists r.\exists r.A$.
Then $\rho_{C}\in (\exists r.B)^{\Imc_{C,\Tmc}}$ since $\Tmc\models C \sqsubseteq \exists r.B$. 
We therefore find a homomorphism $h$ from $T_{\exists r.B}$
to $\Imc_{C,\Tmc}$ mapping $\rho_{\exists r.B}$ to $\rho_{C}$. This homomorphism maps $T_{B}$ (which has a single 
node only) into $\Delta^{\Imc_{C}}\setminus \Delta^{\Imc_{C,\Tmc}}$. The only basic concept $B$ with
$\rho_{C}\in B^{{\Imc}_{C}}$ is $\exists r.\top$ but clearly 
$\Tmc\not\models \exists r.\top \sqsubseteq \exists r.B$ and so the observation we sketched above 
does not hold for $\mathcal{EL}$. 
\end{example}
We present this result in a more formal way.
Let $T_{1}$ and $T_{2}$ be trees with
labelling functions $l_{1}$ and $l_{2}$, respectively. We call $T_{1}$
a \emph{subtree of} $T_{2}$ if the following conditions hold:
$T_{1}\subseteq T_{2}$, $l_{1}$ is the restriction of $l_{2}$ to
$T_{1}$, and if $d\in T_{1}$ and $d'$ is a successor of $d$ in
$T_{2}$, then $d'$ is a successor of $d$ in $T_{1}$ as well.  The
\emph{one-neigbourhood $N_{\Imc_{C}}(d)$ of $d\in \Delta^{\Imc_{C}}$}
is the set of concept names $A$ with $d\in A^{\Imc_{C}}$ and basic
concepts $\exists r.\top$ such that there exists $d'\in
\Delta^{\Imc_{C}}$ with $(d,d')\in r^{\Imc_{C}}$.
\begin{lemma}\label{lem:can2}
Let $\Tmc$ be a \ourDLLitehorn TBox, $D=\exists r.D'$ and assume $h: T_{D} \rightarrow \Imc_{C,\Tmc}$ is such that 
$h(\rho_{D})=d\in \Delta^{\Imc_{C}}$ and the image of the subtree $T_{D'}$ of
$T_{D}$ under $h$ is included in $\Delta^{\Imc_{C,\Tmc}}\setminus \Delta^{\Imc_{C}}$. Then there exists $I\subseteq N_{\Imc_{C}}(d)$ such that 
$\Tmc\models \bigsqcap_{E\in I}E\sqsubseteq D$.
Moreover, if $\Tmc$ is a \ourDLLite TBox, then there exists such a set $I\subseteq N_{\Imc_{C}}(d)$ with a single concept.
\end{lemma}  
\begin{proof}(sketch) This property of canonical models for \ourDLLitehorn has been proved implicitly in many papers, for example
\citep{DBLP:journals/jair/ArtaleCKZ09}. We give a sketch. Let $N$ be the conjunction of all $E\in N_{\Imc_{C}}(d)$ and assume 
$\Tmc\not\models N\sqsubseteq D$. Consider the canonical model $\Imc_{N,\Tmc}$. 
By definition, the one-neighbourhoods of $\rho_{N}$ in $\Imc_{N,\Tmc}$ and $d$ in $\Imc_{C,\Tmc}$ coincide.
Now observe that the canonical model of any concept expression $C_{0}$ and TBox $\Tmc$ is obtained 
from $\Imc_{C_{0}}$ by hooking tree-shaped interpretations $\Imc_{d}$ with root $d$ to every $d$ in $\Imc_{C_{0}}$. 
As in \ourDLLitehorn the concept expressions on the left-hand side of CIs are basic concepts, the interpretations 
$\Imc_{d}$ only depend on the one-neighbourhood $N_{\Imc_{C_{0}}}(d)$ of $d$ in $\Imc_{C_{0}}$. Thus, the tree-shaped
interpretations hooked to $\rho_{N}$ in $\Imc_{N,\Tmc}$ and to $d$ in $\Imc_{C,\Tmc}$ coincide and
the homomorphism $h$ given in Lemma~\ref{lem:can2} provides a homomorphism $h: T_{D} \rightarrow \Imc_{N}$ 
such that $h(\rho_{D})=\rho_{N}$. By Lemma~\ref{lem:can1}, $\Tmc\models N \sqsubseteq D$. We have derived a 
contradiction.
For \ourDLLite one only requires a single member of $N_{\Imc_{C}}(d)$ since the left-hand side of
CIs in \ourDLLite consists of a single basic concept only. 
\end{proof}
We close the introduction of description logics with some comments
about the choice of our languages. In the DL literature it is not
uncommon to consider the weaker variant DL-Lite$_{\mathcal{R}}$ of
\ourDLLite in which only basic concepts are admitted on the right-hand
side of CIs, but compound concepts are not. This is often without loss
of generality since every \ourDLLite TBox can be expressed in
DL-Lite$_{\mathcal{R}}$ by using additional role names; in this way,
standard \ourDLLite reasoning tasks such as subsumption and
conjunctive query answering can be reduced in polynomial time to the
corresponding tasks for DL-Lite$_{\mathcal{R}}$.
Such a reduction
is not possible in the framework of exact learning that we are
concerned with in this paper. In fact, in contrast to \ourDLLite
TBoxes, TBoxes in DL-Lite$_{\mathcal{R}}$ are trivially polynomial
time learnable using either membership queries only or
equivalence queries only as there are only polynomially many CIs and
RIs over a given signature.

\subsection{Exact Learning}
We introduce the relevant notation for exact learning.
A \emph{learning framework} $\Fmf$ is a triple $(X, \Lmc, \mu)$, where $X$ 
is a set of \emph{examples} (also called \emph{domain} or \emph{instance space}), 
$\Lmc$ is a set of \emph{concepts},\footnote{The
  similarity of this name to `concept expression' is 
  accidental and should not be taken to mean that these two notions
  are closely related. Both is standard terminology in the respective
  area.} %
and $\mu$ is a mapping from $\Lmc$
to $2^X$.   
We say that $x\in X$ is a \emph{positive example} for $l \in \Lmc$ if $x\in \mu(l)$ and a
\emph{negative example} for $l$ if $x\not\in \mu(l)$. 

We give a formal definition of polynomial query and time learnability within a learning framework. 
Let $\Fmf = (X, \Lmc, \mu)$ be a learning framework.  We are interested in the
exact identification of a \emph{target}   concept representation $l\in\Lmc$
by posing queries to oracles.
Let  ${\sf MEM}_{\Fmf,l}$ be the oracle that takes as input some $x \in X$ and
returns `yes' if $x \in \mu(l)$ and `no' otherwise. 
A \emph{membership query} is a call to the oracle ${\sf MEM}_{\Fmf,l}$.
Similarly, for every $l \in \Lmc$, we denote by ${\sf EQ}_{\Fmf,l}$ the oracle
that takes as input a \emph{hypothesis} concept representation $h \in \Lmc$ and
returns `yes' if $\mu(h) = \mu(l)$ and a \emph{counterexample} 
$x \in \mu(h) \oplus \mu(l)$ otherwise, where $\oplus$ denotes the symmetric
set difference.
There is no assumption regarding which counterexample in $\mu(h) \oplus \mu(l)$
is chosen by the oracle.
An \emph{equivalence query} is a call to the oracle ${\sf EQ}_{\Fmf,l}$.

A \emph{learning algorithm} for \Fmf is a deterministic algorithm that
takes no input, is allowed to make queries to ${\sf MEM}_{\Fmf,l}$ and
${\sf EQ}_{\Fmf,l}$ (without knowing what the target $l$ to be learned
is), and that eventually halts and outputs some $h\in\Lmc$ with
$\mu(h) = \mu(l)$.
We say that \Fmf is \emph{exact learnable} if there is a learning
algorithm for \Fmf and that \Fmf is \emph{polynomial query learnable}
if it is exact learnable by an algorithm $A$ such that at every step
the sum of the sizes of the inputs to membership and equivalence
queries made by $A$ up to that step is bounded by a polynomial
$p(|l|,|x|)$, where $l$ is the target and $x \in X$ is the largest
counterexample seen so far \citep{arias2004exact}.  Finally, \Fmf is
\emph{polynomial time learnable} if it is exact learnable by an
algorithm $A$ such that at every step (we count each call to an oracle
as one step of computation) of computation the time used by $A$ up to
that step is bounded by a polynomial $p(|l|,|x|)$, where $l\in \Lmc$
is the target and $x \in X$ is the largest counterexample seen so far.
Clearly, a learning framework $\Fmf$ that is polynomial time learnable
is also polynomial query learnable.

The aim of this paper is to study learnability of description logic
TBoxes. In this context, each DL $L$ gives rise to a learning
framework $(X, \Lmc, \mu)$, as follows: $\Lmc$ is the set of all
TBoxes formulated in $L$, $X$ is the set of all CIs and RIs formulated
in $L$, and $\mu(\Tmc)=\{\alpha \in X\mid \Tmc\models \alpha\}$ for
every $\Tmc \in \Lmc$. Observe that $\mu(\Tmc)=\mu(\Tmc')$ iff $\Tmc\equiv \Tmc'$, for all TBoxes $\Tmc$ and $\Tmc'$.
We say that \emph{$L$ TBoxes are polynomial query learnable} if the learning framework defined by $L$ is
polynomial query learnable, and likewise for polynomial time
learnability. What does not show up directly in this representation is
our assumption that the signature of the target TBox is known to the
learner. Note that this is a standard assumption. For example, when
learning propositional Horn formulas, it is common to assume that the
variables in the target formula are known to the learner.

\section{Learning \ourDLLite TBoxes}\label{sec:dllite}
We prove that \ourDLLite TBoxes are polynomial query learnable. If
inverse roles are disallowed in CIs and RIs of the target TBox then
our algorithm runs in polynomial time and thus shows that TBoxes in
this restricted language are polynomial time learnable. Without this
restriction, polynomial time learnability remains open.

To simplify the presentation, we make two minor assumptions about the
target TBox~$\Tmc$. We will show later how these assumptions can be
overcomed. First, we assume that $\Tmc$ does not entail non-trivial
role equivalences, that is, there do not exist distinct roles $r$ and
$s$ such that $\Tmc\models r\equiv s$.  This allows us to avoid
dealing with classes of equivalent roles, simplifying notation. The
second requirement is a bit more subtle. A concept inclusion is in
\emph{reduced form} if it is between basic concepts or its left-hand
side is a concept name. A TBox \Tmc is in \emph{named form} if all CIs
in it are in reduced form and it contains a concept name $A_{r}$ such
that $A_{r}\equiv \exists r.\top \in \Tmc$, for each role $r$. We
assume that the target TBox is in named form and that all CIs
considered by the learner are in reduced form. In particular,
counterexamples returned by the oracle are immediately converted into
this form. %

\begin{algorithm}[t]
\begin{algorithmic}[1]
\Require A \ourDLLite TBox $\Tmc$ in named form given to the oracle; $\Sigma_\Tmc$ given to the learner.
\Ensure  TBox $\Hmc$, computed by the learner, such that $\Tmc\equiv\Hmc$.
\setstretch{1.1} 
\State Compute   
$\mathcal{H}_{basic}  =  \{r\sqsubseteq s \mid 
\Tmc \models r\sqsubseteq s\}  \cup  
  \{B_{1} \sqsubseteq B_{2} \mid 
\Tmc \models B_{1}\sqsubseteq B_{2}, \mbox{ $B_{1},B_{2}$ basic}\}$ 
\State Set $\mathcal{H}_{add}=\emptyset$
\While{$\mathcal{H}_{basic}\cup \mathcal{H}_{add}\not\equiv\Tmc$}
\State Let $A \sqsubseteq C$ be the returned positive counterexample
 for $\mathcal{T}$ relative
to $\mathcal{H}_{basic}\cup \mathcal{H}_{add}$ 
\If {there is $A\sqsubseteq C'\in \mathcal{H}_{add}$}
\State Replace $A \sqsubseteq C'$ by $A \sqsubseteq C\sqcap C'$ in $\mathcal{H}_{add}$
\Else
\State Add $A \sqsubseteq C$ to $\mathcal{H}_{add}$
\EndIf
\EndWhile
\State \textbf{return} $\Hmc = \mathcal{H}_{basic} \cup \mathcal{H}_{add}$
\end{algorithmic}
\caption{Na\"\i{}ve learning algorithm for \ourDLLite\label{alg:dl-lite-naive}}
\end{algorithm}

\begin{example}\label{ex:named-form}
  Although the TBox $\Tmc$ from Example~\ref{ex:canonical} does not
  entail role equivalences and all its CIs are in
  reduced form,
  it is not in named form.  To fix
  this, we introduce concept names $A_{\mn{supervisor\_of}}$,
  $A_{\mn{conduct\_research}}$ and $A_{\mn{advisor\_of}}$ and extend
  \Tmc with the following equivalences:
{
\allowdisplaybreaks
\begin{eqnarray}
\label{eq:12} A_{\mn{supervisor\_of}}    & \equiv & \exists \mn{supervisor\_of}.\top\\
\label{eq:13} A_{\mn{conduct\_research}} & \equiv & \exists \mn{conduct\_research}.\top\\
\label{eq:14} A_{\mn{advisor\_of}}       & \equiv & \exists \mn{advisor\_of}.\top.
\end{eqnarray}
}
Notice that $\mn{Graduate}$ acts as a name for $\exists \mn{has\_degree}.\top$ so
no new definition is needed for the role $\mn{has\_degree}$.
The TBox $\Tmc' = \Tmc\cup\{(\ref{eq:12}), (\ref{eq:13}), (\ref{eq:14})\}$ is in named form.
\end{example}

To develop the learning algorithm it is instructive to start with a
na\"\i{}ve version that does not always terminate but which can be
refined to obtain the desired algorithm.  This version is presented as
Algorithm~\ref{alg:dl-lite-naive}.  Given the signature $\Sigma_\Tmc$
of the target TBox $\Tmc$, the learner starts with computing the set
$\Hmc_{basic}$ by posing to the oracle the membership query
`$\Tmc\models r \sqsubseteq s$?'  for all $r,s \in \Sigma_{\Tmc}$ and
`$\Tmc\models B_{1}\sqsubseteq B_{2}$?'  for all basic concept
$B_1,B_2$ over~$\Sigma_{\Tmc}$.  Observe that
$\Tmc\models\Hmc_{basic}$. Then it enters the main \textbf{while}
loop.  Note that the condition `$\Hmc_{basic}\cup \Hmc_{add}
\not\equiv \Tmc$?' in Line~3 is implemented using an equivalence query
to the oracle, and that $A \sqsubseteq C$ in Line~4 refers to the
counterexample returned by the oracle in the case that equivalence
does not hold. The counterexample must be positive since we maintain
the invariant $\Tmc\models \Hmc_{basic}\cup\Hmc_{add}$ throughout the
run of the algorithm.  If there is no CI of the form $A\sqsubseteq C'$ in
$\Hmc_{add}$ then $A\sqsubseteq C$ is added to $\Hmc_{add}$, otherwise
$A \sqsubseteq C \sqcap C'$ is (Lines~6 and~8).  The algorithm terminates when $
\Hmc_{basic}\cup\Hmc_{add} \equiv \Tmc$, implying that the target TBox
has been
learned. %

\begin{example}\label{ex:naive-run}
{\em For the TBox $\Tmc'$ from Example~\ref{ex:named-form}, Algorithm~\ref{alg:dl-lite-naive} 
first computes $\Hmc_{basic}$ which coincides with $\Tmc'$ except that
$\mn{Prof} \sqsubseteq  \exists \mn{supervisor\_of}.\mn{Student}$ is not included since
the concept $\exists \mn{supervisor\_of}.\mn{Student}$ is not basic.
 In the main loop the only counterexamples to $\Hmc_{basic}\cup\Hmc_{add}\equiv \Tmc'$
are (up to logical equivalence modulo $\Hmc_{basic}$) the CIs
$$
\mn{Prof} \sqsubseteq  \exists \mn{supervisor\_of}.\mn{Student}, \quad
\mn{Prof} \sqsubseteq  \exists \mn{advisor\_of}.\mn{Student}.
$$
If the oracle returns the first CI in the first iteration, the algorithm terminates immediately having learned $\Tmc'$. 
Otherwise the oracle first returns the second CI and then returns the first CI in the second iteration.
The algorithm terminates with 
$$
\Hmc_{add} = \{ \mn{Prof} \sqsubseteq  \exists \mn{supervisor\_of}.\mn{Student} \sqcap
\exists \mn{advisor\_of}.\mn{Student}\}
$$
which is equivalent to $\Tmc'$.
}
\end{example}
We now consider five examples on which this na\"\i{}ve algorithm fails
to terminate after polynomially many steps (or at all), each example
motivating a different \emph{modification step} that is added to
Algorithm~\ref{alg:dl-lite-naive} after Lines~4 and 5. The final,
corrected algorithm is given as Algorithm~\ref{alg:dl-lite} below.
Each modification step takes as input a counterexample $A \sqsubseteq
C$ against the equivalence $ \Hmc_{basic}\cup\Hmc_{add} \equiv \Tmc$
and modifies it by posing membership queries to the oracle to obtain a
CI $A'\sqsubseteq C'$ which is still a counterexample and has
additional desired properties.  CIs satisfying all five additional
properties will be called \emph{$\Tmc$-essential}. The five
modification steps are of three different types:
\begin{enumerate}
\item two \emph{saturations steps}: the underlying tree of $T_{C}$ is left unchanged but the labelling is modified
by adding concept names to node labels or replacing roles in edge labels;
\item two \emph{merging steps}: nodes in the tree $T_{C}$ are merged,
  resulting in a tree with fewer nodes;
\item a \emph{decomposition step}: $T_{C}$ is replaced by a subtree or a subtree is removed
from~$T_{C}$, and the concept name $A$ on the left-hand side might be replaced.
\end{enumerate}
The saturation and merging steps do not change the left-hand side
$A$
of the CI $A
\sqsubseteq C$ and result in a logically stronger CI $A \sqsubseteq
C'$ in the sense that $\emptyset \models C' \sqsubseteq
C$. In contrast, the decomposition step can be regarded as a reset
operation in which also the left-hand side can change and which is
logically not related to $A
\sqsubseteq
C$. We start with an example which motivates the first saturation
step.
\begin{example}\label{ex:csat}
Let 
$$
    \Tmc=\{A\sqsubseteq \exists r.A\} \cup \Tmc_{\textsc{nf}},
$$ 
where $\Tmc_\textsc{nf} = \{A_r\equiv \exists r.\top\}$ ensures that $\Tmc$ is in named form.
First, Algorithm~\ref{alg:dl-lite-naive} computes $\mathcal{H}_{basic}$.
Afterwards the oracle can provide for the $n$-th equivalence query in the while
loop the positive counterexample $A \sqsubseteq \exists r^{n+1}.\top$, for any $n \geq 1$ (here we set inductively
$\exists r^{m+1}.\top= \exists r.\exists r^{m}.\top$ and $\exists r^{1}.\top =
\exists r.\top$). Thus, the algorithm does not terminate.
\end{example}
Informally, the problem for the learner in Example~\ref{ex:csat} is
that the concepts $\exists r^{n}.\top$ used in the counterexamples $A
\sqsubseteq \exists r^{n}.\top$ get larger and larger, but still none
of the counterexamples implies $A \sqsubseteq \exists r.A$.  We
address this problem by saturating $T_{C}$ with implied concept
names. For the following discussion, recall that we do not distinguish
between the concept expression $C$ and its tree representation
$T_{C}$. For example, if we say that $C'$ is obtained from $C$ by
adding a concept name $B$ to the label of node $d$ in $C$, then this
stands for: $C'$ is the concept expression corresponding to the tree
obtained from $T_{C}$ by adding $B$ to the label of $d$ in $T_{C}$.
\begin{definition}[Concept saturation for $\Tmc$]\label{def:csat}
  Let $A\sqsubseteq C$ be a CI with $\Tmc\models A\sqsubseteq C$.  A
  CI $A\sqsubseteq C'$ is \emph{obtained from $A\sqsubseteq C$ by
    concept saturation for} $\Tmc$ if $\Tmc\models A\sqsubseteq
  C'$ and $C'$ is obtained from $C$ by adding a concept name to the
  label of some node of $C$.  We say that $A\sqsubseteq C$ is
  \emph{concept saturated for $\Tmc$} if there is no $A \sqsubseteq
  C'$ with $C \neq C'$ that can be obtained from $A \sqsubseteq C$ by
  concept saturation.
\end{definition}
Observe that the learner can compute a concept saturated $A\sqsubseteq
C'$ from a counterexample $A\sqsubseteq C$ by posing polynomially many
membership queries to the oracle: it simply asks for any node $d$ in $T_{C}$ and concept name
$E\in \Sigma_{\Tmc}$ whether $\Tmc\models A \sqsubseteq C^{E,d}$,
where $C^{E,d}$ is obtained from $C$ by adding $E$ to the label of
$d$. If the answer is positive, it replaces $C$ by $C^{E,d}$ and
proceeds.
\begin{example}[Example~\ref{ex:csat} continued]\label{ex:csat-cont}
  The CIs $A \sqsubseteq \exists r^{n}.\top$ are not concept saturated
  for $\Tmc$. For example, for $n=2$, the concept
  saturation of $A \sqsubseteq \exists r^{2}.\top$ is
$A \sqsubseteq A \sqcap A_{r} \sqcap \exists r.(A \sqcap A_{r}
  \sqcap \exists r.(A_{r}\sqcap A))$.  Now observe that if the CI $A
  \sqsubseteq C$ returned by the oracle to the first equivalence query
  is transformed by the learner into a concept saturated CI (after
  Line~4), then the TBox $\Tmc=\{A\sqsubseteq \exists r.A\} \cup
  \Tmc_{\textsc{nf}}$ is learned in one step: the only possible
  counterexamples returned by the oracle to the equivalence query
  $\Hmc_{basic}\equiv \Tmc$ are of the form $A \sqsubseteq C_{1}
  \sqcap \exists r.C_{2}$ for some concepts $C_{1}$ and $C_{2}$.
  Concept saturation results in a concept of the form $C_{1}'\sqcap
  \exists r.(A \sqcap C_{2}')$ and $\{A\sqsubseteq C_{1}'\sqcap
  \exists r.(A \sqcap C_{2}')\} \models A \sqsubseteq \exists r.A$.
\end{example}

The following example motivates the second saturation step.
Here and in the subsequent examples we do not transform the TBoxes into named form
as this does not effect the argument and simplifies presentation.
\begin{example}\label{ex:1role}
Consider
for $n \geq 1$ the TBoxes
\begin{eqnarray*}
\Tmc_{n}   & = & \{A \sqsubseteq \exists e_1.\exists e_2. \ldots \exists e_n.\top\} \cup \{e_i \sqsubseteq r_i,e_i \sqsubseteq s_i \mid 1\leq i \leq n\}.
\end{eqnarray*}
For $M\subseteq \{1,\ldots,n\}$, set $C_{M}= \exists t_1.\exists t_2. \ldots \exists t_n.\top$, where 
$t_i = r_i$ if $i \in M$ and $t_i = s_i$ if $i \notin M$.
Then %
for the first $2^{n}$ equivalence queries in the while loop the oracle can provide
a positive counterexample $A \sqsubseteq C_{M}$
by always choosing a fresh set $M\subseteq \{1,\ldots,n\}$.
\end{example}
Intuitively, the problem for the learner in Example~\ref{ex:1role} is
that there are exponentially many logically incomparable CIs that are
entailed by $\Tmc_{n}$ but do not entail $A \sqsubseteq \exists
e_1.\exists e_2. \ldots \exists e_n.\top$.  A step towards resolving
this problem is to replace the roles $r_{i}$ and $s_{i}$ by the roles
$e_{i}$ in the counterexamples $A \sqsubseteq C_{M}$.
\begin{definition}[Role saturation for $\Tmc$]\label{def:rsat}
  Let $A\sqsubseteq C$ be a CI with $\Tmc\models A\sqsubseteq C$.  A
  CI $A\sqsubseteq C'$ is \emph{obtained from $A \sqsubseteq C$ by
    role saturation for} $\Tmc$ if $\Tmc\models A \sqsubseteq C'$ and
  $C'$ is obtained from $C$ by replacing in some edge label a role $r$
  by a role $s$ with $\Tmc\models s\sqsubseteq r$. We say that
  $A\sqsubseteq C$ is \emph{role saturated for $\Tmc$} if there is no
$A \sqsubseteq C'$ with $C \neq C'$ that can be obtained from $A
\sqsubseteq C$ by role saturation.
\end{definition}
Similarly to concept saturation, the learner can compute a role
saturated $A\sqsubseteq C'$ from a counterexample $A\sqsubseteq C$ by
posing polynomially many membership queries.  Observe that in
Example~\ref{ex:1role} the role saturation of any $A \sqsubseteq
C_{M}$ is $A \sqsubseteq \exists e_1.\exists e_2. \ldots \exists
e_n.\top$. Thus, if the counterexample $A \sqsubseteq C_{M}$ returned
by the first equivalence query is transformed into a role saturated
CI, then the algorithm terminates after one step. We now introduce and
motivate our two merging rules.
\begin{example}\label{ex:1merge}
{\em
Consider the TBox 
$$\Tmc=\{A \sqsubseteq \exists r.\top \sqcap \exists s.\top \sqcap \exists e.B\}.$$
and fix an $n \geq 1$. 
For $M\subseteq \{1,\ldots,n\}$, 
set $C_{M}= \exists t_1.\exists t_1^{-}.\exists t_2.\exists t_2^{-}. 
\ldots \exists t_n.\exists t_n^{-}.\exists e.\top$, where 
$t_i = r$ if $i \in M$ and $t_i = s$ if $i \notin M$. %
Figure \ref{fig:parent-child} (left) illustrates the concept expression
$C_{\{1,3\}}$, assuming $n=3$.
\begin{figure}[t]
 \begin{boxedminipage}[h]{\columnwidth}
 \centering
\begin{tikzpicture}
	\begin{pgfonlayer}{nodelayer}
		\node [style=newstyle1] (0) at (-9.25, 1.5) {};
		\node [style=newstyle1] (1) at (-9.25, 2) {};
		\node [style=newstyle1] (2) at (-9.25, 2) {};
		\node [style=newstyle1] (3) at (-9.25, 2.5) {};
		\node [style=newstyle1] (4) at (-9.25, 2.5) {};
		\node [style=newstyle1] (5) at (-9.25, 3) {};
		\node [style=newstyle1] (6) at (-9.25, 3) {};
		\node [style=newstyle1] (7) at (-9.25, 3.5) {};
		\node [style=none] (8) at (-9.5, 1.75) {$r$};
		\node [style=none] (9) at (-9, 2.25) {$r^-$};
		\node [style=none] (10) at (-9.5, 2.75) {$s$};
		\node [style=none] (11) at (-9, 3.25) {$s^-$};
		\node [style=none] (12) at (-9.5, 3.75) {$r$};
		\node [style=newstyle1] (13) at (-9.25, 3.5) {};
		\node [style=newstyle1] (14) at (-9.25, 4) {};
		\node [style=newstyle1] (15) at (-9.25, 3.5) {};
		\node [style=newstyle1] (16) at (-9.25, 4.5) {};
		\node [style=none] (17) at (-9, 4.25) {$r^-$};
		\node [style=newstyle1] (18) at (-9.25, 4) {};
		\node [style=newstyle1] (19) at (-9.25, 4.5) {};
		\node [style=newstyle1] (20) at (-9.25, 4.5) {};
		\node [style=newstyle1] (21) at (-9.25, 4.5) {};
		\node [style=none] (22) at (-9.5, 4.75) {$e$};
		\node [style=newstyle1] (23) at (-9.25, 4.5) {};
		\node [style=newstyle1] (24) at (-9.25, 5) {};
		\node [style=newstyle1] (25) at (-9.25, 4.5) {};
		\node [style=newstyle1] (26) at (-9.25, 5) {};
		\node [style=newstyle1] (27) at (-6, 2.25) {};
		\node [style=newstyle1] (28) at (-6.75, 2.75) {};
		\node [style=newstyle1] (29) at (-6.25, 3.5) {};
		\node [style=newstyle1] (30) at (-5.5, 4.25) {};
		\node [style=none] (31) at (-9, 1.5) {};
		\node [style=none] (32) at (-6, 2) {};
		\node [style=none] (33) at (-9, 2) {};
		\node [style=none] (34) at (-7, 2.75) {};
		\node [style=none] (35) at (-6.75, 2.5) {$r$};
		\node [style=none] (36) at (-6.25, 3) {$s$};
		\node [style=none] (37) at (-5.75, 3.75) {$e$};
		\node [style=none] (38) at (-9, 2.5) {};
		\node [style=none] (39) at (-9, 3) {};
		\node [style=none] (40) at (-7.75, 3) {$\vdots$};
		\node [style=none] (41) at (-7.75, 4) {$h$ };
		\node [style=none] (42) at (-6.25, 2.25) {};
		\node [style=none] (43) at (-9, 5) {};
		\node [style=none] (44) at (-5.75, 4.25) {};
		\node [style=none] (45) at (-6.25, 2.5) {};
		\node [style=none] (46) at (-9, 4.5) {};
		\node [style=none] (47) at (-6.5, 3.5) {};
		\node [style=none] (48) at (-5.25, 4.25) {$B$};
	\end{pgfonlayer}
	\begin{pgfonlayer}{edgelayer}
		\draw [style=newstyle2] (0) to (1);
		\draw [style=newstyle2] (2) to (3);
		\draw [style=newstyle2] (4) to (6);
		\draw [style=newstyle2] (5) to (7);
		\draw [style=newstyle2] (15) to (14);
		\draw [style=newstyle2] (18) to (16);
		\draw [style=newstyle2] (25) to (24);
		\draw [style=newstyle2] (27) to (28);
		\draw [style=newstyle2] (27) to (29);
		\draw [style=newstyle2] (27) to (30);
		\draw [style=newstyle6] (31.center) to (32.center);
		\draw [style=newstyle6] (33.center) to (34.center);
		\draw [style=newstyle6] (38.center) to (42.center);
		\draw [style=newstyle6] (43.center) to (44.center);
		\draw [style=newstyle6] (46.center) to (45.center);
		\draw [style=newstyle6] (39.center) to (47.center);
	\end{pgfonlayer}
\end{tikzpicture}  \end{boxedminipage}
 \caption{Tree representation of %
 $C_{\{1,3\}}$ and homomorphism to $\exists r.\top \sqcap \exists s.\top \sqcap \exists e.B$.}
 \label{fig:parent-child}
\end{figure}
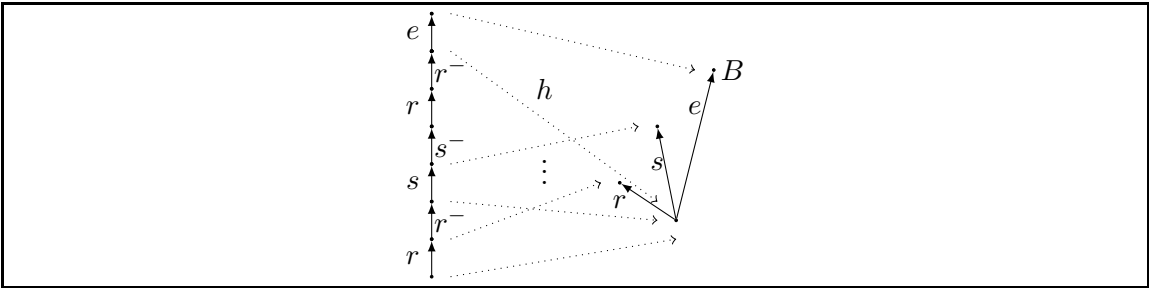
By Lemma~\ref{lem:hom1}, $\Tmc\models A \sqsubseteq C_{M}$ since there is a
homomorphism $h$
from $T_{C_M}$ to the labelled tree that corresponds to $
\exists r.\top \sqcap \exists s.\top \sqcap \exists e.B$, as shown in Figure
\ref{fig:parent-child}. 
Thus, the oracle can provide for the first $2^{n}$ equivalence queries
a positive counterexample $A \sqsubseteq C_{M}$ by always choosing a
fresh set $M\subseteq \{1,\ldots,n\}$.
}
\end{example}
The problem for the learner in Example~\ref{ex:1merge} is similar to
that in Example~\ref{ex:1role}: there are exponentially many logically
incomparable CIs that are entailed by $\Tmc$ but do not entail $A
\sqsubseteq \exists r.\top \sqcap \exists s.\top \sqcap \exists e.B$.
A step towards solving this problem is to merge the predecessor and
successor nodes of a node if the edge labels are inverse to each other
and the resulting CI is still implied by the TBox.  
\begin{definition}[Parent/Child Merging for $\Tmc$]
\label{def:parchimer}
  A concept $C'$ is \emph{obtained from a concept $C$ by parent/child
    merging} if $C'$ is obtained from $C$ by choosing nodes $d,d',d''$
  such that $d$ is an $r$-successor of $d'$, and $d''$ is an
  $r^{-}$-successor of $d$, for some role $r$, and then removing
  $d''$, setting $l(d')=l(d')\cup l(d'')$, and making every
  $s$-successor $e$ of $d''$ in $C$ an $s$-successor of $d'$, for any
  role $s$.

  Let $A\sqsubseteq C$ be a CI with $\Tmc\models A\sqsubseteq C$.  A
  CI $A \sqsubseteq C'$ is \emph{obtained from $A \sqsubseteq C$ by
    parent/child merging} if $\Tmc \models A \sqsubseteq C'$ and $C'$
  is obtained from $C$ by parent/child merging.  We say that
  $A\sqsubseteq C$ is \emph{parent/child merged for $\Tmc$} if
  there is no $A \sqsubseteq C'$ with $C \neq C'$ that can be obtained
  from $A \sqsubseteq C$ by parent/child merging.
\end{definition}
Note that when $C'$ is obtained from $C$ by parent/child merging with
$d'$ and $d''$ as in Definition~\ref{def:parchimer}, then $\emptyset
\models C'\sqsubseteq C$. To show this, one can use
Lemma~\ref{lem:hom1} and the natural homomorphism $h$ from $T_{C}$ to
$T_{C'}$, that is, the identity except that 
$h(d'')=d'$.

Similarly to the saturation operations, the learner can compute a
parent/child merged $A\sqsubseteq C'$ by posing polynomially many
membership queries.  In Example~\ref{ex:1merge} the parent/child
merging of any $A \sqsubseteq C_{M}$ with
$\emptyset \neq M \neq \{1,\dots,n\}$ is
$A \sqsubseteq \exists r.\top \sqcap \exists s.\top \sqcap \exists
e.\top$,
as illustrated in Figure~\ref{fig:parent-child2}: in the first step
the nodes $d_0$ and $d_2$ are merged, two additional merging steps
give $\exists r.\top \sqcap \exists s.\top \sqcap \exists e.\top$.
The following example motivates the second merging operation.

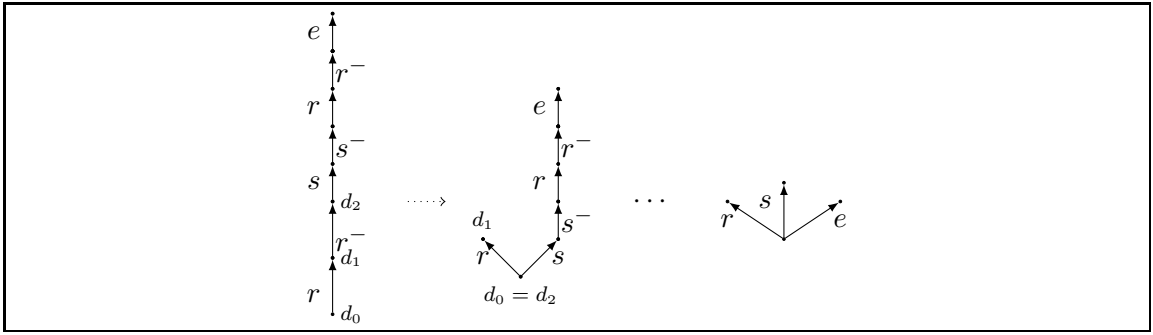
\begin{figure}[t]
 \begin{boxedminipage}[h]{\columnwidth}
 \centering
\begin{tikzpicture}
	\begin{pgfonlayer}{nodelayer}
		\node [style=newstyle1] (0) at (-6.75, 1.5) {};
		\node [style=newstyle1] (1) at (-6.75, 2.25) {};
		\node [style=newstyle1] (2) at (-6.75, 2.25) {};
		\node [style=newstyle1] (3) at (-6.75, 3) {};
		\node [style=newstyle1] (4) at (-6.75, 3) {};
		\node [style=newstyle1] (5) at (-6.75, 3.5) {};
		\node [style=newstyle1] (6) at (-6.75, 3.5) {};
		\node [style=newstyle1] (7) at (-6.75, 4) {};
		\node [style=none] (8) at (-7, 1.75) {$r$};
		\node [style=none] (9) at (-6.5, 2.5) {$r^-$};
		\node [style=none] (10) at (-7, 3.25) {$s$};
		\node [style=none] (11) at (-6.5, 3.75) {$s^-$};
		\node [style=none] (12) at (-7, 4.25) {$r$};
		\node [style=newstyle1] (13) at (-6.75, 4) {};
		\node [style=newstyle1] (14) at (-6.75, 4.5) {};
		\node [style=newstyle1] (15) at (-6.75, 4) {};
		\node [style=newstyle1] (16) at (-6.75, 5) {};
		\node [style=none] (17) at (-6.5, 4.75) {$r^-$};
		\node [style=newstyle1] (18) at (-6.75, 4.5) {};
		\node [style=newstyle1] (19) at (-6.75, 5) {};
		\node [style=newstyle1] (20) at (-6.75, 5) {};
		\node [style=newstyle1] (21) at (-6.75, 5) {};
		\node [style=none] (22) at (-7, 5.25) {$e$};
		\node [style=newstyle1] (23) at (-6.75, 5) {};
		\node [style=newstyle1] (24) at (-6.75, 5.5) {};
		\node [style=newstyle1] (25) at (-6.75, 5) {};
		\node [style=newstyle1] (26) at (-6.75, 5.5) {};
		\node [style=none] (27) at (-6.5, 1.5) { {\scalefont{0.7}{{$d_0$}}}};
		\node [style=none] (28) at (-6.5, 3) {};
		\node [style=none] (29) at (-6.5, 3.5) {};
		\node [style=none] (30) at (-6.5, 5.5) {};
		\node [style=none] (31) at (-6.5, 5) {};
		\node [style=newstyle1] (32) at (-3.75, 4.5) {};
		\node [style=newstyle1] (33) at (-4.75, 2.5) {};
		\node [style=newstyle1] (34) at (-3.75, 4) {};
		\node [style=none] (35) at (-4.75, 2.25) {$r$};
		\node [style=none] (36) at (-4, 3.25) {$r$};
		\node [style=newstyle1] (37) at (-3.75, 3) {};
		\node [style=newstyle1] (38) at (-3.75, 4) {};
		\node [style=newstyle1] (39) at (-3.75, 4.5) {};
		\node [style=newstyle1] (40) at (-3.75, 3.5) {};
		\node [style=newstyle1] (41) at (-3.75, 3) {};
		\node [style=none] (42) at (-3.5, 2.75) {$s^-$};
		\node [style=none] (43) at (-3.5, 2.5) {};
		\node [style=none] (44) at (-3.5, 4.5) {};
		\node [style=none] (45) at (-3.75, 2.25) {$s$};
		\node [style=newstyle1] (46) at (-3.75, 4) {};
		\node [style=newstyle1] (47) at (-3.75, 4) {};
		\node [style=newstyle1] (48) at (-4.25, 2) {};
		\node [style=newstyle1] (49) at (-3.75, 3) {};
		\node [style=newstyle1] (50) at (-3.75, 4) {};
		\node [style=newstyle1] (51) at (-3.75, 4) {};
		\node [style=none] (52) at (-3.5, 4) {};
		\node [style=none] (53) at (-3.5, 3.75) {$r^-$};
		\node [style=newstyle1] (54) at (-4.75, 2.5) {};
		\node [style=newstyle1] (55) at (-3.75, 2.5) {};
		\node [style=newstyle1] (56) at (-3.75, 3.5) {};
		\node [style=newstyle1] (57) at (-4.25, 2) {};
		\node [style=newstyle1] (58) at (-3.75, 2.5) {};
		\node [style=none] (59) at (-4, 4.25) {$e$};
		\node [style=none] (60) at (-5.75, 3) {};
		\node [style=none] (61) at (-5.25, 3) {};
		\node [style=none] (62) at (-2.5, 3) {$\cdots$};
		\node [style=newstyle1] (63) at (-0.75, 2.5) {};
		\node [style=newstyle1] (64) at (-1.5, 3) {};
		\node [style=newstyle1] (65) at (-0.75, 3.25) {};
		\node [style=newstyle1] (66) at (0, 3) {};
		\node [style=none] (67) at (-1.5, 2.75) {$r$};
		\node [style=none] (68) at (-1, 3) {$s$};
		\node [style=none] (69) at (0, 2.75) {$e$};
		\node [style=none] (70) at (-6.5, 3) { {\scalefont{0.7}{{$d_2$}}}};
		\node [style=none] (71) at (-6.5, 2.25) { {\scalefont{0.7}{{$d_1$}}}};
		\node [style=none] (72) at (-4.25, 1.75) { {\scalefont{0.7}{{$d_0=d_2$}}}};
		\node [style=none] (73) at (-4.75, 2.75) { {\scalefont{0.7}{{$d_1$}}}};
	\end{pgfonlayer}
	\begin{pgfonlayer}{edgelayer}
		\draw [style=newstyle2] (0) to (1);
		\draw [style=newstyle2] (2) to (3);
		\draw [style=newstyle2] (4) to (6);
		\draw [style=newstyle2] (5) to (7);
		\draw [style=newstyle2] (15) to (14);
		\draw [style=newstyle2] (18) to (16);
		\draw [style=newstyle2] (25) to (24);
		\draw [style=newstyle2] (55) to (49);
		\draw [style=newstyle2] (48) to (54);
		\draw [style=newstyle2] (38) to (32);
		\draw [style=newstyle2] (41) to (40);
		\draw [style=newstyle2] (56) to (51);
		\draw [style=newstyle2] (57) to (58);
		\draw [style=newstyle6] (60.center) to (61.center);
		\draw [style=newstyle2] (63) to (64);
		\draw [style=newstyle2] (63) to (65);
		\draw [style=newstyle2] (63) to (66);
	\end{pgfonlayer}
\end{tikzpicture}  \end{boxedminipage}
 \caption{Parent/Child Merging of $C_{\{1,3\}}$ for $\Tmc$ }
 \label{fig:parent-child2}
\end{figure}

\begin{example}\label{ex:2merge}
Define concept expressions $C_{i}$ by induction as follows: 
$$
C_{1}= \exists r.\top\sqcap\exists s.\top, \quad
C_{i+1}= C_{1} \sqcap \exists e.C_{i}
$$
and let 
$$
    \Tmc_{n}=\{ A\sqsubseteq \exists e.C_{n}\}.
$$ 
For $M\subseteq \{1,\ldots,n\}$, set $C_{1}^{M}= \exists r.\top$ if $1\in M$ and $C_{1}^{M}=\exists s.\top$ if
$1\notin M$. Also, let $C_{i+1}^{M}=   \exists r.\top \sqcap \exists e.C_{i}^{M}$ if
$i+1\in M$ and $C_{i+1}^{M}=  \exists s.\top \sqcap \exists e.C_{i}^{M}$ if $i+1 \notin M$, 
$1 \leq i <n$.   Figure \ref{fig:sibling} illustrates concept expressions of the form 
$C_n$ and $C_{n}^{M}$. 
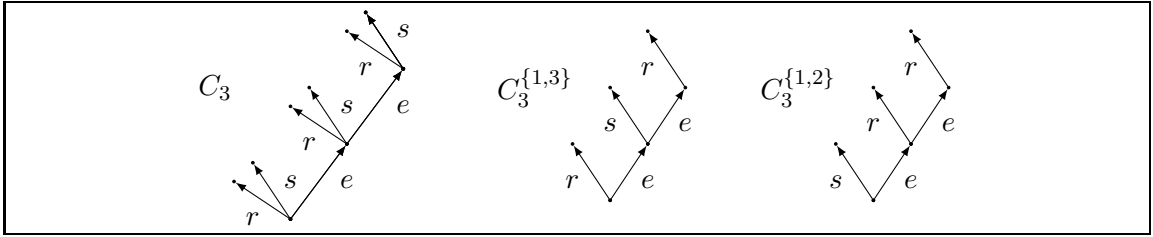
\begin{figure}[t]
 \begin{boxedminipage}[h]{\columnwidth}
 \centering
\begin{tikzpicture}
	\begin{pgfonlayer}{nodelayer}
		\node [style=none] (0) at (-4.75, 4.75) {$C_3$};
		\node [style=newstyle1] (1) at (0, 4) {};
		\node [style=newstyle1] (2) at (3.5, 4) {};
		\node [style=none] (3) at (5, 4.25) {$e$};
		\node [style=none] (4) at (0.5, 4.25) {$s$};
		\node [style=newstyle1] (5) at (4.5, 5.5) {};
		\node [style=newstyle1] (6) at (1, 4) {};
		\node [style=none] (7) at (4.5, 5) {$r$};
		\node [style=newstyle1] (8) at (4.5, 4) {};
		\node [style=none] (9) at (3, 4.75) {$C^{\{1,2\}}_3$};
		\node [style=newstyle1] (10) at (0.5, 3.25) {};
		\node [style=none] (11) at (4.5, 3.5) {$e$};
		\node [style=newstyle1] (12) at (4, 4.75) {};
		\node [style=newstyle1] (13) at (1.5, 4.75) {};
		\node [style=newstyle1] (14) at (0.5, 4.75) {};
		\node [style=newstyle1] (15) at (1, 4) {};
		\node [style=newstyle1] (16) at (1, 5.5) {};
		\node [style=newstyle1] (17) at (5, 4.75) {};
		\node [style=none] (18) at (3.5, 3.5) {$s$};
		\node [style=none] (19) at (4, 4.25) {$r$};
		\node [style=newstyle1] (20) at (4, 3.25) {};
		\node [style=none] (21) at (0, 3.5) {$r$};
		\node [style=none] (22) at (1, 5) {$r$};
		\node [style=newstyle1] (23) at (4.5, 4) {};
		\node [style=newstyle1] (24) at (0.5, 3.25) {};
		\node [style=none] (25) at (-0.5, 4.75) {$C^{\{1,3\}}_3$};
		\node [style=none] (26) at (1.5, 4.25) {$e$};
		\node [style=newstyle1] (27) at (4, 3.25) {};
		\node [style=none] (28) at (1, 3.5) {$e$};
		\node [style=newstyle1] (29) at (-3.5, 4.75) {};
		\node [style=none] (30) at (-3.5, 4) {$r$};
		\node [style=newstyle1] (31) at (-3, 4) {};
		\node [style=none] (32) at (-4.25, 3) {$r$};
		\node [style=newstyle1] (33) at (-2.25, 5) {};
		\node [style=newstyle1] (34) at (-2.75, 5.75) {};
		\node [style=newstyle1] (35) at (-3.75, 3) {};
		\node [style=none] (36) at (-3, 3.5) {$e$};
		\node [style=newstyle1] (37) at (-2.25, 5) {};
		\node [style=newstyle1] (38) at (-3, 4) {};
		\node [style=newstyle1] (39) at (-3, 4) {};
		\node [style=newstyle1] (40) at (-2.75, 5.75) {};
		\node [style=newstyle1] (41) at (-3.75, 3) {};
		\node [style=newstyle1] (42) at (-3, 4) {};
		\node [style=newstyle1] (43) at (-3.75, 3) {};
		\node [style=newstyle1] (44) at (-4.25, 3.75) {};
		\node [style=none] (45) at (-3, 4.5) {$s$};
		\node [style=newstyle1] (46) at (-3.75, 3) {};
		\node [style=newstyle1] (47) at (-4.5, 3.5) {};
		\node [style=none] (48) at (-3.75, 3.5) {$s$};
		\node [style=none] (49) at (-2.25, 4.5) {$e$};
		\node [style=newstyle1] (50) at (-3.75, 4.5) {};
		\node [style=newstyle1] (51) at (-2.25, 5) {};
		\node [style=none] (52) at (-2.75, 5) {$r$};
		\node [style=none] (53) at (-2.25, 5.5) {$s$};
		\node [style=newstyle1] (54) at (-2.75, 5.75) {};
		\node [style=newstyle1] (55) at (-2.25, 5) {};
		\node [style=newstyle1] (56) at (-2.25, 5) {};
		\node [style=newstyle1] (57) at (-3, 5.5) {};
		\node [style=newstyle1] (58) at (-2.25, 5) {};
	\end{pgfonlayer}
	\begin{pgfonlayer}{edgelayer}
		\draw [style=newstyle2] (10) to (1);
		\draw [style=newstyle2] (24) to (15);
		\draw [style=newstyle2] (6) to (13);
		\draw [style=newstyle2] (6) to (14);
		\draw [style=newstyle2] (13) to (16);
		\draw [style=newstyle2] (20) to (2);
		\draw [style=newstyle2] (27) to (23);
		\draw [style=newstyle2] (8) to (17);
		\draw [style=newstyle2] (8) to (12);
		\draw [style=newstyle2] (17) to (5);
		\draw [style=newstyle2] (41) to (44);
		\draw [style=newstyle2] (46) to (38);
		\draw [style=newstyle2] (31) to (37);
		\draw [style=newstyle2] (31) to (50);
		\draw [style=newstyle2] (37) to (40);
		\draw [style=newstyle2] (35) to (47);
		\draw [style=newstyle2] (43) to (39);
		\draw [style=newstyle2] (42) to (33);
		\draw [style=newstyle2] (42) to (29);
		\draw [style=newstyle2] (33) to (34);
		\draw [style=newstyle2] (55) to (57);
		\draw [style=newstyle2] (58) to (54);
	\end{pgfonlayer}
\end{tikzpicture}  \end{boxedminipage}
 \caption{Tree representation of the concept expressions $C_3$,
 $C^{\{1,3\}}_3$ and
 $C^{\{1,2\}}_3$.}
 \label{fig:sibling}
\end{figure}

As an answer to the first $2^{n}$ equivalence queries the oracle can
compute a positive counterexample $A \sqsubseteq \exists
e.C_{n}^{M}$ by always choosing a fresh set $M\subseteq
\{1,\ldots,n\}$. %
\end{example}
To deal with this example we introduce a modification step that
identifies siblings in $C$ rather than a parent and a child.
\begin{definition}[Sibling Merging for $\Tmc$]\label{def:sib}
  A concept $C'$ is \emph{obtained from a concept $C$ by sibling
    merging} if $C'$ is obtained from $C$ by choosing nodes $d,d',d''$
  such that $d'$ and $d''$ are $r$-successors of $d$, for some role
  $r$, and then removing $d''$, setting $l(d')=l(d')\cup l(d'')$, and
  making every $s$-successor $e$ of $d''$ in $T_{C}$ an $s$-successor
  of $d'$, for any role $s$.

  Let $A\sqsubseteq C$ be a CI with $\Tmc\models A\sqsubseteq C$.  A
  CI $A \sqsubseteq C'$ is \emph{obtained from $A \sqsubseteq C$ by
    sibling merging} if $\Tmc \models A \sqsubseteq C'$ and $C'$
  is obtained from $C$ by sibling merging.  We say that
  $A\sqsubseteq C$ is \emph{sibling merged for $\Tmc$} if there
  is no $A \sqsubseteq C'$ with $C \neq C'$ that can be obtained from
  $A \sqsubseteq C$ by sibling merging.
\end{definition}
It can be verified that when $C'$ is obtained from $C$ by sibling
merging, then $\emptyset \models C'\sqsubseteq C$.

In Example~\ref{ex:2merge} the counterexamples $A \sqsubseteq
C_{n}^{M}$ are actually sibling merged for $\Tmc_{n}$.  Thus,
producing a sibling merged $A \sqsubseteq C'$ directly from the
counterexamples returned by the oracle does not overcome the problem
illustrated by the example. Instead, we apply sibling merging after
Line~5 of the algorithm: instead of adding $A \sqsubseteq C \sqcap C'$
to $\Hmc_{add}$, the learner computes a sibling merged $A \sqsubseteq
D$ from this CI and adds it to $\Hmc_{add}$. For
Example~\ref{ex:2merge}, this is illustrated in
Figure~\ref{fig:sibling2}. Clearly, after at most $n+1$
counterexamples, the learner has added $A \sqsubseteq \exists e. C_n$, as required.

\begin{figure}[h]
 \begin{boxedminipage}[h]{\columnwidth}
 \centering
\begin{tikzpicture}
	\begin{pgfonlayer}{nodelayer}
		\node [style=none] (0) at (-1.25, 4.75) {$C^{\{1,2\}}_5$};
		\node [style=none] (1) at (-4.25, 3.5) {$r$};
		\node [style=none] (2) at (-2.75, 4.25) {$e$};
		\node [style=newstyle1] (3) at (-3.25, 5.5) {};
		\node [style=none] (4) at (-3.25, 3.5) {$e$};
		\node [style=newstyle1] (5) at (-3.75, 3.25) {};
		\node [style=none] (6) at (-3.25, 5) {$r$};
		\node [style=newstyle1] (7) at (-4.25, 4) {};
		\node [style=newstyle1] (8) at (-3.25, 4) {};
		\node [style=none] (9) at (-4.75, 4.75) {$C^{\{1,3\}}_5$};
		\node [style=newstyle1] (10) at (-2.75, 4.75) {};
		\node [style=newstyle1] (11) at (-3.25, 4) {};
		\node [style=none] (12) at (-3.75, 4.25) {$s$};
		\node [style=newstyle1] (13) at (-3.75, 3.25) {};
		\node [style=newstyle1] (14) at (-3.75, 4.75) {};
		\node [style=none] (15) at (0.75, 4.25) {$e$};
		\node [style=none] (16) at (-0.75, 3.5) {$s$};
		\node [style=newstyle1] (17) at (-0.25, 4.75) {};
		\node [style=newstyle1] (18) at (0.25, 4) {};
		\node [style=newstyle1] (19) at (-0.25, 3.25) {};
		\node [style=none] (20) at (0.25, 5) {$r$};
		\node [style=none] (21) at (-0.25, 4.25) {$r$};
		\node [style=newstyle1] (22) at (0.25, 5.5) {};
		\node [style=newstyle1] (23) at (-0.75, 4) {};
		\node [style=none] (24) at (0.25, 3.5) {$e$};
		\node [style=newstyle1] (25) at (-0.25, 3.25) {};
		\node [style=newstyle1] (26) at (0.75, 4.75) {};
		\node [style=newstyle1] (27) at (0.25, 4) {};
		\node [style=none] (28) at (1.75, 4) {};
		\node [style=none] (29) at (2.25, 4) {};
		\node [style=none] (30) at (5, 4.5) {$e$};
		\node [style=newstyle1] (31) at (4.5, 5.75) {};
		\node [style=newstyle1] (32) at (4.5, 4.25) {};
		\node [style=none] (33) at (4, 4.25) {$r$};
		\node [style=newstyle1] (34) at (3.25, 4) {};
		\node [style=none] (35) at (4.5, 3.75) {$e$};
		\node [style=newstyle1] (36) at (3.75, 4.75) {};
		\node [style=newstyle1] (37) at (3.75, 3.25) {};
		\node [style=none] (38) at (3.75, 3.75) {$s$};
		\node [style=newstyle1] (39) at (3.75, 3.25) {};
		\node [style=newstyle1] (40) at (4.5, 4.25) {};
		\node [style=newstyle1] (41) at (5, 5) {};
		\node [style=none] (42) at (5, 4.5) {$e$};
		\node [style=newstyle1] (43) at (4.5, 5.75) {};
		\node [style=newstyle1] (44) at (4.5, 4.25) {};
		\node [style=none] (45) at (3.25, 3.25) {$r$};
		\node [style=newstyle1] (46) at (3, 3.75) {};
		\node [style=newstyle1] (47) at (4, 5) {};
		\node [style=newstyle1] (48) at (3.75, 3.25) {};
		\node [style=none] (49) at (4.5, 4.75) {$s$};
		\node [style=newstyle1] (50) at (3.75, 3.25) {};
		\node [style=newstyle1] (51) at (4.5, 4.25) {};
		\node [style=newstyle1] (52) at (5, 5) {};
		\node [style=none] (53) at (4.5, 5.25) {$r$};
	\end{pgfonlayer}
	\begin{pgfonlayer}{edgelayer}
		\draw [style=newstyle2] (13) to (7);
		\draw [style=newstyle2] (5) to (8);
		\draw [style=newstyle2] (11) to (10);
		\draw [style=newstyle2] (11) to (14);
		\draw [style=newstyle2] (10) to (3);
		\draw [style=newstyle2] (19) to (23);
		\draw [style=newstyle2] (25) to (18);
		\draw [style=newstyle2] (27) to (26);
		\draw [style=newstyle2] (27) to (17);
		\draw [style=newstyle2] (26) to (22);
		\draw [style=newstyle2] (39) to (34);
		\draw [style=newstyle2] (37) to (40);
		\draw [style=newstyle2] (32) to (41);
		\draw [style=newstyle2] (32) to (36);
		\draw [style=newstyle2] (41) to (31);
		\draw [style=newstyle2] (50) to (46);
		\draw [style=newstyle2] (48) to (51);
		\draw [style=newstyle2] (44) to (52);
		\draw [style=newstyle2] (44) to (47);
		\draw [style=newstyle2] (52) to (43);
		\draw [style=newstyle6] (28.center) to (29.center);
	\end{pgfonlayer}
\end{tikzpicture}  \end{boxedminipage}
 \caption{Sibling Merging $C^{\{1,3\}}_3$ and
 $C^{\{1,2\}}_3$ for $\Tmc$.}
 \label{fig:sibling2}
\end{figure}
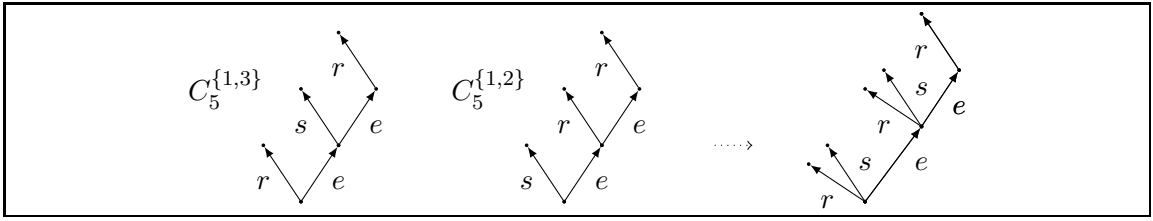

Finally, we need a decomposition rule. The following variant of
Example~\ref{ex:csat} illustrates that the four modification steps
introduced so far do not yet lead to a polynomial learning algorithm
even if they are applied both after Line~4 and after Line~5 in
Algorithm~\ref{alg:dl-lite-naive}.
\begin{example}\label{ex:csatmore}
Let 
$$
    \Tmc=\{A\sqsubseteq B, B \sqsubseteq \exists r.B\}.
$$ 
The oracle can provide for the $n$-th equivalence query the positive counterexample
$A\sqsubseteq 
C_{B,n}$, where $C_{B,n}= A \sqcap D_{B,n}$ and, inductively, $D_{B,0}= B$ and $D_{B,n+1}= B \sqcap \exists r.D_{B,n}$,
for any $n \geq 0$.
The algorithm does not terminate even with the four modification steps introduced above applied after Lines~4 and 5:
the CIs $A \sqsubseteq C_{B,n}$ are concept and role saturated and they are parent/child and sibling merged.
\end{example}
The problem illustrated in Example~\ref{ex:csatmore} is that so far the
learning algorithm attempts to learn $\Tmc$ without ever considering
to add to $\Hmc_{add}$ a CI whose left-hand side is $B$ (rather than
$A$). To deal with this problem we introduce a `reset step' that, in
contrast to the previous modification steps, can lead to a different
left-hand side and also to a CI that does not imply the original CI
given \Tmc, as in all previous modification steps.
\begin{definition}[Decomposed CI for $\Tmc$]\label{def:decom}
  Let $A\sqsubseteq C$ be a CI with $\Tmc\models A\sqsubseteq C$.  We
  say that $A\sqsubseteq C$ is \emph{decomposed for} $\Tmc$ if for
  every non-root node $d$ in $C$, every concept name $A' \in l(d)$,
  and every $r$-successor $d'$ of $d$ in $C$, we have $\Tmc
  \not\models A' \sqsubseteq \exists r.C'$ where $C'$ corresponds to
  the subtree of $C$ rooted at $d'$.
\end{definition}
In contrast to the previous four modification steps, the membership
queries used by the learner to obtain a decomposed CI do not only depend on $\Tmc$
but also on the hypothesis $\Hmc_{add} \cup \Hmc_{basic}$ computed up
to that point: starting from CI $A \sqsubseteq C$, the learner takes a
non-root node $d$ in $C$, a concept name $A' \in l(d)$, and an
$r$-successor $d'$ of $d$ in $C$, and then checks using a membership
query whether $\Tmc\models A' \sqsubseteq \exists r.C'$, where $C'$ is
the subtree rooted at $d'$ in $C$. If the check succeeds, 
$A \sqsubseteq C$
is replaced by
\begin{itemize}
\item[(a)] $A'\sqsubseteq \exists r.C'$ if $\mathcal{H}_{basic}\cup \mathcal{H}_{add}\not\models A'\sqsubseteq \exists r.C'$; and otherwise by
\item[(b)] $A \sqsubseteq C|^-_{d'\downarrow}$, where $C|^-_{d'\downarrow}$ is obtained from $C$ by
removing the subtree rooted in $d'$ from~$C$.
\end{itemize}  
Note that
$\{ A \sqsubseteq C|^-_{d'\downarrow}, A' \sqsubseteq \exists r.C'
\}\models A \sqsubseteq C.  $
Thus, one of the CIs $A \sqsubseteq C|^-_{d'\downarrow}$ and
$A' \sqsubseteq \exists r.C'$ is not entailed by
$\mathcal{H}_{basic}\cup \mathcal{H}_{add}$, and this is the CI that
replaces the original CI.

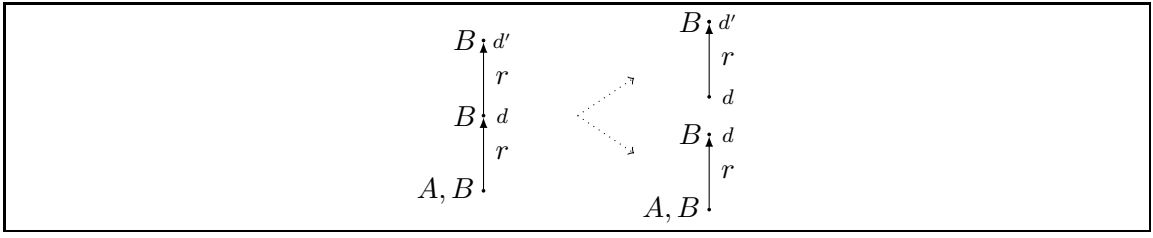
\begin{figure}
\begin{boxedminipage}[t]{\columnwidth}
\center
\begin{tikzpicture}
	\begin{pgfonlayer}{nodelayer}
		\node [style=newstyle1] (0) at (-2.5, 3.5) {};
		\node [style=newstyle1] (1) at (-2.5, 1.5) {};
		\node [style=newstyle1] (2) at (-2.5, 2.5) {};
		\node [style=none] (3) at (-2.5, 3.5) {};
		\node [style=none] (4) at (-2.25, 2) {$r$};
		\node [style=none] (5) at (-2.75, 2.5) {$B$};
		\node [style=none] (6) at (-3, 1.5) {$A, B$};
		\node [style=none] (7) at (-2.75, 3.5) {$B$};
		\node [style=none] (8) at (-2.25, 3.5) {\scalefont{0.7}{$d'$}};
		\node [style=none] (9) at (-2.25, 2.5) {\scalefont{0.7}{$d$}};
		\node [style=none] (10) at (-1.25, 2.5) {};
		\node [style=none] (11) at (-0.5, 3) {};
		\node [style=none] (12) at (0.75, 2.75) {\scalefont{0.7}{$d$}};
		\node [style=newstyle1] (13) at (0.5, 2.75) {};
		\node [style=none] (14) at (0.5, 3.75) {};
		\node [style=newstyle1] (15) at (0.5, 3.75) {};
		\node [style=none] (16) at (0.75, 3.25) {$r$};
		\node [style=none] (17) at (0.5, 3.75) {};
		\node [style=newstyle1] (18) at (0.5, 2.25) {};
		\node [style=none] (19) at (0.25, 2.25) {$B$};
		\node [style=none] (20) at (0.75, 2.25) {\scalefont{0.7}{$d$}};
		\node [style=none] (21) at (0, 1.25) {$A, B$};
		\node [style=newstyle1] (22) at (0.5, 1.25) {};
		\node [style=none] (23) at (0.25, 3.75) {$B$};
		\node [style=none] (24) at (0.5, 3.75) {};
		\node [style=none] (25) at (0.75, 3.75) {\scalefont{0.7}{$d'$}};
		\node [style=newstyle1] (26) at (0.5, 3.75) {};
		\node [style=none] (27) at (-2.25, 3) {$r$};
		\node [style=none] (28) at (0.75, 1.75) {$r$};
		\node [style=none] (29) at (-0.5, 2) {};
	\end{pgfonlayer}
	\begin{pgfonlayer}{edgelayer}
		\draw [style=newstyle2] (2) to (0);
		\draw [style=newstyle3] (10.center) to (11.center);
		\draw [style=newstyle2] (13) to (15);
		\draw [style=newstyle2] (1) to (2);
		\draw [style=newstyle2] (22) to (18);
		\draw [style=newstyle3] (10.center) to (29.center);
	\end{pgfonlayer}
\end{tikzpicture}
 \end{boxedminipage}
\caption{Illustration of decomposition of CIs. Here $C|^-_{d'\downarrow} = A\sqcap B\sqcap \exists r.B$ and 
$C' = B$.}
\label{fig:mini}
\end{figure}
In Example~\ref{ex:csatmore}, assume that the oracle returns $A \sqsubseteq
C$ with $C=A \sqcap B \sqcap \exists r.(B \sqcap \exists r.B)$ as the
first counterexample. The tree $T_{C}$ corresponding to $C$ is shown
on the left-hand side of Figure~\ref{fig:mini}. This CI is %
not decomposed for $\Tmc$: the label of node
$d$ contains $B$, the concept $C'$ rooted in $d'$ in $C$ is $B$ and
$\Tmc\models B\sqsubseteq \exists r.B$.  Since $\Hmc\cup
\Hmc_{add}\not\models B\sqsubseteq \exists r.B$, Case (a) applies and
$A \sqsubseteq B$ is replaced by $B\sqsubseteq \exists r.B$.

\smallskip This finishes the description of the modification steps. It
turns out that they cure all problems with the initial version of the
algorithm and enable polynomial query learnability. 
\begin{definition}
A \ourDLLite CI is \emph{$\Tmc$-essential} if it is concept saturated, role saturated,
parent/child merged, sibling merged, and decomposed for $\Tmc$.
\end{definition}
After Lines~4 and~5 of Algorithm~\ref{alg:dl-lite-naive}, we need to make
the CI currently considered $\Tmc$-essential, by exhaustively
applying the modification steps described above in all possible orders. The resulting refined
version of the learning algorithm is shown as
Algorithm~\ref{alg:dl-lite}. We next analyse the properties of this
algorithm.

\newcommand{\hai}{\hskip\algorithmicindent}
\begin{algorithm}[t]
\begin{algorithmic}[1]
\Require A  \ourDLLite TBox $\Tmc$ in named form given to the oracle; $\Sigma_\Tmc$ given to the learner.
\Ensure  TBox $\Hmc$, computed by the learner, such that $\Tmc\equiv\Hmc$.
\State Compute   
$\mathcal{H}_{basic}  =  \{r\sqsubseteq s \mid 
\Tmc \models r\sqsubseteq s\}  \cup  
  \{B_{1} \sqsubseteq B_{2} \mid 
\Tmc \models B_{1}\sqsubseteq B_{2}, \mbox{ $B_{1},B_{2}$ basic}\}$ 
\State Set $\mathcal{H}_{add}=\emptyset$
\While{$\mathcal{H}_{basic}\cup \mathcal{H}_{add}\not\equiv\Tmc$}
\State Let $A \sqsubseteq C$ be the returned positive counterexample
 for $\mathcal{T}$ relative
to $\mathcal{H}_{basic}\cup \mathcal{H}_{add}$ 
\State Find a $\Tmc$-essential CI $A' \sqsubseteq C'$ such that
  $\mathcal{H}_{basic}\cup \mathcal{H}_{add}\not\models A' \sqsubseteq C'$ \hfill 
\If {there is $A'\sqsubseteq C''\in \mathcal{H}_{add}$}
\State Find $\Tmc$-essential CI $A' \sqsubseteq C^{\ast}$ such that 
  $\emptyset \models C^{\ast}\sqsubseteq C'' \sqcap C'$ \hfill 
\State Replace $A' \sqsubseteq C''$ by $A' \sqsubseteq C^{\ast}$ in $\mathcal{H}_{add}$
\Else
\State Add $A' \sqsubseteq C'$ to $\mathcal{H}_{add}$
\EndIf
\EndWhile
\State \textbf{return} $\Hmc=\mathcal{H}_{basic} \cup \mathcal{H}_{add}$
\end{algorithmic}
\caption{The learning algorithm for \ourDLLite\label{alg:dl-lite}}
\end{algorithm}

\subsection*{Polynomial Query Bound on the Algorithm}
If Algorithm \ref{alg:dl-lite} terminates, then it obviously has found a TBox
$\mathcal{H}_{basic}\cup \mathcal{H}_{add}$ that is logically equivalent to
$\Tmc$.  It thus remains to show that the algorithm terminates after
polynomially many polynomial size queries. Observe that $\mathcal{H}_{add}$
contains at most one CI $A\sqsubseteq C$ for each concept name $A$.  At each
step in the while loop, either some $A'\sqsubseteq C'$ is added to $\Hmc_{add}$
such that no CI with $A'$ on the left-hand side existed in $\Hmc_{add}$ before
(Line~10) or an existing CI $A' \sqsubseteq C''$ in $\Hmc_{add}$ is replaced by
a fresh CI $A'\sqsubseteq C^{\ast}$ with $\emptyset\models C^{\ast} \sqsubseteq
C''$.

We start with showing that Lines~5 and~7 can be implemented with polynomially many 
membership queries.
The next lemma addresses Line~5.
\begin{lemma}\label{lem:t-essential0}
Given a positive counterexample $A \sqsubseteq C$
for $\Tmc$ relative to 
$\Hmc_{basic}\cup \Hmc_{add}$, one can construct a $\Tmc$-essential 
counterexample $A' \sqsubseteq C'$
using only polynomially many polynomial size membership queries in $|C|+|\Tmc|$.
\end{lemma}
\begin{proof}
Let $A \sqsubseteq C$ be a positive counterexample for $\Tmc$ relative to $\Hmc_{basic}\cup \Hmc_{add}$
and assume the five modification steps introduced above are applied exhaustively
by posing membership queries to the oracle. Observe that the number of applications
of modifications steps is bounded polynomially in $|C|\times |\Tmc|$. To show this, let $n_{C}$ be the number of 
nodes in $T_{C}$. Then $n_{C''}=n_{C'}$ if $A'' \sqsubseteq C''$ 
is obtained from $A' \sqsubseteq C'$ by a concept or role saturation step and $n_{C''}<n_{C'}$ if 
$A''\sqsubseteq C''$ is obtained from $A'\sqsubseteq C'$ by a merging or decomposition
step. Thus, the number of applications of merging and decomposition steps is bounded by $n_{C}$
and the number of applications of concept and role saturated steps is bounded by 
$|\Sigma_{\Tmc}|\times n_{C}$ and $|\Sigma_{\Tmc}|\times n_{C}^{2}$, respectively.
Thus, after at most $n_{C} + |\Sigma_{\Tmc}|\times n_{C} + |\Sigma_{\Tmc}|\times n_{C}^{2}$
steps no modification step is applicable and the final CI is $\Tmc$-essential. 
We verify that it is also a positive counterexample for \Tmc relative to
$\mathcal{H}_{basic}\cup \mathcal{H}_{add}$. It suffices to show that
the CI resulting from each single modification step is entailed by
\Tmc, but not by $\mathcal{H}_{basic}\cup \mathcal{H}_{add}$. The
former has been shown when we introduced the modification steps. Regarding the
latter, in the first four modification steps we have
$\Hmc_{basic}\models C' \sqsubseteq C$ if $A \sqsubseteq C$ is
replaced by $A\sqsubseteq C'$.  Hence
$\mathcal{H}_{basic}\cup \mathcal{H}_{add}\not\models A\sqsubseteq
C'$.
For the decomposition step, we have already argued,
after Definition~\ref{def:decom}, that the added CI
is not entailed by $\mathcal{H}_{basic}\cup \mathcal{H}_{add}$. 
\end{proof}

The following lemma addresses Line~7. %
\begin{lemma}\label{lem:dllite0}
Assume that $A\sqsubseteq C_{1}$ and $A\sqsubseteq C_{2}$ are $\Tmc$-essential.
Then one can construct a $\Tmc$-essential $A \sqsubseteq C$ such that $\emptyset \models C\sqsubseteq C_{1}
\sqcap C_{2}$ using polynomially many polynomial size membership queries in $|C_{1}|+|C_{2}|$.
\end{lemma}
\begin{proof}
  We start with $A\sqsubseteq C_{1}\sqcap C_{2}$. Using the fact that
  $C_{1}$ and $C_{2}$ are both \Tmc-essential, one can show that this
  CI is (i)~\erightcondition, (ii)~\erolecondition, (iii) parent/child
  merged for $\Tmc$, and (iv) decomposed for $\Tmc$. Assume, for example, that $A \sqsubseteq C_{1}\sqcap C_{2}$
  is not concept saturated. Then one can add a new concept name $A'$ to the label $l(d)$ for
  some node $d$ in $T_{C_{1}\sqcap C_{2}}$ and $\Tmc\models A\sqsubseteq C'$ for the resulting concept $C'$.
  Clearly $d$ is a node in $T_{C_{1}}$ or in $T_{C_{2}}$. Assume without loss of generality that 
  $d$ is in $T_{C_{1}}$ and let $C_{1}'$ be the concept obtained from $C_{1}$ by adding $A'$ to $l(d)$.
  Then $\Tmc\models A \sqsubseteq C_{1}'$ since $\Tmc\models A \sqsubseteq C'$ which contradicts the assumption
  that $A \sqsubseteq C_{1}$ is concept saturated. The remaining three modification steps are considered 
  similarly. We now exhaustively apply the modification
  step \esiblingrule and use the resulting CI as the desired
  $A\sqsubseteq C$. Similarly to the argument above one can show that a CI with properties
  (i)--(iv) still has those properties after applying sibling merging.
  Thus, $A \sqsubseteq C$ is $\Tmc$-essential. We have argued in the proof of 
  Lemma~\ref{lem:t-essential0} already that the number of applications of a sibling merging step to
  a CI of the form $A\sqsubseteq D$ is bounded by the number of nodes in $T_{D}$. Thus, the number of
  modification steps is bounded polynomially in $|C_{1}|+|C_{2}|$.
\end{proof}

To analyse the algorithm further we first prove a polynomial upper
bound on the size of $\Tmc$-essential CIs. To this end, we require the
notion of an isomorphic embedding and an auxiliary lemma. A
homomorphism $h:T_{C} \rightarrow \Imc$ is an \emph{isomorphic
  embedding for $\Tmc$} if it is injective, $A\in l(d)$ if
$h(d)\in A^{\Imc}$ for all concept names $A$, and for $r=l(d,d')$ it
holds that $\Tmc\models r\sqsubseteq s$ for all
$(h(d),h(d'))\in s^{\Imc}$.  The following lemma shows that for
$\Tmc$-essential CIs $A \sqsubseteq C$ and any $D$ that interpolates
between $A$ and $C$ (meaning that $\Tmc\models A \sqsubseteq D$ and
$\Tmc\models D \sqsubseteq C$) the homomorphism $h$ from $T_{C}$ to
$\Imc_{D,\Tmc}$ that witnesses $\rho_{D}\in C^{\Imc_{D,\Tmc}}$ (see
Lemma~\ref{lem:can1}) is an isomorphic embedding.
\begin{lemma}\label{lem:isoemb}
Assume the $A \sqsubseteq C$ is $\Tmc$-essential, $\Tmc\models A \sqsubseteq D$ and $\Tmc \models D \sqsubseteq C$. 
Then any homomorphism $h: T_{C}\rightarrow \Imc_{D,\Tmc}$ that maps $\rho_{C}$ to $\rho_{D}$ is an isomorphic embedding.
\end{lemma}
\begin{proof}
  Assume first that $h$ is not injective. Then there is a parent/child
  or sibling merging $C'$ of $C$ and a
  homomorphism $f:T_{C'} \rightarrow \Imc_{D,\Tmc}$ such that
  $h=f\circ g$ for the natural homomorphism
  $g:T_{C}\rightarrow T_{C'}$.  By Lemma~\ref{lem:can1},
  $\Tmc \models D \sqsubseteq C'$. Thus,
  $\Tmc\models A \sqsubseteq C'$ and we have derived a contradiction
  to the assumption that $C$ is parent/child and sibling merged.  

  Now let $T'$ be the following labelled tree: the nodes in $T'$ are
  the same as in $T$, $A\in l'(d)$ iff $h(d)\in A^{\Imc_{D,\Tmc}}$,
  and for any two nodes $d,d'$ with $d'$ a successor of $d$:
  $l'(d,d')=r$ for the unique role $r$ with
  $\Tmc\models r\sqsubseteq s$ for all $s$ with
  $(d,d')\in s^{\Imc_{D,\Tmc}}$. Let $C'$ be the concept
  expression that corresponds to $T'$. Then
  $\rho_{D}\in {C'}^{\Imc_{D,\Tmc}}$ and so, by
  Lemma~\ref{lem:can1}, $\Tmc\models D \sqsubseteq C'$.  Thus
  $\Tmc\models A \sqsubseteq C'$. But then $A\sqsubseteq C'$ can be
  obtained from $A \sqsubseteq C$ by concept and role saturation
  steps.  As $A\sqsubseteq C$ is concept and role saturated already,
  $C=C'$ and so $h$ is an isomorphic embedding.
\end{proof}
We are now able to prove that $\Tmc$-essential CIs are of polynomial
size in $\Tmc$. Let $n_C$ denote the number of nodes in the tree
representation $T_{C}$ of $C$ and let
$$
A^{\Tmc}= \{A\} \cup \{ B \mid A\sqsubseteq B\in \mathcal{H}_{basic}\} \cup \{D\mid \Tmc\models A\sqsubseteq B, B\sqsubseteq D\in \Tmc\}.
$$
\begin{lemma}\label{lem:size}
If $A \sqsubseteq C$ is $\Tmc$-essential, then $\displaystyle n_C \leq \!\!\sum_{D \in A^{\Tmc}} n_D$.
\end{lemma}
\begin{proof}
  Assume $A \sqsubseteq C$ is $\Tmc$-essential.  Let
  $D_0:= \bigsqcap_{D \in A^{\Tmc}}D$ and $\Imc_{D_0,\Tmc}$ be the
  canonical model of $D_0$ and \Tmc. By Lemma~\ref{lem:can1}, there is
  a homomorphism $h: T_{C} \rightarrow \Imc_{D_{0},\Tmc}$ mapping
  $\rho_{C}$ to $\rho_{D_{0}}$. By Lemma~\ref{lem:isoemb}, $h$ is
  an isomorphic embedding.  Using that $A\sqsubseteq C$ is decomposed
  for $\Tmc$ we now show that $h$ maps $T_{C}$ into the restriction of
  $\Imc_{D_{0}}$ to $\Delta^{\Imc_{D_{0}}}$ from which the lemma
  follows since $h$ is injective.  

  For a proof by contradiction, assume that there exists $d'$ in
  $T_{C}$ with $h(d')\not\in \Delta^{\Imc_{D_0}}$.  As
  $h(\rho_{C})\in \Delta^{\Imc_{D_{0}}}$, we may assume that all $d'$
  on the path from $\rho_{C}$ to $d'$ are mapped to
  $\Delta^{\Imc_{D_0}}$.  In particular, the parent $d$ of $d'$ in
  $T_{C}$ is mapped into $\Delta^{\Imc_{D_0}}$. Let $l(d,d')= r$.
  Observe that the whole subtree rooted in $d'$ must be mapped into
  $\Delta^{\Imc_{D_0,\Tmc}}\setminus\Delta^{\Imc_{D_0}}$ since
  otherwise $h$ would not be injective.

  Let $C'= \exists r.C''$, where $C''$ corresponds to the subtree
  rooted in $d'$ in $C$.  By Lemma~\ref{lem:can2} %
  there exists a basic concept $B$ such that
  $h(d)\in B^{\Imc_{D_0,\Tmc}}$ and $\Tmc\models B \sqsubseteq C'$. As
  $\Tmc$ is in named form there exists a concept name $E$ with
  $\Tmc\models E \equiv B$. Thus $h(d)\in E^{\Imc_{D_0,\Tmc}}$ and
  $\Tmc\models E \sqsubseteq C'$. As $h$ is an isomorphic embedding,
  $E\in l(d)$. We make a case distinction:
\begin{itemize}
\item $h(d)\not=\rho_{D_0}$. Then $A \sqsubseteq C$ is not decomposed for $\Tmc$ since $C$ contains an edge $(d,d')$
such that $E$ is in the node label of $d$, $l(d,d')=r$, and $\Tmc\models E \sqsubseteq \exists r.C''$. We have derived
a contradiction.
\item $h(d)=\rho_{D_0}$.
  As $h(d')\in \Delta^{\Imc_{D_0,\Tmc}}\setminus \Delta^{\Imc_{D_0}}$,
  by construction of $\Imc_{D_0,\Tmc}$, there exists a CI
  $B_{0} \sqsubseteq D\in \Tmc$ with $B_{0}$ a basic concept such that
  $\rho_{D_0}\in B_{0}^{\Imc_{D_0,\Tmc}}$ and $h(d')$ is in the copy
  of the tree-shaped interpretation $\Imc_{D}$ which was attached to
  $\rho_{D_0}$ in the construction of $\Imc_{D_0,\Tmc}$. But since
  $\Tmc\models A \sqsubseteq B_{0}$ we have $D\in A^{\Tmc}$ and so
  $\rho_{D_{0}}\in D^{\Imc_{D_{0}}}$. But then, by the
  construction of $\Imc_{D_{0},\Tmc}$, no fresh $\Imc_{D}$ was
  attached to $\rho_{D_0}$ because $D$ is already satisfied in
  $\Imc_{D_{0}}$ and we have derived a contradiction.
\end{itemize}
\end{proof}
We are now in the position to prove that the learning algorithm terminates after posing a polynomial number of queries.
\begin{lemma}
\label{lem:newpolylem}
For every concept name $A$, the number of replacements of a CI $A
\sqsubseteq C$ in $\Hmc_{add}$ by a CI of the form $A \sqsubseteq C'$
is bounded polynomially in $|\Tmc|$.
\end{lemma}
\begin{proof}
  All CIs $A \sqsubseteq C$ ever added to $\Hmc_{add}$ are
  \Tmc-essential. We show that when $A\sqsubseteq C$ is replaced with
  $A\sqsubseteq C'$, then the number of nodes in the tree
  representation of $C'$ is strictly larger than the number of nodes
  in the tree representation of $C$. By Lemma~\ref{lem:size}, the
  number of replacements is thus bounded by $\sum_{D \in A^{\Tmc}}
  n_D$, which is polynomial in $|\Tmc|$.
  
  Note that when $A \sqsubseteq C$ is replaced with $A\sqsubseteq C'$,
  then $\emptyset \models C'\sqsubseteq C$ and $\emptyset \not\models
  C \sqsubseteq C'$. It thus suffices to
  establish the following.

\medskip 
\noindent {\bf Claim}. If $A \sqsubseteq C$ and $A\sqsubseteq C'$ are
$\Tmc$-essential, $\emptyset \models C'\sqsubseteq C$ and $\emptyset
\not\models C \sqsubseteq C'$, then $T_{C}$ is obtained from $T_{C'}$
by removing at least one subtree.

\medskip
\noindent
We prove the claim. Since $\emptyset \models C' \sqsubseteq C$, by Lemma~\ref{lem:can1} there is a homomorphism $h$ from $T_{C}$ to 
the canonical model $\Imc_{C'}$ that maps $\rho_{C}$ to $\rho_{C'}$. Then $h$ is also a homomorphism into the canonical model 
$\Imc_{C',\Tmc}$ and thus, by Lemma~\ref{lem:isoemb}, $h$ is an isomorphic embedding into $\Imc_{C',\Tmc}$. Then, trivially,
$h$ is also an isomorphic embedding into $\Imc_{C'}$ which means that
$T_{C}$ is obtained from $T_{C'}$ by removing subtrees. Since $\emptyset
\not\models C \sqsubseteq C'$, at least one subtree must in fact have
been
removed.
\end{proof}
We have obtained the following main result of this section.
\begin{theorem}\label{DL-Lite}
\ourDLLite TBoxes are polynomial query learnable using membership and equivalence queries.
Moreover, \ourDLLite TBoxes without inverse roles can be learned in polynomial time using
membership and equivalence queries. 
\end{theorem}

\begin{proof}
  Recall that our algorithm requires the target TBox to be in named
  form. We first show Theorem~\ref{DL-Lite} under that assumption and
  then argue that the assumption can be dropped.

  In each iteration of Algorithm~\ref{alg:dl-lite}, either a CI is
  added to $\Hmc_{add}$ or a CI is replaced in $\Hmc_{add}$. Since the
  number of times the former happens is bounded by $|\Sigma_\Tmc|$ and
  (by Lemma~\ref{lem:newpolylem}) the number of times the latter
  happens is polynomial in $|\Tmc|$, the number
  of %
  iterations of Algorithm~\ref{alg:dl-lite} is polynomial in $|\Tmc|$.
  For polynomial query learnability of \ourDLLite TBoxes, it remains
  to show that
in each iteration Algorithm~\ref{alg:dl-lite} makes only polynomially many 
polynomial size queries in $|\Tmc|$ and the size of the largest counterexample 
seen so far.  We start with equivalence queries, made only in Line~3.
We have already argued that the number of iterations is polynomial in
$|\Tmc|$ and thus so is the number of equivalence queries made.
Regarding their size, we observe that there are at most
$|\Sigma_\Tmc|^2$ CIs in $\Hmc_{basic}$ and at most $|\Sigma_\Tmc|$
CIs in $\Hmc_{add}$, that the size of CIs in $\Hmc_{basic}$ is
constant and by Lemma~\ref{lem:size} the size of CIs in $\Hmc_{add}$
is polynomial in $|\Tmc|$.  Membership queries are made only in
Lines~5 and 7 for which it suffices to invoke
Lemmas~\ref{lem:t-essential0} and~\ref{lem:dllite0}.

Now for the ``moreover'' part of Theorem~\ref{DL-Lite}. Observe that
since each (membership or equivalence) query counts as one step of
computation, the only potentially costly step of
Algorithm~\ref{alg:dl-lite} is the implementation of the decomposition
step in Line 5, which relies on making subsumption checks of the form
$\Hmc_{basic} \cup \Hmc_{add}\models A\sqsubseteq C$. As discussed in
Section~\ref{sect:prelim}, deciding subsumption in \ourDLLite is
\NP-complete while in \EL with role inclusions subsumption is in
\PTime.  As \ourDLLite without inverse roles is a fragment of \EL with
role inclusions, we obtain polynomial time learnability for TBoxes in
this case.

To drop the requirement that the target TBox is in named form, we show
that any polynomial (query or time) learning algorithm for TBoxes in
named form can be transformed into the same kind of algorithm for
unrestricted target TBoxes. In fact, the learner can use at most
$\Omc(|\Sigma_\Tmc|^{2})$ membership queries ``Does $\Tmc$ entail $r
\sqsubseteq s$?'' to compute for every role $r$ the class $[r]_{\Tmc}$
of roles $s$ with $\Tmc\models s\equiv r$ and choose
a representative $r_{\Tmc}$ for this class. Then whenever some $s\in [r]_\Tmc$
is used in any counterexample returned by the oracle, it gets replaced with
$r_\Tmc$.  Likewise, whenever $\Tmc$ does not have a name for some $\exists
r.\top$, the algorithm still uses the concept name $A_r$ in its internal
representations (although they are no longer included in the signature
$\Sigma_\Tmc$ of the target TBox) and replaces $\exists r. \top$ with $A_r$ in
the counterexamples returned by the oracle.  It also replaces each $A_r$ with
$\exists r. \top$ in membership queries to the oracle and in the hypothesis
used for posing equivalence queries. 
\end{proof}

\section{Learning \ourDLLitehorn TBoxes}
\label{sec:elih-rhs-horn-subsumption}
 
We study exact learnability of TBoxes in \ourDLLitehorn, the extension
of \ourDLLite that admits conjunctions of basic concepts on the
left-hand side of CIs.
This language is a generalisation of both $\ourDLLite$ and   propositional Horn
logic. 
In fact, the algorithm we present combines the classical algorithms
for propositional Horn logic
\citep{DBLP:journals/ml/AngluinFP92,DBLP:conf/icml/FrazierP93} with
the algorithm for $\ourDLLite$ presented in Section~\ref{sec:dllite}.
The resulting algorithm is quite subtle and indeed this is the reason
why we treated the \ourDLLite case separately in Section~\ref{sec:dllite}.

To simplify the presentation, we make the same assumptions as in
Section~\ref{sec:dllite} about the target TBox \Tmc with signature
$\Sigma_\Tmc$. 
In particular, we assume that $\Tmc$ is in named form, suitably
generalised to \ourDLLitehorn: there are no distinct roles $r$ and $s$
such that $\Tmc\models r\equiv s$, for each role $r$ the TBox $\Tmc$
contains an equivalence $A_r\equiv \exists r.\top$ and all CIs of
$\Tmc$ are either CIs between basic concepts or contain no concept
expressions of the form $\exists r.\top$ on the left-hand side, for
any role $r$.  Denote by ${\sf lhs}(\alpha)$ the set of concept names
that occur as conjuncts on the left-hand side of a CI $\alpha$ and
denote by ${\sf rhs}(\alpha)$ the set of concept expressions that
occur as top-level conjuncts on the right-hand side of $\alpha$ (that
is, they are not nested inside restrictions).
We often do not distinguish between the set ${\sf lhs}(\alpha)$ and the
conjunction over all its concept expressions, and similarly for
${\sf rhs}(\alpha)$. For example, if $\alpha_1 =C_1\sqsubseteq D_1$ and $\alpha_2=C_2\sqsubseteq D_2$
then ${\sf lhs}(\alpha_1)\sqsubseteq {\sf rhs}(\alpha_2)$ stands for
$C_1\sqsubseteq D_2$. Also, if ${\sf lhs}(\alpha_1) =\{A_1, A_2, A_3\}$ and
${\sf lhs}(\alpha_2) = \{A_2, A_3, A_4\}$ then
${\sf lhs}(\alpha_1) \cap {\sf lhs}(\alpha_2)\sqsubseteq D$
stands for $A_2\sqcap A_3\sqsubseteq D$.

The algorithm for learning \ourDLLitehorn TBoxes is shown as
Algorithm~\ref{alg:dl-lite-horn}.
Like Algorithm~\ref{alg:dl-lite},  Algorithm~\ref{alg:dl-lite-horn} first
determines the set $\Hmc_{basic}$ that contains all CIs $B_1\sqsubseteq B_2$ with
$B_{1},B_{2}$ basic concepts such that $\Tmc\models B_1\sqsubseteq B_2$ and all
RIs $r\sqsubseteq s$ such that $\Tmc\models r\sqsubseteq s$. 
The hypothesis \Hmc is the union of $\Hmc_{basic}$  and $\Hmc_{add}$. In
contrast to Algorithm~\ref{alg:dl-lite}, $\Hmc_{add}$ is an \emph{ordered list}
of CIs rather than a set. We write $\alpha_i$ to denote the CI $\alpha$ at
position $i$ in the list $\Hmc_{add}$.  In the learning algorithm, working with
an ordered list of CIs allows the learner to pick the first $\alpha_{i}$ in
$\Hmc_{add}$ with a certain property and merge it with a new CI, a technique we
adopt from~\citep{DBLP:journals/ml/AngluinFP92,DBLP:conf/icml/FrazierP93}.
As in Algorithm~\ref{alg:dl-lite}, $\Tmc\models \Hmc$
is a loop invariant, thus, $\gamma$ is necessarily positive. 
The algorithm terminates when $\Hmc\equiv\Tmc$.

\begin{algorithm}[t]
\begin{algorithmic}[1]
\Require A \ourDLLitehorn TBox $\Tmc$ in named form given to the oracle; $\Sigma_\Tmc$ given to the learner.
\Ensure  TBox $\Hmc$, computed by the learner, such that $\Tmc\equiv\Hmc$.
\State Compute   
$\mathcal{H}_{basic}  =  \{r\sqsubseteq s \mid 
\Tmc \models r\sqsubseteq s\}  \cup  
  \{B_{1} \sqsubseteq B_{2} \mid 
\Tmc \models B_{1}\sqsubseteq B_{2}, \mbox{ $B_{1},B_{2}$ basic}\}$  
\State Set $\Hmc_{add}$ to be the empty list and $\Hmc = \Hmc_{basic}   \cup \Hmc_{add}$
\While {$\Hmc \not\equiv \Tmc$} \label{alg:line-main-loop}
\State Let $\gamma$ be the returned  positive counterexample for
\Tmc and \Hmc \label{alg:line-pce}
\State Find a \Tmc-essential $\gamma'$ with $\Hmc\not\models\gamma'$
and  $|\{\exists r.F\mid \exists r.F\in
{\sf rhs}(\gamma')\}|\leq 1$
\label{line:at-most-one}
\State Left saturate $\gamma'$ for \Hmc \label{alg:line-t-essential1}
\If{there is  $A\in \NC$ 
such that  
$\Tmc \models {\sf lhs}(\gamma')\sqsubseteq A$ and $\Hmc \not\models
{\sf lhs}(\gamma')\sqsubseteq A$}\label{alg:line-condition2}
\State $\Hmc:=$\Call{CN-Refine}{$\Hmc,{\sf lhs}(\gamma')\sqsubseteq A$}
\Else
\State $\Hmc:=$\Call{$\exists$-Refine}{$\Hmc,\gamma'$}
\label{alg:line-exists}
\EndIf
\State Set $\Hmc = \Hmc_{basic} \cup  \Hmc_{add}$
\EndWhile 
\State \textbf{return} $\Hmc$
\end{algorithmic}
\caption{The learning algorithm for \ourDLLitehorn TBoxes\label{alg:dl-lite-horn}}
\end{algorithm} 

\begin{algorithm} [h]
\begin{algorithmic}[1]
\If{there is  $A\in \NC$ and $\alpha_i \in \Hmc_{add}$
such that  $\Tmc\models{\sf lhs}(\alpha_i)\cap  {\sf lhs}(\gamma)
\sqsubseteq A$, \ 
\\
$\quad\quad\quad$
 and $\Hmc \not\models {\sf lhs}(\alpha_i)\cap  {\sf lhs}(\gamma)
\sqsubseteq A$}\label{alg:line-condition2}
\State Concept saturate $\gamma'={\sf lhs}(\alpha_i)\cap  {\sf lhs}(\gamma)
\sqsubseteq A$ for \Tmc 
\State Replace the first such $\alpha_i$ in $\Hmc_{add}$ by
$\gamma'$ \label{alg:line-replacement1}
\Else
\State Concept saturate $ \gamma $ for \Tmc
\State Append $ \gamma $ to the list $\Hmc_{add}$ \label{alg:line-append0}
\EndIf
\State \textbf{return} $\Hmc$
\end{algorithmic}
\caption{\textbf{Function} CN-Refine($\Hmc,\gamma$) \label{alg:refine-left}}
\end{algorithm}

\begin{algorithm} [h]
\begin{algorithmic}[1]
\If{there is  $C\in {\sf rhs}(\gamma)$ of the form $\exists r.D$ and $\alpha_i \in \Hmc_{add}$
such that  $\Tmc \models  {\sf lhs}(\alpha_i)\cap  {\sf lhs}(\gamma) 
\sqsubseteq C$
\\
$\quad\quad\quad$
 and $\Hmc \not\models  {\sf lhs}(\alpha_i)\cap  {\sf lhs}(\gamma) 
\sqsubseteq C$ }
\If{$\Tmc \models {\sf lhs}(\alpha_i)\cap  {\sf lhs}(\gamma) 
\sqsubseteq C \sqcap {\sf rhs}(\alpha_i)$}\label{alg:line-condition22} 
\State Find a \Tmc-essential ${\sf lhs}(\alpha_i)\cap  {\sf lhs}(\gamma)  \sqsubseteq D^{\ast}$ 
with $\emptyset \models D^{\ast} \sqsubseteq   C \sqcap {\sf rhs}(\alpha_i)$ \label{alg:line-t-essential2} 
\State Replace the first such $\alpha_i$ in $\Hmc_{add}$ by
${\sf lhs}(\alpha_i)\cap  {\sf lhs}(\gamma)  \sqsubseteq D^{\ast}$ 
\label{alg:line-replacement2}
\Else
\State Concept saturate $\gamma'={\sf lhs}(\alpha_i)\cap
{\sf lhs}(\gamma)\sqsubseteq C$ for \Tmc 
\State Replace the first such $\alpha_i$ in $\Hmc_{add}$ by
$\gamma'$\label{alg:line-replacement3}
\EndIf
\Else
\State Append  $\gamma$ to the list $\Hmc_{add}$ \label{alg:line-append}
\EndIf
\State \textbf{return} $\Hmc$
\end{algorithmic}
\caption{\textbf{Function} $\exists$-Refine($\Hmc,\gamma$)
\label{alg:refine-exists}}
\end{algorithm}
Algorithm~\ref{alg:dl-lite-horn} uses membership queries to compute a
\Tmc-essential counterexample $\gamma$ such that ${\sf rhs}(\gamma)$ contains at most one 
concept expression of the form $\exists r.F$
(Line~\ref{line:at-most-one}) and
which is `Left saturated for \Hmc' (Line~\ref{alg:line-t-essential1});
here, 
a CI is left saturated for \Hmc if its left-hand side contains all subsuming concept names
w.r.t.~$\Hmc$ (Definition~\ref{def:left-sat} below) and \Tmc-essential
if it satisfies the conditions for $\Tmc$-essential CIs from Section~\ref{sec:dllite}, appropriately
modified for CIs with conjunctions of concept names on the left-hand
side (Definition~\ref{def:t-essential-horn} below). 
Then, the algorithm checks whether there is
a concept name $A$ such that ${\sf lhs}(\gamma)\sqsubseteq A$ is a
positive counterexample. If so, then it calls Function
CN-Refine (Algorithm~\ref{alg:refine-left}) and  updates
the hypothesis 
either by refining some $\alpha_i$ in $\Hmc_{add}$
or by appending
a new CI to $\Hmc_{add}$. %
The number of replacements of any given $\alpha_{i}$ in $\Hmc_{add}$
in CN-Refine is bounded by $|\Sigma_\Tmc|$ since whenever $\alpha_i$ is
replaced in CN-Refine$(\Hmc,\gamma)$, then
${\sf lhs}(\alpha_i)\cap {\sf lhs}(\gamma)\subsetneq {\sf
  lhs}(\alpha_i)$.\footnote{This
  is a consequence of the fact that $\Hmc_{add}$ only contains concept
  saturated CIs (defined essentially as in the previous section, see
  Definition~\ref{def:t-essential-horn}
  below): ${\sf lhs}(\alpha_i)\cap {\sf lhs}(\gamma)= {\sf
      lhs}(\alpha_i)$
    and
    $\Tmc \models {\sf lhs}(\alpha_i)\cap {\sf lhs}(\gamma)\sqsubseteq
    A$
    implies $A \in \mn{rhs}(\alpha_i)$ by concept saturatedness, thus
    contradicting
    $\Hmc \not\models {\sf lhs}(\alpha_i)\cap {\sf lhs}(\gamma)
    \sqsubseteq A$.}

If there is no concept name $A$ such that ${\sf lhs}(\gamma)\sqsubseteq A$ is a
positive counterexample then Algorithm~\ref{alg:dl-lite-horn} calls
Function $\exists$-Refine (Algorithm~\ref{alg:refine-exists}). In
this case one considers the existential restrictions that occur on top-level
on the right-hand side of $\gamma$.
Note that $\exists$-Refine can be viewed as a variation of the body of
the while loop in Algorithm~\ref{alg:dl-lite} in which one considers
sets of concept names on the left-hand side of CIs rather than a
single concept name. Recall that in Algorithm~\ref{alg:dl-lite}, the
new CI $\gamma$ and a CI $\alpha$ in $\Hmc_{add}$ are merged if they
have the same concept name on the left-hand side.  In contrast, now
they are merged if the intersection of their left-sides is still
subsumed by some existential restriction $C$ from ${\sf rhs}(\gamma)$
(Lines~1 and 2). There are two cases: if the intersection is also
subsumed by ${\sf rhs}(\alpha)$ (checked in
Line~\ref{alg:line-condition22}), then in the next line a
\Tmc-essential counterexample is computed and the first such
$\alpha_{i}$ is replaced by the new CI. Otherwise it follows that
${\sf lhs}(\alpha_i)\cap {\sf lhs}(\gamma)\subsetneq {\sf
  lhs}(\alpha_i)$
and the first such $\alpha_{i}$ is replaced by the CI computed in
Line~7. Note that the latter can happen at most $|\Sigma_\Tmc|$ times
for each CI in $\Hmc_{add}$ (and the former can happen at most
$|\Tmc|$ times for each CI in $\Hmc_{add}$, see
Lemma~\ref{lem:horn-replacements} below).
If no CI can be refined with $\gamma$ then $\exists$-Refine appends $\gamma$ to $\Hmc_{add}$.   

We now define the step `left-saturate $\gamma$ for $\Hmc$' %
used in Line~\ref{alg:line-t-essential1} of Algorithm~\ref{alg:dl-lite-horn}.
Observe that this step is meaningless for \ourDLLite.
\begin{definition}[Left saturation for \Hmc]\label{def:left-sat}
A CI $\gamma'$ is obtained from a CI $\gamma$ by \emph{left saturation for} $\Hmc$ 
if ${\sf rhs}(\gamma')={\sf rhs}(\gamma)$ and 
${\sf lhs}(\gamma')= \{ A\in \Sigma_{\Tmc}\mid \Hmc \models {\sf lhs}(\gamma) \sqsubseteq A\}$.
A CI $\gamma$ is \emph{left saturated for $\Hmc$} if it coincides with its left saturation for $\Hmc$.
\end{definition}
One can clearly left saturate any CI $\gamma$ for $\Hmc$ by checking whether $\Hmc\models {\sf lhs}(\gamma)\models A$
for every $A\in \Sigma_{\Tmc}$. The following example
shows that Line~\ref{alg:line-t-essential1} is necessary for
Algorithm~\ref{alg:dl-lite-horn} to be polynomial. A similar step is
also necessary in Frazier et al.'s algorithm learning propositional
Horn logic from entailments~\citep{DBLP:conf/icml/FrazierP93}.
\begin{example}\upshape
Assume Line~\ref{alg:line-t-essential1} of Algorithm~\ref{alg:dl-lite-horn}
 is omitted. 
Let for $n\geq 2$,
\begin{eqnarray*}
\Tmc_{n}   & = & \{E_1\sqcap \cdots \sqcap E_n\sqsubseteq A\} \cup \{A_i \sqsubseteq E_i,B_i \sqsubseteq E_i \mid 1\leq i\leq n\}.
\end{eqnarray*}
For $M\subseteq \{1,\ldots,n\}$, set $C_{M}=\bigsqcap_{i\leq n}C_i$,
where $C_i = A_i$ if $i \in M$ and $C_i = B_i$ if $i \notin M$.
Then the oracle can provide for the first $2^{n}$ equivalence queries
in the while loop of Algorithm~\ref{alg:dl-lite-horn} a positive
counterexample $C_{M} \sqsubseteq A$ by always choosing a fresh set
$M\subseteq \{1,\ldots,n\}$.
\end{example}
For the refinements on the right-hand side, we extend to \ourDLLitehorn the notion of 
$\mathcal{T}$-essential CIs introduced in the previous section:
\begin{enumerate}
\item (Concept saturation for $\Tmc$)
A CI $\gamma'$ is \emph{obtained from a CI $\gamma$ by concept saturation for}~$\Tmc$ 
if ${\sf lhs}(\gamma)={\sf lhs}(\gamma')$, ${\sf rhs}(\gamma')$ is obtained from 
${\sf rhs}(\gamma)$ by adding a
concept name to the label of some node of ${\sf rhs}(\gamma)$, and 
$\Tmc\models \gamma'$.
A CI $\gamma$ is \emph{concept saturated for $\Tmc$} if $\Tmc\models
\gamma$ and there is no $\gamma'$ with $\gamma \neq \gamma'$ that can
be obtained from $\gamma$ by concept saturation for $\Tmc$.
\item (Role saturation for $\Tmc$)
A CI $\gamma'$ is \emph{obtained from $\gamma$ by role saturation for} $\Tmc$
if ${\sf lhs}(\gamma)={\sf lhs}(\gamma')$, ${\sf rhs}(\gamma')$ is obtained from ${\sf rhs}(\gamma)$
by replacing in some edge label a role $r$ by a role 
$s$ with $\Tmc\models s\sqsubseteq r$, and $\Tmc\models \gamma'$.
A CI $\gamma$ is \emph{role saturated for $\Tmc$} if $\Tmc\models
\gamma$ and 
there is no $\gamma'$ with $\gamma \neq \gamma'$ that can be obtained
from $\gamma$ by role saturation for $\Tmc$.
 \item (Parent/child merged for \Tmc) 
A CI $\gamma'$ is \emph{obtained from a CI $\gamma$ by parent/child merging for} $\Tmc$ if
${\sf lhs}(\gamma)={\sf lhs}(\gamma')$, ${\sf rhs}(\gamma')$ is obtained from ${\sf rhs}(\gamma')$ by parent/child
merging (as in Definition~\ref{def:parchimer}), and $\Tmc\models \gamma'$.  
A CI $\gamma$ is \emph{parent/child merged for $\Tmc$} if $\Tmc\models
\gamma$ and 
there is no $\gamma'$ with $\gamma \neq \gamma'$ that can be obtained
from
$\gamma$ by parent/child merging for~$\Tmc$.

\item (Sibling merged for \Tmc) A CI $\gamma'$ is \emph{obtained from
    a CI $\gamma$ by sibling merging for} $\Tmc$ if
  ${\sf lhs}(\gamma)={\sf lhs}(\gamma')$, ${\sf rhs}(\gamma')$ is
  obtained from ${\sf rhs}(\gamma')$ by sibling merging (as in
  Definition~\ref{def:sib}), and $\Tmc\models \gamma'$.  A CI $\gamma$
  is \emph{sibling merged for $\Tmc$} if $\Tmc\models \gamma$ and
  there is no $\gamma'$ with $\gamma \neq \gamma'$ that can be
  obtained
  from $\gamma$ by sibling merging for $\Tmc$.
\item (Decomposed CI for \Tmc) A CI $\gamma$ is \emph{decomposed for} $\Tmc$ if
$\Tmc\models\gamma$ and
for every non-root node $d$ in ${\sf rhs}(\gamma)$, every role $r$, and every $r$-successor $d'$ of $d$ in
${\sf rhs}(\gamma)$ we have 
$\Tmc \not\models l(d) \sqsubseteq \exists r.C'$, 
where
$C'$ corresponds to the subtree
of ${\sf rhs}(\gamma)$ rooted at $d'$.
\end{enumerate}

\begin{definition}\label{def:t-essential-horn}
A \ourDLLitehorn CI %
is \emph{$\mathcal{T}$-essential} if 
it is concept saturated, role saturated, parent/child merged, sibling merged,
and decomposed for $\Tmc$.
\end{definition}
The saturation, merging and decomposition steps defined above are straightforward generalisations of 
Definitions~\ref{def:csat} to \ref{def:decom} to CIs with conjunctions on the left-hand side. 
One can easily generalise the arguments from \ourDLLite to show that for any CI 
$\gamma$ with $\Hmc\not\models\gamma$ one can compute a 
$\Tmc$-essential $\gamma'$ with $\Hmc\not\models\gamma'$ using 
polynomially many membership queries and entailment checks relative to $\Hmc$.
For the analysis of the learning algorithm it is crucial that
all CIs in the ordered list $\Hmc_{add}$ are $\Tmc$-essential at all times,  
which we prove next.
\begin{lemma}\label{lem:always-t-essential}
At any point in the execution of Algorithm~\ref{alg:dl-lite-horn}, all CIs in $\Hmc_{add}$ are
$\Tmc$-essential. 
\end{lemma}
\begin{proof}
If a CI $\gamma$ is of the form $A_1 \sqcap \cdots \sqcap A_n\sqsubseteq A$ with $A$ a concept name,
then the set ${\sf rhs}(\gamma')$ of the concept saturation $\gamma'$ of $\gamma$ for $\Tmc$ contains concept names
only. Thus, $\gamma'$ is $\Tmc$-essential.
It follows that the CIs added to $\Hmc_{add}$ in Lines~\ref{alg:line-replacement1}
and~\ref{alg:line-append0} of CN-Refine are \Tmc-essential. 
Also, it is easy to see that if $\gamma$ is \Tmc-essential
and $C \in {\sf rhs}(\gamma)$ then the concept saturation of 
${\sf lhs}(\gamma)\sqsubseteq C$ for $\Tmc$ is \Tmc-essential as well.
Thus, the CI $\gamma'$ in Line~\ref{alg:line-replacement3} of $\exists$-Refine is \Tmc-essential.
\end{proof}

\subsection*{Polynomial Query Bound on the Algorithm}

As in the previous section it is immediate that upon termination the
algorithm has found a TBox
$\Hmc=\mathcal{H}_{basic}\cup \mathcal{H}_{add}$ that is logically
equivalent to the target TBox $\Tmc$.  It thus
remains to show that it issues only polynomially many queries of
polynomial size. We first discuss how Lines~\ref{line:at-most-one} and
\ref{alg:line-t-essential1} of Algorithm~\ref{alg:dl-lite-horn} and
Line~\ref{alg:line-t-essential2} of
$\exists$-Refine %
can be implemented.  The next lemma addresses
Lines~\ref{line:at-most-one} and \ref{alg:line-t-essential1} of
Algorithm~\ref{alg:dl-lite-horn}.
\begin{lemma}\label{lem:t-essential}
Given a positive counterexample $\gamma$
for $\Tmc$ relative to $\Hmc$, one can construct 
with polynomially many polynomial size membership queries in  $|\gamma|$ and $|\Tmc|$, a  
counterexample $\gamma'$ that is left saturated for $\Hmc$,
$\Tmc$-essential and such that $|\{\exists r.F\mid \exists r.F\in
{\sf rhs}(\gamma')\}|\leq 1$. 
\end{lemma}
The proof of Lemma~\ref{lem:t-essential} is a straightforward extension of the proof of Lemma~\ref{lem:t-essential0}
and uses the observation that a left-saturated $\gamma'$ for $\Hmc$ can be computed from $\gamma$ by adding all
concept names $A\in \Sigma_{\Tmc}$ with $\Hmc\models {\sf
lhs}(\gamma)\sqsubseteq A$ to ${\sf lhs}(\gamma)$.
This lemma also requires that $|\{\exists r.F\mid \exists r.F\in
{\sf rhs}(\gamma')\}|\leq 1$.
If there is $A \in {\sf rhs}(\gamma')$
such that $\Hmc \not\models {\sf lhs}(\gamma') \sqsubseteq A$ 
then we can simply drop  all conjuncts of the
form  $\exists r.F$ from ${\sf rhs}(\gamma')$. 
Otherwise, we can satisfy the condition by  simply choosing a
conjunct $\exists r.F \in {\sf rhs}(\gamma')$   such that
$\Hmc \not\models {\sf lhs}(\gamma') \sqsubseteq \exists r.F$ and
then   apply    `Concept saturation for \Tmc' to
${\sf lhs}(\gamma') \sqsubseteq \exists r.F$.
The resulting $\gamma'$ is left saturated for \Hmc, \Tmc-essential and has at most one
conjunct of the form $\exists r.F$ in ${\sf rhs}(\gamma')$. 

\smallskip

The following lemma addresses Line~\ref{alg:line-t-essential2} of $\exists$-Refine.
\begin{lemma}\label{lem:dllite00}
Assume that $\alpha$ and 
$\gamma$ are $\Tmc$-essential and there is $C \in {\sf rhs}(\gamma)$
such that $\Tmc\models {\sf lhs}(\alpha)\cap
{\sf lhs}(\gamma) \sqsubseteq {\sf rhs}(\alpha)\sqcap C$. Then one can construct, with polynomially many polynomial 
size membership queries 
in $|{\sf rhs}(\alpha)|$ and $|C|$,
a $\Tmc$-essential ${\sf lhs}(\alpha)\cap
{\sf lhs}(\gamma) \sqsubseteq D^\ast$ 
such that
$\emptyset \models D^\ast\sqsubseteq {\sf rhs}(\alpha)\sqcap C$.
\end{lemma}
\begin{proof} 
Assume $\Tmc\models {\sf lhs}(\alpha)\cap
{\sf lhs}(\gamma) \sqsubseteq {\sf rhs}(\alpha)\sqcap C$.
Then, similar to Lemma~\ref{lem:dllite0}, one can show that
 the only
property of \Tmc-essential CIs that can fail is being sibling merged
for \Tmc and that after applying the
step `Sibling merging for \Tmc' to ${\sf lhs}(\alpha)\cap
{\sf lhs}(\gamma) \sqsubseteq {\sf rhs}(\alpha)\sqcap C$ the resulting
CI is \Tmc-essential, as required.
\end{proof}

We also have to show that the number of CIs in $\Hmc_{add}$ is bounded polynomially in
$|\Tmc|$ and %
for each position of $\Hmc_{add}$ the number of replacements is bounded polynomially in
$|\Tmc|$. These properties follow from the following lemma.
\begin{lemma}
\label{lem:term200}
Let $\Hmc_{add}$ be a ordered list of CIs computed at some point of an 
execution of Algorithm~\ref{alg:dl-lite-horn}.  
Then 
\begin{enumerate}
\item [(i)] the length of $\Hmc_{add}$ is bounded by the number of CIs in $\Tmc$ and
\item [(ii)] The number of replacements of an existing  CI $\alpha \in \Hmc_{add}$ is
bounded polynomially in $ |\Tmc|$.
\end{enumerate}
\end{lemma}
The rest of the section is devoted to proving Lemma~\ref{lem:term200}.
We first show Point~(ii) of Lemma~\ref{lem:term200} and start by generalising
Lemma~\ref{lem:size} on the size of $\Tmc$-essentials CIs. For any conjunction $C$ of concept names we set
\begin{eqnarray*}
C^{\Tmc} & = & \{D \mid \Tmc \models  
C \sqsubseteq A_1 \sqcap \cdots \sqcap A_k \text{ and }
A_1 \sqcap \cdots \sqcap A_k \sqsubseteq D \in \Tmc\} \cup \\
    & &  \{B \mid \Tmc\models C \sqsubseteq B, B \text{ basic concept over $\Sigma_{\Tmc}$}\}
\end{eqnarray*}
Recall that for any concept expression $C$ we denote by $n_{C}$ the number of nodes in the 
tree $T_{C}$ corresponding to $C$.
\begin{lemma}\label{lem:size2}
If $\alpha$ is \Tmc-essential, then $n_{{\sf rhs}(\alpha)} \leq \sum_{D \in {\sf lhs}(\alpha)^\Tmc} n_D$.
\end{lemma}
\begin{proof}
The proof is almost the same as the proof of Lemma~\ref{lem:size}.
Assume $\alpha$ is $\Tmc$-essential. 
Let $D_0:= \bigsqcap_{D \in {\sf lhs}(\alpha)^{\Tmc}}D$ and let $\Imc_{D_0,\Tmc}$ be the canonical model of 
$D_0$ and \Tmc. Now one can prove in almost the same way as in the proof of Lemma~\ref{lem:size}
that the homomorphism $h$ from $T_{{\sf rhs}(\alpha)}$ into $\Imc_{D_{0},\Tmc}$
mapping $\rho_{{\sf rhs}(\alpha)}$ to $\rho_{D_{0},\Tmc}$ is an injective mapping into $\Imc_{D_{0}}$ (using
Lemma~\ref{lem:can2} for \ourDLLitehorn instead of \ourDLLite).
\end{proof}
We are now in the position to prove Point~(ii) of Lemma~\ref{lem:term200}.

\begin{lemma}\label{lem:horn-replacements}
The number of replacements of an existing  CI 
$\alpha \in \Hmc_{add}$ is
bounded polynomially in $ |\Tmc|$. %
\end{lemma}
\begin{proof}
A CI $\alpha \in \Hmc_{add}$ can be replaced 
in Line~\ref{alg:line-replacement1} of
CN-Refine or in Lines~\ref{alg:line-replacement2}
or~\ref{alg:line-replacement3} of $\exists$-Refine. 
If $\alpha$ is replaced by $\alpha'$ in Line~\ref{alg:line-replacement1} of
CN-Refine or in Line~\ref{alg:line-replacement3} of
$\exists$-Refine
then ${\sf lhs}(\alpha')\subsetneq {\sf lhs}(\alpha)$, 
so the number of replacements is bounded by $|\Sigma_\Tmc|$.
If $\alpha$ is replaced by $\alpha'$ in Line~\ref{alg:line-replacement2}
of $\exists$-Refine, then either ${\sf lhs}(\alpha')\subsetneq {\sf lhs}(\alpha)$
or ${\sf lhs}(\alpha') = {\sf lhs}(\alpha)$. For the latter case one can
show as in the proof of Lemma~\ref{lem:newpolylem} for \ourDLLite, the following 

\medskip
\noindent
{\bf Claim}. If $A_1\sqcap \cdots \sqcap A_n \sqsubseteq C$ and
$A_1\sqcap \cdots \sqcap A_n\sqsubseteq C'$ are $\Tmc$-essential, 
and $\emptyset \models C'\sqsubseteq C$, then $T_{C}$ is obtained from $T_{C'}$ by removing subtrees.

\smallskip

Thus, each time $\alpha \in \Hmc_{add}$ is
replaced in Line~\ref{alg:line-replacement2} 
of $\exists$-Refine without decreasing
the number of concept names in ${\sf lhs}(\alpha)$, 
the number $n_{{\sf rhs}(\alpha)}$ of nodes in the tree representation of ${\sf rhs}(\alpha)$
strictly increases. By Lemma~\ref{lem:size2}, $n_{{\sf rhs}(\alpha)}$
is bounded polynomially in $|\Tmc|$ and the lemma follows.
\end{proof}
We now come to the proof of Point~(i) of Lemma~\ref{lem:term200}. To formulate an upper bound on
the length of $\Hmc_{add}$ in terms of $\Tmc$ it is convenient to assume that the right-hand side of every
CI in $\Tmc$ is \emph{primitive}, that is, either a concept name or a concept expression of the 
form $\exists r.D$. This assumption is w.l.o.g. since one can equivalently transform every CI $C\sqsubseteq D_{1}\sqcap D_{2}$ 
into two CIs $C \sqsubseteq D_{1}$ and $C \sqsubseteq D_{2}$. We call such a TBox \emph{rhs-primitive}.
Note that CIs in \Hmc may still have multiple concepts on the right-hand side. 

A concept $C$ is called \emph{concept saturated for \Tmc} if
$\Tmc \models C\sqsubseteq C'$ whenever $C'$ results from $C$ by
adding a new concept name $A'$ to the label of some node in
$T_C$. Denote by $C^{\sf sat}$ the (unique) concept 
obtained from $C$ by adding concept names to the node labels of
$T_{C}$ until it is concept saturated for \Tmc.  The following definition
enables us to link the CIs in $\Hmc_{add}$ to the CIs in $\Tmc$.
\begin{definition}\label{def:map}
Let $\Tmc$ be rhs-primitive. We say that a CI $\alpha$ \emph{has target} $\beta \in \Tmc$ if 
\begin{enumerate}
\item ${\sf lhs}(\beta) \subseteq {\sf lhs}(\alpha)$ and 
\item there exists $D \in {\sf rhs}(\alpha)\setminus {\sf lhs}(\alpha)$ such that 
$\emptyset\models {\sf rhs}(\beta)^{\sf sat} \sqsubseteq
D$.
\end{enumerate}
\end{definition}
     
We aim to show that Algorithm \ref{alg:dl-lite-horn} maintains the invariant 
that 
\begin{itemize}

\item[(iii)] every $\alpha  \in \Hmc_{add}$ has some target $\beta  \in \Tmc$ and

\item[(iv)] every $\beta  \in \Tmc$  is the target of at most one $\alpha  \in \Hmc_{add}$.

\end{itemize}
Then Point~(i) of Lemma~\ref{lem:term200} clearly follows. 

\begin{example} \upshape To illustrate Definition~\ref{def:map},
  suppose that
  $$\Tmc=\{A_1 \sqcap A_4\sqsubseteq A_2, \ A_2 \sqsubseteq \exists
  r.A_3, \ A_3\sqsubseteq A_4, \ A_r \equiv \exists r.\top\}$$
  is the target TBox. $\Tmc$ is rhs-primitive.  To simplify notation,
  we use $\beta_i$ to denote the $i$-th CI occurring in \Tmc
  above. Assume $\Hmc_{basic}=\{\beta_3,\beta_4\}$ and
  $\Hmc_{add}=\emptyset$.  Let
  $\alpha_1 = A_1 \sqcap A_3 \sqsubseteq A_2$. Then there is no
  $\beta_i \in\Tmc$ such that $\alpha_1$ has target
  $\beta_i$. However, by applying left saturation for \Hmc to
  $\alpha_1$ we obtain
  $\alpha'_1 = A_1 \sqcap A_3\sqcap A_4 \sqsubseteq A_2$ and since
  ${\sf lhs}(\beta_1) \subseteq {\sf lhs}(\alpha'_1)$ and
  $A_2 \not\in {\sf lhs}(\alpha'_1)$, $\alpha'_1$ has target
  $\beta_1$.  For
  $\alpha_2 = A_1 \sqcap A_4 \sqsubseteq \exists r.A_3$, there is no
  $\beta_i \in \Tmc$ such that $\alpha_2$ has target~$\beta_i$. But
  $\alpha_2$ is not \Tmc-essential and making it \Tmc-essential
  results in
  $\alpha'_2 = A_1 \sqcap A_4 \sqsubseteq A_r\sqcap A_1\sqcap
  A_2\sqcap A_4 \sqcap \exists r.(A_3\sqcap A_4)$
  which again has target $\beta_1$.
Finally, let $\alpha_3 = A_2  \sqsubseteq 
\exists r.A_4$.
As ${\sf lhs}(\beta_2) \subseteq {\sf lhs}(\alpha_3)$ 
and $ \emptyset\models A_r \sqcap \exists r.(A_3\sqcap A_4) \sqsubseteq
\exists r.A_4$, 
\ $\alpha_3$ has target $\beta_2$. Note that $\alpha_3$ is not
\Tmc-essential, but the result of making it $\Tmc$-essential also has
target $\beta_2$.
\end{example}
Point~(iii) is a consequence of the following lemma. 
\begin{lemma}\label{lem:has-target}
Let $\Tmc$ be rhs-primitive and let $\gamma$ be a \Tmc-essential CI
such that $\emptyset\not\models \gamma$. Then $\gamma$ has some target $\beta \in \Tmc$.
¡“%
\end{lemma}

\begin{proof}
Assume $\gamma$ is $\Tmc$-essential and $\emptyset\not\models\gamma$. Assume for a proof by contradiction
that $\gamma$ has no target in $\Tmc$. We first show the following

\smallskip
\noindent
{\bf Claim 1.} If $\gamma$ has no target in $\Tmc$ and $\Tmc\models
{\sf lhs}(\gamma)\sqsubseteq A$ then $A\in {\sf lhs}(\gamma)$, for all $A \in \NC$.

\smallskip
For the proof of Claim~1, consider the canonical model $\Imc_{{\sf lhs}(\gamma),\Tmc}$ 
of ${\sf lhs}(\gamma)$ and $\Tmc$. Recall that $\rho_{{\sf lhs}(\gamma),\Tmc}$ denotes the
root of $\Imc_{{\sf lhs}(\gamma),\Tmc}$. By Lemma~\ref{lem:can1}, $\rho_{{\sf lhs}(\gamma),\Tmc}\in D^{\Imc_{{\sf lhs}(\gamma),\Tmc}}$
iff $\Tmc\models {\sf lhs}(\gamma)\sqsubseteq D$, for any concept $D$. Thus, it suffices to prove that
$\rho_{{\sf lhs}(\gamma),\Tmc}\in A^{\Imc_{{\sf lhs}(\gamma),\Tmc}}$ implies $A\in {\sf lhs}(\gamma)$, for all concept names $A$.
The proof is by induction over the sequence $\Imc_{0},\ldots$ used to construct $\Imc_{{\sf lhs}(\gamma),\Tmc}$,
where $\Imc_{0}=\Imc_{{\sf lhs}(\gamma)}$. For $\Imc_{{\sf lhs}(\gamma)}$ this is the case by definition. 
Now suppose the claim holds for $\Imc_{n}$ and 
$\rho_{{\sf lhs}(\gamma),\Tmc}\in A^{\Imc_{n+1}}\setminus A^{\Imc_{n}}$. Then there either exist concept names $A_{1},\ldots,A_{k}$
with $A_{1}\sqcap \cdots \sqcap A_{k}\sqsubseteq A\in \Tmc$ and $\rho_{{\sf lhs}(\gamma),\Tmc}\in (A_{1}\sqcap \cdots \sqcap A_{k})^{\Imc_{n}}$
or there exists $\exists r.\top$ with $\exists r.\top\sqsubseteq A\in \Tmc$ and $\rho_{{\sf lhs}(\gamma),\Tmc}\in 
(\exists r.\top)^{\Imc_{n}}$. In the first case, we have $\{A_{1},\ldots,A_{k}\}\subseteq {\sf lhs}(\gamma)$ by induction hypothesis
and so $A\in {\sf lhs}(\gamma)$ because otherwise $A_{1}\sqcap \cdots \sqcap A_{k}\sqsubseteq A$ would be a target of 
$\gamma$.
In the second case there must be an $\Imc_{m}$ with $m<n$ such that there are 
$E_{1}\sqcap \cdots \sqcap E_{k}\sqsubseteq \exists s.D\in \Tmc$
and $s\sqsubseteq r\in \Tmc$ with $\rho_{{\sf lhs}(\gamma),\Tmc}\in (E_{1}\sqcap \cdots \sqcap E_{k})^{\Imc_{m}}$ (the case $s=r$ is 
similar and omitted). %
It follows that $A\in {\lhs}(\gamma)$ because otherwise $E_{1}\sqcap \cdots \sqcap E_{k}\sqsubseteq \exists s.D$ would be a target of 
$\gamma$ since, by induction hypothesis, $\{E_{1},\ldots,E_{k}\}\subseteq {\sf lhs}(\gamma)$ and 
$A\in (\exists s.D)^{\sf sat}$. This finishes the proof of Claim~1.

By Claim~1, as $\emptyset\not\models \gamma$, there is a conjunct of the form $\exists r.F$ in ${\sf rhs}(\gamma)$.
Let $({\sf lhs}(\alpha))^{\Tmc}$ be as above and $\Imc_{ D_0,\Tmc}$ be the canonical model of
$D_0 = \bigsqcap_{D\in ({\sf lhs}(\alpha))^{\Tmc}}D$ and \Tmc.  
As $\exists r.F \in {\sf rhs}(\gamma)$ and 
$\gamma$ is \Tmc-essential one can show in the same way
as in the proof of Lemma~\ref{lem:size} that there is an injective homomorphism 
from the labelled tree $T_{\exists r.F}$ corresponding to $\exists r.F$
into the restriction of $\Imc_{ D_0,\Tmc}$ to $\Delta^{\Imc_{D_0}}$
mapping the root of $T_{\exists r.F}$ to the root $\rho_{D_{0},\Tmc}$ of $\Imc_{ D_0,\Tmc}$.
Thus, by definition of $({\sf lhs}(\alpha))^{\Tmc}$, there is 
$\beta \in \Tmc$ such that $\Tmc\models {\sf lhs}(\alpha)\sqsubseteq 
{\sf lhs}(\beta)$ and $\emptyset\models {\sf rhs}(\beta)^{\sf sat}\sqsubseteq \exists r.F$. 
By Lemma~\ref{lem:can1}, $\rho_{{\sf lhs}(\gamma),\Tmc} \in
A^{\Imc_{{\sf lhs}(\gamma),\Tmc}}$, for all $A \in {\sf lhs}(\beta)$. Hence, by Claim~1 and again Lemma~\ref{lem:can1},
${\sf lhs}(\beta)\subseteq {\sf lhs}(\gamma)$. We have shown that $\gamma$ has target $\beta$ and so 
derived a contradiction.  
\end{proof}

Point~(iii) is a direct consequence of Lemma~\ref{lem:has-target} and the fact that all CIs in $\Hmc_{add}$
are \Tmc-essential (Lemma~\ref{lem:always-t-essential}).
To prove Point~(iv), we first establish the following intermediate
Lemmas~\ref{lem:replace01} and~\ref{lem:replace02}.   

\begin{lemma}\label{lem:replace01}
Let $\Tmc$ be rhs-primitive and let $\Hmc, \gamma$ be inputs to CN-Refine.
Let $\alpha_{i}\in \Hmc_{add}$, $\beta \in \Tmc$, and concept name $A\not\in {\sf lhs}(\gamma)$ satisfy the following conditions:
(a) ${\sf lhs}(\beta)\subseteq {\sf lhs}(\gamma)$; (b) $\Tmc\models{\sf lhs}(\beta)\sqsubseteq A$; (c) ${\sf lhs}(\beta) \subseteq {\sf lhs}(\alpha_i)$.
Then there is some $j \leq i$ such that $\alpha_j$ is replaced
in Line~\ref{alg:line-replacement1} of CN-Refine. 
\end{lemma}

\begin{proof}
Assume $\Hmc$, $\gamma$, $\alpha_{i}$, $\beta$, and $A$ satisfy the conditions of the lemma.
If CN-Refine replaces some $\alpha_j$ with $j < i$
then we are done. Suppose this does not happen. Then we need to show that $\alpha_i$ is replaced. 
By Conditions (a), (b), and (c), $\Tmc \models   {\sf lhs}(\gamma) \cap {\sf lhs}(\alpha_i) \sqsubseteq
A$. As $\gamma$ is \eleftcondition, $A \not\in {\sf lhs}(\gamma)$ implies
that $\Hmc \not\models {\sf lhs}(\gamma) \sqsubseteq A$. 
So $\Hmc \not\models {\sf lhs}(\gamma) \cap {\sf lhs}(\alpha_i) \sqsubseteq A$. 
Then, the condition in Lines~1 and~2 of  CN-Refine is satisfied and $\alpha_i$ is replaced. 
\end{proof}

\begin{lemma}\label{lem:replace02}
Let $\Tmc$ be rhs-primitive and let $\Hmc,\gamma$ be inputs to $\exists$-Refine.
If $\gamma$ has target $\beta \in \Tmc$ 
and $\alpha_i \in \Hmc_{add}$ satisfies ${\sf lhs}(\beta)
\subseteq {\sf lhs}(\alpha_i)$, then there is some $j \leq i$ such that $\alpha_j$ is replaced
in Line~\ref{alg:line-replacement2}
or~\ref{alg:line-replacement3} of
$\exists$-Refine.
\end{lemma}

\begin{proof} 
Let $\Hmc$, $\gamma$, $\beta$, and $\alpha_{i}$ satisfy the conditions of the lemma.
If $\exists$-Refine replaces
some $\alpha_j$ with
$j < i$ then we are done. Suppose this does not happen. 
We need to show that  $\alpha_i$ is replaced. 
We first show that there is a concept $C$ of the form $\exists r.F$ in ${\sf rhs}(\gamma)$ such
that ${\sf lhs}(\gamma) \sqsubseteq C$ has target $\beta$.
Note that if Algorithm~\ref{alg:dl-lite-horn} calls 
$\exists$-Refine then
there is no concept name $A$ such that $\Tmc\models {\sf
lhs}(\gamma)\sqsubseteq A$ and $\Hmc\not\models {\sf
lhs}(\gamma)\sqsubseteq A$ (Line~\ref{alg:line-exists}). As $\gamma$ is left saturated for \Hmc and $\Tmc$-essential,
this implies $\NC \cap {\sf rhs}(\gamma)\subseteq {\sf lhs}(\gamma)$.
But then any $C \in {\sf rhs}(\gamma)\setminus {\lhs}(\gamma)$
with
$\emptyset\models {\sf rhs}(\beta)^{\sf sat}\sqsubseteq C$ is compound.
As $\gamma$ has target $\beta$ it follows that ${\sf lhs}(\gamma) \sqsubseteq C$ has target $\beta$
for some $C$ of the form $\exists r.F$ in ${\sf rhs}(\gamma)$.  
By Line~\ref{line:at-most-one} of Algorithm~\ref{alg:dl-lite-horn},
there is only one such conjunct $C$ in ${\sf rhs}(\gamma)$. From 
$\emptyset\models {\sf rhs}(\beta)^{\sf sat}\sqsubseteq C$ we obtain $\Tmc \models {\sf lhs}(\beta) \sqsubseteq C$. 
Since ${\sf lhs}(\beta) \subseteq {\sf lhs}(\gamma) \cap {\sf lhs}(\alpha_i)$, we have that  
$\Tmc \models  {\sf lhs}(\gamma)\cap {\sf lhs}(\alpha_i) \sqsubseteq
C$.
As $\gamma$ is a positive counterexample, $\Hmc\not\models
\gamma$. From $\NC \cap {\sf rhs}(\gamma)\subseteq {\sf lhs}(\gamma)$ we thus obtain
$\Hmc \not\models   {\sf lhs}(\gamma) \sqsubseteq C$, and so,
$\Hmc \not\models  {\sf lhs}(\alpha_i)\cap {\sf lhs}(\gamma) \sqsubseteq
C$.
Hence, the condition in Lines~1 and~2 of $\exists$-Refine
is satisfied and   
$\alpha_i$ is replaced (in
Line~\ref{alg:line-replacement2} or~\ref{alg:line-replacement3}).
\end{proof}

Point (iv) above is now a direct consequence of the following lemma. 
\begin{lemma}\label{lem:target} 
At any point in the execution of Algorithm~\ref{alg:dl-lite-horn}, 
if $\alpha_j \in \Hmc_{add}$ has target $\beta \in \Tmc$ then 
${\sf lhs}(\beta) \not\subseteq {\sf lhs}(\alpha_i)$, for all $i<j$. 
\end{lemma}

\begin{proof}
The proof is by induction on the number $k$ of iterations. For $k = 1$ the lemma is vacuously true. 
Assume it holds for $k = n$, $n \geq 1$. %
Now the algorithm modifies $\Hmc_{add}$ in response to receiving a
positive counterexample  in iteration $k = n+1$. We make a case distinction:

\medskip
\noindent
Case 1.
Algorithm~\ref{alg:dl-lite-horn} calls 
CN-Refine:
Let $\Hmc,\gamma$ be the inputs to CN-Refine.
Assume first that the condition in Lines~1 and~2 is not satisfied. Then  
CN-Refine appends the result of concept saturating
$\gamma$ for \Tmc to $\Hmc_{add}$. Call this CI $\gamma'$.
Suppose that the lemma fails to hold.
This can only happen if $\gamma'$ has a target $\beta\in \Tmc$  
and there is $\alpha_i \in \Hmc_{add}$ such that 
 ${\sf lhs}(\beta) \subseteq {\sf lhs}(\alpha_i)$. 
Then, since ${\sf lhs}(\gamma')={\sf lhs}(\gamma)$, we have that ${\sf lhs}(\beta)\subseteq {\sf lhs}(\gamma)$
and, since ${\sf rhs}(\gamma')\subseteq \NC$, there is a concept name $A\not\in {\sf lhs}(\gamma)$ such that $\emptyset\models{\sf rhs}(\beta)^{\sf sat}\sqsubseteq A$. 
So $\Tmc\models {\sf lhs}(\beta)\sqsubseteq A$.
Then Lemma~\ref{lem:replace01} applies to $\Hmc$, $\gamma$, $\alpha_{i}$, $\beta$ and $A$
which contradicts the assumption that CN-Refine
  did not replace any $\alpha_j \in \Hmc_{add}$, $j \leq i$.

\smallskip

Now assume that the condition in Lines~1 and~2 is satisfied.
Suppose that the lemma fails to hold.
This can only happen if there are $\alpha_i,\alpha_j \in \Hmc_{add}$ with $i < j$
such that either (a) $\alpha_i$ is replaced by $\alpha'_i$,  
${\sf lhs}(\beta) \subseteq {\sf lhs}(\alpha'_i)$ 
and 
$\alpha_j$ has target $\beta$; or (b) $\alpha_j$ is replaced by $\alpha'_j$,  
$\alpha'_j$ has target $\beta$ and 
${\sf lhs}(\beta) \subseteq {\sf lhs}(\alpha_i)$.
In case (a), from ${\sf lhs}(\gamma ) \cap {\sf lhs}(\alpha_i) = {\sf lhs}(\alpha'_i)$, 
we obtain ${\sf lhs}(\beta) \subseteq {\sf lhs}(\gamma ) \cap {\sf lhs}(\alpha_i)$.
Thus, ${\sf lhs}(\beta) \subseteq {\sf lhs}(\alpha_i)$. This contradicts the induction hypothesis.
Now assume case (b). Since $\alpha'_j$ has target $\beta$, we obtain:
\begin{enumerate}
\item ${\sf lhs}(\beta) \subseteq {\sf lhs}(\alpha'_j)$; and
\item as ${\sf rhs}(\alpha'_j)\subseteq \NC$, there is $A \in \NC$ with
$A\in {\sf rhs}(\alpha'_j)\setminus {\sf
lhs}(\alpha'_j)$ and $\emptyset\models {\sf rhs}(\beta)^{\sf
sat}\sqsubseteq A$. 
\end{enumerate} 
Since ${\sf lhs}(\gamma ) \cap {\sf lhs}(\alpha_j) =
{\sf lhs}(\alpha'_j)$, it follows from Point~1 that
${\sf lhs}(\beta) \subseteq {\sf lhs}(\alpha_j)$ and
${\sf lhs}(\beta) \subseteq {\sf lhs}(\gamma)$.
From $\emptyset\models {\sf rhs}(\beta)^{\sf
sat}\sqsubseteq A$ we obtain $\Tmc\models {\sf lhs}(\beta)\sqsubseteq A$. 
If $A \in {\sf rhs}(\alpha'_j)\setminus {\sf
lhs}(\alpha'_j)$ then
either $A \in {\sf rhs}(\alpha_j)\setminus{\sf lhs}(\alpha_j)$ or
$A \not\in {\sf lhs}(\gamma)$.  
So either $\alpha_j$ has target $\beta$ or $A \not\in {\sf lhs}(\gamma)$. 
$\alpha_j$ does not have target $\beta$ as this would contradict the induction hypothesis.
Thus, $A \not\in {\sf lhs}(\gamma)$ and the conditions of Lemma~\ref{lem:replace01} are 
satisfied by $\Hmc$, $\gamma$, $\alpha_{i}$, $\beta$, and $A$. Thus, some $\alpha_{i'}$ with $i'\leq i$
is replaced which contradicts the assumption that $\alpha_j$ is replaced.

\medskip
\noindent
Case 2. Algorithm~\ref{alg:dl-lite-horn} calls 
$\exists$-Refine:
Let $\Hmc,\gamma$ be the inputs to $\exists$-Refine.
Assume first that the condition in Lines~1 and 2 is not satisfied. Then  
$\exists$-Refine appends  
$\gamma$ to $\Hmc_{add}$.
Suppose the lemma fails to hold. This can only happen 
if $\gamma$ has a target $\beta\in \Tmc$   
and there is $\alpha_i \in \Hmc_{add}$ such that 
${\sf lhs}(\beta) \subseteq {\sf lhs}(\alpha_i)$.
By Lemma~\ref{lem:replace02}, this contradicts the assumption that $\exists$-Refine
did not replace any $\alpha_j \in \Hmc_{add}$, $j \leq i$.

Assume now that the condition in Lines~1 and 2 is satisfied.
Suppose that the lemma fails to hold. This can only happen if there are 
$\alpha_i,\alpha_j \in \Hmc_{add}$ with $i < j$ such that either (a) $\alpha_i$ is replaced by $\alpha'_i$,   
${\sf lhs}(\beta) \subseteq {\sf lhs}(\alpha'_i)$ 
and 
$\alpha_j$ has target $\beta$; or (b) $\alpha_j$ is replaced by $\alpha'_j$,  
$\alpha'_j$ has target $\beta$ and 
${\sf lhs}(\beta) \subseteq {\sf lhs}(\alpha_i)$.
For case (a) we argue as above: from ${\sf lhs}(\gamma ) \cap {\sf lhs}(\alpha_i) = {\sf lhs}(\alpha'_i)$, 
we obtain ${\sf lhs}(\beta) \subseteq {\sf lhs}(\gamma ) \cap {\sf lhs}(\alpha_i)$. Thus,
${\sf lhs}(\beta) \subseteq {\sf lhs}(\alpha_i)$, which contradicts the induction hypothesis.
Now assume case (b). As $\alpha'_j$ has target 
$\beta$, we obtain the following:
\begin{enumerate}
\item ${\sf lhs}(\beta) \subseteq {\sf lhs}(\alpha'_j)$; and
\item there is  %
$D\in {\sf rhs}(\alpha'_j)\setminus {\sf
lhs}(\alpha'_j)$ and $\emptyset\models {\sf rhs}(\beta)^{\sf
sat}\sqsubseteq D$. 
\end{enumerate} 
Since ${\sf lhs}(\gamma ) \cap {\sf lhs}(\alpha_j) =
{\sf lhs}(\alpha'_j)$, it follows from Point~1 that
${\sf lhs}(\beta) \subseteq {\sf lhs}(\alpha_j)$ and
${\sf lhs}(\beta) \subseteq {\sf lhs}(\gamma)$.
Recall that if Algorithm~\ref{alg:dl-lite-horn} calls 
$\exists$-Refine then
there is no $A\in\NC$ such that $\Tmc\models {\sf
lhs}(\gamma)\sqsubseteq A$ and $\Hmc\not\models {\sf
lhs}(\gamma)\sqsubseteq A$.  
So  $\NC \cap {\sf rhs}(\gamma)\subseteq {\sf lhs}(\gamma)$ (by
left saturation of $\gamma$  for \Hmc).
Assume $D \in \NC$. Since  $D\in {\sf rhs}(\alpha'_j)\setminus {\sf
lhs}(\alpha'_j)$ (Point 2), 
it follows that 
$D \not\in {\sf lhs}(\alpha_j)$. As ${\sf lhs}(\alpha'_j)\subseteq
{\sf lhs}(\alpha_j)$, we have that
 $D\in {\sf rhs}(\alpha_j)$.
So $D \in {\sf rhs}(\alpha_j)\setminus {\sf lhs}(\alpha_j)$. 
This means that
$\alpha_j$ has target $\beta$, which contradicts the induction
hypothesis. Otherwise, $D$ is of the form $\exists r.F$.
Then, either $D \in {\sf rhs}(\gamma)$ 
or there is $D' \in {\sf rhs}(\alpha_j)$ such that $\emptyset\models
D\sqsubseteq D'$.
In the latter case, $D'\in {\sf rhs}(\alpha_j)\setminus {\sf
lhs}(\alpha_j)$ and $\emptyset\models {\sf rhs}(\beta)^{\sf
sat}\sqsubseteq D'$, so $\alpha_j$ has target $\beta$, which contradicts the induction hypothesis. 
In the former case,
$\gamma$ has target $\beta$. Then $\Hmc$, $\gamma$, and $\alpha_{i}$ satisfy the conditions of 
 Lemma~\ref{lem:replace02}. 
 Thus, some $\alpha_{i'}$ with $i'\leq i$ is
 replaced which contradicts the assumption that $\alpha_j$ is replaced.
\end{proof}
We have proved the main result of this section.
\begin{theorem}\label{thm:el-horn-subsumptions} 
\ourDLLitehorn TBoxes are polynomial query  
learnable using membership and equivalence queries.
Moreover, \ourDLLitehorn TBoxes without inverse roles can be learned in polynomial time using
membership and equivalence queries.  
\end{theorem}
\begin{proof}
Polynomial query learnability of \ourDLLitehorn TBoxes follows from Lemma~\ref{lem:term200}
and the analysis of the number of membership queries in Lemmas~\ref{lem:t-essential} and \ref{lem:dllite00},
see the proof of Theorem~\ref{DL-Lite}.
For the second part observe that the only potentially costly steps are entailment
checks of the form $\Hmc\models \alpha$, where $\Hmc$ is a \ourDLLitehorn TBox and
$\alpha$ a \ourDLLitehorn CI, both without inverse roles. Then both $\Hmc$ and $\alpha$ are
in $\mathcal{EL}$ with role inclusions for which entailment is known to be in \PTime \citep{BaBrLu-IJCAI-05}.
\end{proof}

\section{Learning $\mathcal{EL}_\mn{lhs}$ TBoxes}\label{sec:final}
We study polynomial learnability of TBoxes in the restriction
$\mathcal{EL}_\mn{lhs}$ of \EL in which only concept names are allowed on the
right-hand side of CIs. We assume that CIs used in membership queries and in
equivalence queries and those returned as counterexamples are also of this
restricted form and show that under this assumption $\mathcal{EL}_{\mn{lhs}}$
TBoxes can be learned in polynomial time. As in the previous section, our
learning algorithm is an extension of the polynomial time algorithm for
learning propositional Horn theories presented by
\citep{DBLP:journals/ml/AngluinFP92,DBLP:journals/ml/AriasB11}.

There is a certain similarity between
the learning algorithm of this section and the \ourDLLitehorn learning
algorithm introduced in Section~\ref{sec:elih-rhs-horn-subsumption}. In both
cases the left-hand side of inclusions can contain complex concept expressions,
which, unless addressed, might lead to several counterexamples with
unnecessarily strong left-hand sides targeting the same inclusion in the target
TBox. In Algorithm~\ref{alg:dl-lite-horn} storing multiple such counterexamples
in $\Hmc_{add}$ is prevented by taking the intersection of the set of conjuncts of the left-hand sides.
To deal with the more complex left-hand sides of inclusions in $\EL_\mn{lhs}$, a more sophisticated
way of `taking the intersection' of concept expressions is required. To define it, we identify
concept expressions with tree-shaped interpretations and then take their product. Products have 
also been employed in the construction of least common subsumers~\citep{DBLP:conf/ijcai/BaaderKM99}.

In detail, we say that an interpretation \Imc is a \emph{ditree interpretation} if the
directed graph $(\Delta^{\Imc},E)$ with $E=\bigcup_{r \in\NR} r^\Imc$ is a
directed tree and $r^\Imc \cap s^\Imc = \emptyset$ for all distinct $r,s \in
\NR$. We denote the root of a ditree interpretation \Imc with $\rho_\Imc$.
The interpretation $\Imc_{C}$ corresponding to an $\mathcal{EL}$
concept expression $C$ is a ditree interpretation with root $\rho_{C}$.
Conversely, every ditree interpretation $\Imc$ can be viewed as an \EL concept
expression $C_{\Imc}$ in the same way as any labelled tree $T$ with edge labels
that are role names (rather than arbitrary roles) can be seen as an \EL concept
expression.

An interpretation \Imc is a \emph{\Tmc-countermodel} for a given
$\EL_\mn{lhs}$ TBox $\Tmc$ if $\Imc\not\models \Tmc$.  Notice that for any
$\EL_\mn{lhs}$ inclusion $C\sqsubseteq A$ with $\Tmc\models C\sqsubseteq A$ and
$\emptyset\not\models C\sqsubseteq A$ the interpretation $\Imc_C$ is a
$\Tmc$-countermodel. Indeed, by construction of $\Imc_C$, we have $\rho_C\in
C^{\Imc_C}$ and, as $\emptyset\not\models C\sqsubseteq A$, we have
$\rho_C\notin A^{\Imc_C}$. So $\Imc_C\not \models C\sqsubseteq A$ and, as $\Tmc\models C\sqsubseteq A$, 
we have $\Imc_C\not\models \Tmc$.
Conversely, given a \Tmc-countermodel $\Imc$, a learning
algorithm can construct in polynomial time in $|\Sigma_\Tmc|$ all inclusions of the form
$C_\Imc\sqsubseteq A$, where $A$ is a concept name, such that $\Tmc\models
C_\Imc\sqsubseteq A$ by posing membership queries to the oracle. 
Thus a learning algorithm can use inclusions and $\Tmc$-countermodels
interchangeably. We prefer working with interpretations as we can then use the notion of products
to define the `intersection of concept expressions' and the results of Section~\ref{sect:prelim} linking homomorphisms 
with entailment in a direct way.

The \emph{product} of two interpretations \Imc and \Jmc is the
interpretation $\Imc \times \Jmc$ with 
\begin{eqnarray*}
\Delta^{\Imc \times \Jmc} & = & \Delta^\Imc \times \Delta^\Jmc\\
A^{\Imc \times \Jmc} & = &  \{(d,e)\mid d\in A^{\Imc},e\in A^{\Jmc}\}\\
r^{\Imc \times \Jmc} & = &  \{((d,e),(d',e'))\mid (d,d') \in r^\Imc, (e,e') \in r^\Jmc\}
\end{eqnarray*}
Products preserve the membership in $\mathcal{EL}$ concept expressions \citep{DBLP:conf/ijcai/LutzPW11}:
\begin{lemma}\label{lem:prod}
For all interpretations $\Imc$ and $\Jmc$, all $d\in \Delta^{\Imc}$ and $e\in \Delta^{\Jmc}$,
and for all $\mathcal{EL}$ concept expressions $C$ the following holds: $d\in C^{\Imc}$ and $e\in C^{\Jmc}$ if, and only if, $(d,e)\in C^{\Imc\times \Jmc}$.
\end{lemma}
One can easily show that the product of ditree
interpretations is a disjoint union of ditree interpretations. 
If $\Imc$ and $\Jmc$ are ditree interpretations, we denote by $\Imc
\times_{\rho} \Jmc$ the maximal ditree interpretation that is a subinterpretation
of $\Imc\times \Jmc$ and contains $(\rho_{\Imc},\rho_{\Jmc})$. 

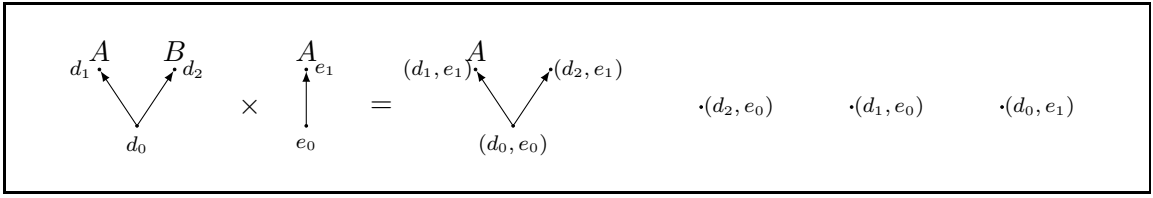
\begin{figure}
\begin{boxedminipage}[h]{\columnwidth}
\vspace*{1em}
\center
\begin{tikzpicture}
	\begin{pgfonlayer}{nodelayer}
		\node [style=newstyle1] (0) at (-8, 3) {};
		\node [style=newstyle1] (1) at (-7.5, 2.25) {};
		\node [style=newstyle1] (2) at (-7, 3) {};
		\node [style=none] (3) at (-6, 2.5) {$\times$};
		\node [style=none] (4) at (-8, 3.25) {$A$};
		\node [style=none] (5) at (-7, 3.25) {$B$};
		\node [style=none] (6) at (-7.5, 2) {\scalefont{0.7}{$d_0$}};
		\node [style=none] (7) at (-5.25, 2) {\scalefont{0.7}{$e_0$}};
		\node [style=none] (8) at (-8.25, 3) {\scalefont{0.7}{$d_1$}};
		\node [style=none] (9) at (-6.75, 3) {\scalefont{0.7}{$d_2$}};
		\node [style=none] (10) at (-3.5, 3) {\scalefont{0.7}{$(d_1,e_1)$}};
		\node [style=none] (11) at (-1.5, 3) {\scalefont{0.7}{$(d_2,e_1)$}};
		\node [style=newstyle1] (12) at (-2, 3) {};
		\node [style=newstyle1] (13) at (-3, 3) {};
		\node [style=newstyle1] (14) at (-2.5, 2.25) {};
		\node [style=none] (15) at (-2.5, 2) {\scalefont{0.7}{$(d_0,e_0)$}};
		\node [style=none] (16) at (-3, 3.25) {$A$};
		\node [style=none] (17) at (-5, 3) {\scalefont{0.7}{$e_1$}};
		\node [style=none] (18) at (-5.25, 3.25) {$A$};
		\node [style=newstyle1] (19) at (-5.25, 3) {};
		\node [style=newstyle1] (20) at (-5.25, 2.25) {};
		\node [style=none] (21) at (-4.25, 2.5) {$=$};
		\node [style=newstyle1] (22) at (4, 2.5) {};
		\node [style=none] (23) at (4.5, 2.5) {\scalefont{0.7}{$(d_0,e_1)$}};
		\node [style=none] (24) at (2.5, 2.5) {\scalefont{0.7}{$(d_1,e_0)$}};
		\node [style=newstyle1] (25) at (2, 2.5) {};
		\node [style=none] (26) at (0.5, 2.5) {\scalefont{0.7}{$(d_2,e_0)$}};
		\node [style=newstyle1] (27) at (0, 2.5) {};
	\end{pgfonlayer}
	\begin{pgfonlayer}{edgelayer}
		\draw [style=newstyle2] (1) to (0);
		\draw [style=newstyle2] (1) to (2);
		\draw [style=newstyle2] (14) to (13);
		\draw [style=newstyle2] (14) to (12);
		\draw [style=newstyle2] (20) to (19);
	\end{pgfonlayer}
\end{tikzpicture}
 \vspace*{1em}
\end{boxedminipage}
\caption{Illustration to Example~\ref{ex:ex99}. \label{fig:ex99}}
\end{figure}
\begin{example}\label{ex:ex99}
Figure~\ref{fig:ex99} depicts the product of the ditree interpretations $\Imc$ with root $d_{0}$ and $\Jmc$ with root 
$e_{0}$. The ditree interpretation $\Imc\times_{\rho}\Jmc$ has root $(d_{0},e_{0})$ and does not contain the nodes
$(d_{2},e_{0})$, $(d_{1},e_{0})$ and $(d_{0},e_{1})$ from $\Imc\times \Jmc$.
\end{example}
Observe that the product $\Imc_{C}\times \Imc_{D}$ of concept expressions $C=A_{1}\sqcap \cdots \sqcap A_{n}$ and $D=B_{1}\sqcap \cdots \sqcap B_{m}$, where
$A_{1},\ldots,A_{n}$ and $B_{1},\ldots,B_{m}$ are concept names, coincides with the interpretation $\Imc_{E}$, where $E$ is the conjunction of all
concept names in $\{A_{1},\ldots,A_{n}\} \cap \{B_{1},\ldots,B_{m}\}$. Thus, products can be seen as a generalisation of taking the intersection
of the concept names from the left-hand side of \ourDLLitehorn concept inclusions used in Section~\ref{sec:elih-rhs-horn-subsumption}.

\smallskip

We will now
describe a class of \Tmc-countermodels that are in a sense minimal and
central to our learning algorithm. 
Let \Tmc be the $\mathcal{EL}_\mn{lhs}$ TBox to be learned, and assume that its
signature $\Sigma_\Tmc$ is known to the learner. 
For a ditree interpretation~\Imc, we
use $\Imc|^-_{\rho}$ to denote the interpretation obtained from
\Imc by removing the root $\rho_\Imc$ of \Imc. For any $d\in \Delta^{\Imc}\setminus\{\rho_{\Imc}\}$, 
we use $\Imc|^-_{d\downarrow}$ to denote \Imc with the subtree rooted at $d$ removed. A \Tmc-countermodel is \emph{essential} if the
following conditions are satisfied:
\begin{enumerate}

\item $\Imc|^-_{\rho} \models \Tmc$;

\item $\Imc|^-_{d\downarrow} \models \Tmc$ for all  $d \in \Delta^\Imc \setminus \{ \rho_\Imc \}$.

\end{enumerate}
Intuitively, Condition~1 states that \Imc contradicts \Tmc only at the
root, that is, the only reason for why \Imc does not satisfy \Tmc is
that for at least one CI $C \sqsubseteq A \in \Tmc$, we have that
$\rho_\Imc \in C^\Imc$ and $\rho_\Imc \notin A^\Imc$. Condition~2 is a
minimality condition which states that for any such $C \sqsubseteq A\in \Tmc$, 
$\rho_\Imc$ is no longer in $C^\Imc$ if we remove any node from \Imc.
Example~\ref{example:elll} at the end of this section shows that working with
essential $\Tmc$-countermodels is needed for our learning algorithm to be in 
polynomial time.

The algorithm for learning $\mathcal{EL}_\mn{lhs}$ TBoxes is given as
Algorithm~\ref{alg:rl2}. 
It maintains an ordered list $\Imf$ of ditree
interpretations that intuitively represents the TBox~\Hmc 
constructed in Line~13.
\begin{algorithm}[t]
\begin{algorithmic}[1]
\setstretch{1.1} 
\Require $\mathcal{EL}_\mn{lhs}$ TBox $\Tmc$ given to the oracle; $\Sigma_\Tmc$ given to the learner.
\Ensure TBox $\Hmc$, computed by the learner, such that $\Tmc\equiv\Hmc$.
\State Set $\Imf$ to the empty list (of ditree interpretations)
\State Set $\Hmc = \emptyset$ 
\While {$\Hmc\not\equiv \Tmc$}
  \State Let $C\sqsubseteq A$ be the returned positive counterexample 
   for \Tmc relative to \Hmc
  \State Find an essential \Tmc-countermodel \Imc with $\Imc \models \Hmc$
  \If {there is a $\Jmc \in \Imf$ such that $\Jmc \not\rightarrow_{\rho} (\Imc \times_{\rho} \Jmc)$ and
    $\Imc \times_{\rho} \Jmc \not\models \Tmc$}
    \State Let $\Jmc$ be the first such element of \Imf 
    \State Find an essential \Tmc-countermodel $\Jmc' \subseteq \Imc
    \times_{\rho} \Jmc$
    \State Replace \Jmc in \Imf with $\Jmc'$ 
  \Else
    \State Append $\Imc$ to $\Imf$
  \EndIf
  \State Construct $\Hmc = \{C_\Imc \sqsubseteq A \mid \Imc \in \Imf, A \mbox{ a concept name in $\Sigma_{\Tmc}$}, \Tmc\models C_\Imc\sqsubseteq A\}$
\EndWhile
\State \textbf{return} $\Hmc$
\end{algorithmic}
\caption{The learning algorithm for $\mathcal{EL}_\mn{lhs}$ TBoxes\label{alg:rl2}}
\end{algorithm}
In Line~6 we write $\Imc\rightarrow_{\rho} \Jmc$ if there is a homomorphism from a
ditree interpretation $\Imc$ to a ditree interpretation $\Jmc$ mapping $\rho_{\Imc}$ to $\rho_{\Jmc}$.
$\Imc\not\rightarrow_{\rho}\Jmc$ denotes that no such homomorphism exists. By Lemma~\ref{lem:hom1},
$\Imc\rightarrow_{\rho} \Jmc$ iff $\emptyset \models C_{\Jmc} \sqsubseteq C_{\Imc}$ which can be 
checked in polynomial time in the size of $\Imc$ and $\Jmc$.
In Line~8, we write $\Jmc'\subseteq \Imc \times_{\rho} \Jmc$ as shorthand for the condition that $\Jmc'$ is a subinterpretation of
$\Imc \times_{\rho}\Jmc$ that is obtained from $\Imc \times_{\rho} \Jmc$ by removing subtrees. 
Note that the assumption in Line~4 that a \emph{positive} counterexample is returned is justified by the
construction of \Hmc in Lines~2 and~13, which ensures that, at all
times, $\Tmc \models \Hmc$. 

We now provide additional details on how to realise 
lines 5, 8 and 13.  
Line~13 is the easiest:
we simply use membership queries `$\Tmc\models C_{\Imc}\sqsubseteq A$?' with $\Imc \in \Imf$ and $A\in \Sigma_\Tmc$
to find all CIs $C_\Imc \sqsubseteq A$ entailed by $\Tmc$. We will
later show that the length of \Imf is bounded polynomially in
$|\Tmc|$ and that each interpretation in $\Imf$ is replaced only polynomially many times, 
therefore polynomially many membership queries suffice. Lines~5 and~8 are addressed by Lemmas~\ref{lem:lem1}
and~\ref{lem:lem2} below.
\begin{lemma}
\label{lem:lem1}
Given a positive counterexample $C\sqsubseteq A$ for \Tmc relative to
\Hmc, one can construct an essential \Tmc-countermodel \Imc with $\Imc
\models \Hmc$ using only polynomially many membership queries in $|\Tmc|+|C|$.
\end{lemma}
\begin{proof}
Let $C \sqsubseteq A$ be a positive counterexample for \Tmc relative
to \Hmc. Let $\Imc_C$ be the ditree interpretation of $C$. 
First observe that $\Imc_C\not\models \Tmc$: since $\Hmc \not\models C \sqsubseteq A$, we know that $A$
does not occur as a top-level conjunct in $C$. Consequently,
$\rho_{C} \in C^{\Imc_C} \setminus A^{\Imc_C}$ and thus $\Imc_C
\not\models \Tmc$. 

We construct an essential \Tmc-countermodel \Imc with $\Imc\models \Hmc$ by applying the following rules to $\Imc:=\Imc_{C}$.
\begin{enumerate}
\item Saturate $\Imc$ by exhaustively applying the CIs from
\Hmc as rules: if $D\sqsubseteq B\in \Hmc$ and $d\in D^{\Imc}$, then add $d$ to $B^{\Imc}$.
\item Replace $\Imc$ by a minimal subtree of $\Imc$ refuting $\Tmc$ to address Condition~1 of essential
\Tmc-countermodels. To describe how this can be achieved using membership queries denote for $d\in \Delta^{\Imc}$ by 
$\Imc|_{d}$ the ditree interpretation obtained from $\Imc$ by taking the subtree of $\Imc$ rooted in $d$. Now check using membership queries 
for any $d\in \Delta^{\Imc}\setminus \{\rho_{\Imc}\}$ and concept name $B$ whether $\Tmc\models C_{\Imc|_d}\sqsubseteq B$.
Then replace $\Imc$ by any $\Imc|_{d}$ such that there exists a $B$ with $\Tmc\models C_{\Imc|_d}\sqsubseteq B$ and $d\not\in B^{\Imc|_d}$
but there does not exist a $d'$ in $\Delta^{\Imc|_{d}}$ and a $B'$ with $\Tmc\models C_{\Imc|_{d'}}\sqsubseteq B'$ and $d'\not\in B^{\Imc|_{d'}}$.
If no such $d$ and $B$ exist, then $\Imc$ is not replaced.
\item Exhaustively remove subtrees from $\Imc$ until 
Condition~2 of essential \Tmc-countermodels is also satisfied: if $\Imc|_{d\downarrow}^{-}\not\models
\Tmc$, then replace $\Imc$ by $\Imc|_{d\downarrow}^{-}$. This can again be achieved using the membership queries
$\Tmc\models C_{\Imc|_{d\downarrow}^{-}}\sqsubseteq B$ for $B$ a concept name. 
\end{enumerate}
Now we show that the interpretation $\Jmc$ constructed above has the required properties.
First observe that $\Jmc\models \mathcal{H}$: clearly, the interpretation $\Imc$ constructed in Step~1 is a model of $\mathcal{H}$.
As taking subtrees and removing subtrees from $\Imc$ preserves being a model of $\mathcal{H}$,
we conclude that $\Jmc\models \mathcal{H}$. Next we show that
$\Jmc\not\models \Tmc$: the interpretation $\Imc$ constructed in Step~1 is not a model of $\Tmc$.
In fact, we can use $C_\Imc \sqsubseteq A$ as a positive counterexample for \Tmc relative to \Hmc instead of 
$C \sqsubseteq A$. Observe that $\emptyset \models C_\Imc \sqsubseteq C$, and thus $\Tmc \models C \sqsubseteq A$
implies $\Tmc \models C_\Imc \sqsubseteq A$. On the other hand, $\rho_\Imc \in B^\Imc$ implies $\Hmc \models C
\sqsubseteq B$ for all concept names $B$. Consequently and since $\Hmc \not\models C \sqsubseteq A$, we have $\rho_\Imc \notin
A^\Imc$. Thus $\Imc\not\models \Tmc$. By construction, Steps~2 and 3 preserve the condition that $\Imc$ is not a model of $\Tmc$ and so 
$\Jmc\not\models\Tmc$.
It remains to argue that $\Jmc$ satisfies Conditions~1 and 2 for essential \Tmc-countermodels for \Hmc.
But Condition~1 is ensured by Step~2 and Condition~2 is ensured by Step~3, respectively.
\end{proof}
\begin{lemma}
\label{lem:lem2}
Given essential \Tmc-countermodels $\Imc$ and $\Jmc$ with $\Imc \times_{\rho} \Jmc \not\models \Tmc$, one can construct an essential
\Tmc-countermodel $\Jmc' \subseteq \Imc \times_{\rho} \Jmc$ using
only polynomially many membership queries in
$|\Tmc|+|\Imc|+|\Jmc|$. 
\end{lemma}
\begin{proof}
Let \Imc and \Jmc be essential \Tmc-countermodels with $\Imc \times_{\rho} \Jmc \not\models \Tmc$. %
Obtain the interpretation $\Jmc'$ from $\Imc\times_{\rho}\Jmc$ by exhaustively applying Rule~3 from the proof
of Lemma~\ref{lem:lem1}. As argued above, applying Rule~3 can be implemented using membership queries
and $\Jmc'$ is a \Tmc-countermodel. Thus, it remains to argue that it satisfies Conditions~1 and 2
for $\Tmc$-essential countermodels. For Condition~1, we have to show that ${\Jmc'}|^-_{\rho} \models \Tmc$. We know that
$\Imc|^-_{\rho} \models \Tmc$ and $\Jmc|^-_{\rho} \models \Tmc$. Thus, by Lemma~\ref{lem:prod}, 
$\Imc|^-_{\rho} \times \Jmc|^-_{\rho} \models \Tmc$. Now $\Jmc'|^-_{\rho}$ is obtained
from $\Imc|^-_{\rho} \times \Jmc|^-_{\rho}$ by removing subtrees and removing subtrees preserves being a model of an $\EL_\mn{lhs}$ TBox.
Thus, $\Jmc'|^-_{\rho}\models \Tmc$.
For Condition~2, we have to show that $\Jmc'|^-_{d\downarrow} \models \Tmc$ for all $d \in \Delta^{\Jmc'} \setminus \{ \rho_{\Jmc'} \}$. But
if this is not the case, then the subtree rooted at $d$ would have been removed during the construction of $\Jmc'$ from $\Imc\times_{\rho}\Jmc$
using Rule~3.
\end{proof}
If Algorithm~\ref{alg:rl2} terminates, then it obviously returns a TBox \Hmc
that is equivalent to the target TBox \Tmc. It thus remains to prove
that the algorithm terminates after polynomially many steps, which is 
a consequence of the following lemma.
\begin{lemma}
\label{lem:term2}
Let $\Imf$ be a list computed at some point of an execution of Algorithm~\ref{alg:rl2}.  
Then 
(i) the length of $\Imf$ is bounded by the number of CIs in $\Tmc$ and
(ii) each interpretation in each position of $\Imf$ 
is replaced only $|\Tmc|+|\Tmc|^{2}$ often with a new interpretation. 
\end{lemma}

The rest of this section is devoted to proving Lemma \ref{lem:term2}. 
For easy reference, assume that at each point of the execution of the
algorithm, \Imf has the form $\Imc_0,\dots,\Imc_k$ for some $k \geq
0$. To establish Point~(i) of Lemma~\ref{lem:term2}, we closely follow \citep{DBLP:journals/ml/AngluinFP92}
and show that
\begin{itemize}

\item[(iii)] for every $\Imc_i$, there is a $D_i \sqsubseteq A_i \in
  \Tmc$ with $\Imc_i \not\models D_i \sqsubseteq A_i$ and 

\item[(iv)] if $i \neq j$, then $D_i \sqsubseteq A_i$ and $D_j
  \sqsubseteq A_j$ are not identical.

\end{itemize}
In fact, Point~(iii) is immediate since whenever a new $\Imc_i$ is
added to \Imf in the algorithm, then $\Imc_i$ is a \Tmc-countermodel.
To prove Point~(iv), we first establish the intermediate
Lemma~\ref{lem:polylength1} below. For a ditree interpretation~\Imc and a CI $C \sqsubseteq A$, we
write $\Imc \models^{\rho} C \sqsubseteq A$ if $\rho_{\Imc}\notin
C^\Imc$ or $\rho_\Imc \in A^\Imc$; that is, the CI $C
\sqsubseteq A$ is satisfied at the root of \Imc, but not necessarily
at other points in \Imc. It is easy to see that if some interpretation
\Imc is a \Tmc-countermodel, then there is  $C \sqsubseteq A \in
\Tmc$ such that $\Imc \not\models^{\rho} C \sqsubseteq A$.

The following lemma shows under which conditions  Algorithm \ref{alg:rl2} replaces
an interpretation in the list \Imf. 

\begin{lemma}
\label{lem:polylength1}
If the interpretation \Imc constructed in Line~5 of Algorithm \ref{alg:rl2}
satisfies $\Imc \not \models^{\rho} C \sqsubseteq A \in \Tmc$ and
$\rho_{\Imc_j} \in C^{\Imc_j}$ for some $j$, then $\Jmc=\Imc_i$ is
replaced with $\Jmc'$ in Line~9 for some $i \leq j$.
\end{lemma}
\begin{proof}
  Assume that the interpretation \Imc constructed in Line~5 of Algorithm \ref{alg:rl2}
   satisfies $\Imc \not \models^{\rho} C \sqsubseteq A \in
  \Tmc$ and that there is some $j$ with $\rho_{\Imc_j} \in
  C^{\Imc_j}$.  If there is some $i < j$ such that $\Imc_i \not\rightarrow_{\rho} (\Imc \times_{\rho} \Imc_i)$ and $\Imc \times_{\rho} \Imc_i
  \not\models \Tmc$, then $\Jmc=\Imc_{i'}$ will be replaced with
  $\Jmc'$ in Line~9 for some ${i'} \leq i$ and we are done. Thus
  assume that there is no such $i$. We aim to show that $\Jmc=\Imc_j$
  is replaced with $\Jmc'$ in Line~9. To this end, it suffices to
  prove that $\Imc_j \not\rightarrow_{\rho} (\Imc \times_{\rho} \Imc_j)$ and $\Imc
  \times_{\rho} \Imc_j \not\models \Tmc$. The latter is a consequence of
  $\Imc \not \models^{\rho} C \sqsubseteq A$ and $\rho_{\Imc_j} \in
  C^{\Imc_j}$.

  Assume to the contrary of what we have to show that
  $\Imc_j \rightarrow_{\rho} (\Imc \times_{\rho} \Imc_j)$.  We establish a
  contradiction against $\Imc \models \Hmc$ (which holds by
  construction of \Imc in the algorithm) by showing that
  \begin{enumerate}

  \item $\Imc \not\models^{\rho} C_{\Imc_j} \sqsubseteq A$ and

  \item $C_{\Imc_j} \sqsubseteq A \in \Hmc$.

  \end{enumerate}
  For Point~1, $\Imc_j \rightarrow_{\rho} (\Imc \times_{\rho} \Imc_j)$ and
  $\rho_{\Imc_j} \in (C_{\Imc_j})^{\Imc_j}$ imply $\rho_{\Imc \times_{\rho}
    \Imc_j} \in (C_{\Imc_j})^{\Imc \times_{\rho} \Imc_j}$, which gives
  $\rho_\Imc \in (C_{\Imc_j})^\Imc$, by Lemma~\ref{lem:prod}. 
  It remains to observe that $\Imc \not \models^{\rho} C \sqsubseteq A$ implies $\rho_{\Imc} \notin
  A^{\Imc}$.

  In view of the construction of \Hmc in the algorithm, Point~2 can be
  established by showing that $\Tmc \models C_{\Imc_j} \sqsubseteq A$.
  Since $C \sqsubseteq A \in \Tmc$, it suffices to prove that
  $\emptyset \models C_{\Imc_j} \sqsubseteq C$. This, however, is an immediate consequence
  of the fact that $\rho_{\Imc_j} \in C^{\Imc_j}$ and the definition of~$C_{\Imc_j}$.
\end{proof}
Now, Point~(iv) above is a consequence of the following.
\begin{lemma}
\label{lem:polylength2}
At any time of the algorithm execution, the following condition holds:
if $\Imc_i \not\models^{\rho} C \sqsubseteq A \in \Tmc$ and $j < i$,
then $\rho_{\Imc_j} \notin C^{\Imc_j}$.
\end{lemma}
\begin{proof}
  We prove the invariant formulated in Lemma~\ref{lem:polylength2} by
  induction on the number of iterations of the while loop. Clearly,
  the invariant is satisfied before the loop is entered. We now
  consider the two places where \Imf is modified, that is, Line~9 and
  Line~11, starting with the latter. 

\smallskip

In Line~11, $\Imc$ is appended to \Imf. Assume that $\Imc \not\models^{\rho}
C \sqsubseteq A \in \Tmc$. We have to show that, before \Imc was added
to \Imf, there was no $\Imc_i \in \Imf$ with $\rho_{\Imc_i} \in
C^{\Imc_i}$. This, however, is immediate by
Lemma~\ref{lem:polylength1}.

\smallskip

Now assume that $\Jmc$ was replaced in Line~9 with $\Jmc'$. We have
to show two properties:
  \begin{enumerate}

  \item If $\Jmc' = \Imc_i \not\models^{\rho} C \sqsubseteq A \in \Tmc$ and
    $j < i$, then $\rho_{\Imc_j} \notin C^{\Imc_j}$.

    Assume to the contrary that $\rho_{\Imc_j} \in C^{\Imc_j}$.  Since
    $\Jmc'$ is obtained from $\Imc \times \Jmc$ by removing subtrees
    (see Lemma~\ref{lem:lem2}), $\Jmc' \not \models^{\rho} C \sqsubseteq A$
    implies $\Imc \times \Jmc \not\models^{\rho} C \sqsubseteq
    A$. Consequently, $\Imc \not\models^{\rho} C \sqsubseteq A$ or $\Jmc
    \not\models^{\rho} C \sqsubseteq A$. The former and $\rho_{\Imc_j} \in
    C^{\Imc_j}$ yields $i \leq j$ by Lemma~\ref{lem:polylength1}, in
    contradiction to $j < i$. In the latter case, since $\Imc_i=\Jmc$
    before the replacement of \Jmc with $\Jmc'$, we have a
    contradiction against the induction hypothesis.

  \item If $\Jmc'=\Imc_j$ and $\Imc_i \not\models^{\rho} C \sqsubseteq A \in
    \Tmc$ with $i>j$, then $\rho_{\Imc_j} \notin C^{\Imc_j}$.

    Assume to the contrary that $\rho_{\Imc_j} \in C^{\Imc_j}$.  Since
    $\Jmc'$ is obtained from $\Imc \times_{\rho} \Jmc$ by removing subtrees,
    we then have $\rho_{\Imc \times_{\rho} \Jmc} \in C^{\Imc \times_{\rho} \Jmc}$, thus
    $\rho_\Jmc \in C^\Jmc$. Since $\Imc_j=\Jmc$ before the replacement
    of \Jmc with $\Jmc'$, we have a contradiction against the induction
    hypothesis.
    
  \end{enumerate}
\end{proof}
We now turn towards proving Point~(ii) of Lemma~\ref{lem:term2}. It is a consequence of Lemma~\ref{lem:small2} below.
\begin{lemma}
\label{lem:small1}
If \Imc is an essential $\Tmc$-countermodel, then $|\Delta^\Imc| \leq |\Tmc|$.
\end{lemma}
\begin{proof}
  Let $\Imc$ be an essential $\Tmc$-countermodel. Then $\Imc \not\models \Tmc$, but
  $\Imc|^-_{\rho} \models \Tmc$. It follows that there is a $C
  \sqsubseteq A \in \Tmc$ such that $\rho_\Imc \in C^\Imc \setminus
  A^\Imc$. %
  By Lemma~\ref{hom}, there is a homomorphism $h$ from $\Imc_C$ to $\Imc$ mapping
  $\rho_{\Imc_C}$ to $\rho_\Imc$. We show that $|\Delta^{\Imc}| \leq |C|$, from which
  $|\Delta^{\Imc}| \leq |\Tmc|$ follows. It suffices to show that $h$ is surjective. 
  Assume that this is not the case and let $d \in \Delta^\Imc$ be outside the range of $h$.
  Then $h$ is a homomorphism from $\Imc_C$ to $\Jmc:= \Imc|^-_{d\downarrow}$. 
  Therefore, $\rho^{\Jmc} \in C^\Jmc$ by Lemma~\ref{hom}, which implies $\Jmc \not\models C \sqsubseteq A$.
  But $\Jmc \not\models C \sqsubseteq A$ contradicts the assumption
  that \Imc is an essential \Tmc-countermodel as it violates Condition~2 of being an essential $\Tmc$-countermodel.
\end{proof}
\begin{lemma}
\label{lem:small2}
Let $\Imc_{0},\ldots,\Imc_{n}$ be a list of interpretations such that $\Imc_{i+1}$ replaces $\Imc_{i}$
in Line~9 for all $i<n$. Then $n\leq |\Tmc|+|\Tmc|^{2}$.
\end{lemma}
\begin{proof}
Let $\Imc$ and $\Jmc$ be ditree interpretations. We set $\Imc \leq_{\rho} \Jmc$ if $\rho_{\Imc}\in A^{\Imc}$ implies $\rho_{\Jmc}\in A^{\Jmc}$
for all concept names $A$. We first show that for every $i<n$ either
  \begin{itemize}

  \item[(a)] $\Imc_{i} \not\leq_{\rho} \Imc_{i+1}$ or

  \item[(b)] $\Imc_{i+1} \rightarrow_{\rho} \Imc_i$ via a surjective homomorphism.

  \end{itemize}
  For a proof by contradiction assume that there is $i<n$ such that neither (a) nor (b) holds.
  Since $\Imc_{i+1}$ is obtained from some $\Imc \times_{\rho} \Imc_{i}$ by removing subtrees and 
  $(\Imc \times_{\rho} \Imc_{i})\rightarrow_{\rho}
  \Imc_{i}$ we obtain that $\Imc_{i+1} \rightarrow_{\rho} \Imc_{i}$.
  Since $\Imc_{i+1}$ is an essential \Tmc-countermodel, there is
  a $C \sqsubseteq A \in \Tmc$ such that $\Imc_{i+1} \not\models^{\rho} C
  \sqsubseteq A$. Let $\Jmc$ be the subinterpretation of $\Imc_{i}$ determined by the range of the homomorphism $h$
  from $\Imc_{i+1}$ to $\Imc_{i}$ mapping $\rho_{\Imc_{i+1}}$ to $\rho_{\Imc_{i}}$. By Lemma~\ref{hom}, $\rho_{\Jmc}\in C^{\Jmc}$ and so, since $\rho_{\Jmc}
  \not\in A^{\Jmc}$ because (a) does not hold, $\Jmc \not\models^{\rho} C \sqsubseteq A$.
  $\Imc_{i}$ is an essential $\Tmc$-countermodel and so $\Jmc=\Imc$. But then $h$ is surjective
  and we have derived a contradiction.
 
\medskip

In addition to the property stated above, we also have for all $i<n$:
\begin{itemize}
\item[(c)]  $\Imc_{i+1} \leq_{\rho} \Imc_{i}$ and 
\item[(d)] $\Imc_{i} \not\rightarrow_{\rho} \Imc_{i+1}$.
\end{itemize}
It follows that for any $i<n$ with $\Imc_{i} \leq_{\rho} \Imc_{i+1}$ either $|\Delta^{\Imc_{i}}| < |\Delta^{\Imc_{i+1}}|$ or 
$|A^{\Imc_{i}}|< |A^{\Imc_{i+1}}|$ for some concept name $A$.  
By Lemma~\ref{lem:small1} we have $|\Delta^{\Imc_{i}}|\leq |\Tmc|$ for all $i\leq n$. Hence $k-j\leq |\Tmc|^{2}$
for any subsequence $\Imc_{j} \ldots,\Imc_{k}$ of $\Imc_{0},\ldots,\Imc_{n}$ with $\Imc_{i} \leq_{\rho} \Imc_{i+1}$ for all $j\leq i<k$.
It follows that $n\leq |\Tmc|+|\Tmc|^{2}$.
\end{proof}
We have thus established the main result of this section. Note that we obtain a polynomial time learning
algorithm since checking $\Tmc\models \alpha$ is in polynomial times for $\mathcal{EL}$ TBoxes $\Tmc$
and $\mathcal{EL}$ CIs $\alpha$ (as discussed in Section~\ref{sect:prelim}).
\begin{theorem}\label{thm:el}
$\mathcal{EL}_\mn{lhs}$ TBoxes are polynomial time learnable using membership and equivalence queries.
\end{theorem}
The following example shows that Algorithm~\ref{alg:rl2} does not terminate in polynomial time
if in Line~5 it does not transform the given counterexample into an essential $\Tmc$-countermodel.

\begin{example}\label{example:elll}\upshape
Assume that Line~5 of Algorithm~\ref{alg:rl2} does not modify the
counterexample $C \sqsubseteq A$ given in Line~4 
if the second condition for essential $\Tmc$-countermodels ($\Imc_{C}|_{d\downarrow}^{-}\models \Tmc$ for all 
$d\in \Delta^{\Imc_{C}}\setminus\{\rho_{\Imc_{C}}\}$) is satisfied but the first condition ($\Imc_{C}|^{-}_{\rho}\models \Tmc$)
does not hold. Then for the target TBox $\Tmc = \{ \exists r. A \sqsubseteq A\}$ the oracle can return the infinite sequence 
of positive counterexamples $\exists r^n . A \sqsubseteq A$, with $n$ a prime number. 
In fact,  Algorithm~\ref{alg:rl2} would simply construct 
the list \Imf of interpretations $\Imc_{\exists r^{n}.A}$, $n$ a prime number, and would not
terminate. To show this observe that Algorithm~\ref{alg:rl2} would never replace a CI in the list \Imf
by another CI since $\Imc_{\exists r^{n}.A} \times_{\rho} \Imc_{\exists r^{n+m}.A}=\Imc_{\exists r^{n}.\top}$ and 
$\Imc_{\exists r^{n}.\top}\models \Tmc$.

Now assume that  %
Line~5 of Algorithm~\ref{alg:rl2} does not modify the counterexample $C \sqsubseteq A$ given in Line~4
if the first condition for essential $\Tmc$-countermodels is satisfied but the second condition
does not hold. Let $\Tmc$ be a TBox containing $\exists r.A \sqsubseteq A$ and some CIs containing the concept
names $B_{1}$ and $B_{2}$, say, for simplicity, $B_{1} \sqsubseteq B_{1}$ and $B_{2}\sqsubseteq B_{2}$.
Let $\varphi^1=\exists r.(B_1\sqcap B_2)$ and
$\varphi^{n+1}=\exists r.(\varphi^{n}
\sqcap B_1\sqcap B_2)$. 
Then
the oracle can return $n$ positive counterexamples
$\exists r.A \sqcap C_i \sqsubseteq A$,
where the tree $T_{C_i}$ corresponding to $C_{i}$ is the result of identifying 
the $i$-th node of the tree $T_{\varphi^{i}}$ corresponding to $\varphi_{i}$ with the root of the tree corresponding to     
$\exists r.( B_1 \sqcap \varphi^{n})\sqcap
\exists r.( B_2 \sqcap \varphi^{n})$. %
Note that the product of
$\Imc_{C_1},\ldots,\Imc_{C_{n}}$ is an interpretation
with $O(2^n)$ elements. 
Then, at the $n$-th iteration, Algorithm~\ref{alg:rl2} computes
an interpretation of exponential size in $n$. 

\end{example}

\section{Limits of Polynomial Learnability}\label{sec:EL_not_learnable}
The main result of this section is that $\EL$ TBoxes 
are not polynomial query learnable 
using membership and equivalence queries. We also show that \ourDLLite TBoxes 
are not polynomial query learnable using membership or equivalence queries
alone. 
The latter result also holds for $\mathcal{EL}_{\mn{lhs}}$ TBoxes.
In this case, however, it follows already from the fact that propositional Horn logic is not 
polynomial query learnable from entailments using membership or equivalence queries alone
\citep{DBLP:conf/icml/FrazierP93,DBLP:journals/ml/AngluinFP92,Ang}.

We start by proving the non-polynomial query learnability result for $\EL$
TBoxes. On our way, we also prove non-polynomial query learnability of
\ourDLLite TBoxes using membership queries only. Our proof shows that even acyclic $\EL$ TBoxes are 
not polynomial query learnable and, in fact, heavily relies on the additional properties of acyclic TBoxes. 
Recall that an $\mathcal{EL}$ TBox is called \emph{acyclic} if it satisfies the following 
conditions \citep{Textbook,DBLP:journals/jair/KonevL0W12}:
   \begin{itemize}
   \item all CIs and CEs are of the form $A \sqsubseteq C$ or $A \equiv C$, where $A$ is a concept name;
   \item no concept name occurs more than once on the left-hand side of a CI;
   \item there are no cyclic definitions: there is no sequence $\alpha_{0},\ldots,\alpha_{n}$ of CIs
   such that the concept name on the left-hand side of $\alpha_{0}$ occurs in $\alpha_{n}$ and the concept name
   on the left-hand side of $\alpha_{i+1}$
   occurs in the right-hand side of $\alpha_{i}$ for all $i<n$.
 \end{itemize} 

Our non-polynomial query learnability proof is inspired by Angluin's lower bound for the following
abstract learning problem \citep{DBLP:journals/ml/Angluin87}: a learner
aims to identify one of $N$ distinct sets $L_1,\dots,L_N$
which have the property that there exists a set $L_\cap$ for which
$L_i\cap L_j = L_\cap$, for any $i\neq j$. It is assumed that $L_\cap$
is not a valid argument to an equivalence query. The learner can pose
membership queries ``$x\in L$?'' and equivalence queries ``$H=L$?''.
Then in the worst case it takes at least $N-1$ membership and
equivalence queries to exactly identify a hypothesis $L_i$ from
$L_1,\dots,L_N$. The proof proceeds as follows.
At every stage of computation, the oracle (which here should be viewed
as an adversary) maintains a set of hypotheses $S$, which the learner is not able to
distinguish based on the answers given so far.  Initially, $S =
\{L_1,\dots,L_N\}$.  When the learner asks a membership query $x$, the
oracle returns 'Yes' if $x\in L_\cap$ and 'No' otherwise. In the
latter case, the (unique) $L_i$ such that $x\in L_i$ is removed from
$S$. When the learner asks an equivalence query $H$, the oracle
returns `No' and a counterexample $x \in L_\cap\oplus H$ (the
symmetric difference of $L_{\cap}$ and $H$).  This always exists as
$L_\cap$ is not a valid query. If the counterexample $x$ is not a
member of $L_\cap$, (at most one) $L_i\in S$ such that $x\in L_i$ is
eliminated from $S$. In the worst case, the learner has to reduce the
cardinality of $S$ to one to exactly identify a hypothesis, which
takes $N-1$ queries.

Similarly to the method outlined above, in our proof we maintain a set
of acyclic $\mathcal{EL}$ TBoxes whose members the learning
algorithm is not able to distinguish based on the answers obtained so
far.
For didactic purposes, we first present a set of
acyclic TBoxes
 $S_N = \{\Tmc_1,\dots,\Tmc_N\}$, 
where $N$ is
superpolynomial in the size of every TBox $\Tmc_i$,
for which the oracle can respond to
membership queries in the way described above but which is
polynomial time
learnable when equivalence queries
are also allowed.  We then show how the TBoxes can be modified to obtain a
family of acyclic TBoxes that is not polynomial query learnable using membership
and equivalence queries.  

To present the TBoxes in $S_N$, fix two role names $r$ and $s$.
We use the following abbreviation.
For any
sequence $\sigma=\sigma^1\sigma^2\dots\sigma^n\in \{r,s\}^{n}$,
the expression $\exists \sma.C$
stands for $\exists \sigma^1.\exists\sigma^2\dots\exists\sigma^n.C$.
Then for every such sequence $\sma$, of which there are $N = 2^{n}$ many,
consider the acyclic $\mathcal{EL}$ TBox $\Tmc_\sigma$ defined as
$$
\begin{array}{rcl}
\Tmc_\sma &=& \left\{
A\sqsubseteq \exists {\sma}.M\sqcap X_0
\right\}\cup\Tmc_0 \; \text{ with }\\[1mm]
 \Tmc_0 &=& \left\{X_i\sqsubseteq \exists r.X_{i+1}\sqcap \exists s.X_{i+1} \mid 0\leq i < n\right\}.
\end{array}
$$
Observe that the canonical model $\Imc_{X_{0},\Tmc_{0}}$ of $X_{0}$ and $\Tmc_{0}$ 
consists of a full binary tree whose edges are labelled
with the role names $r$ and $s$ and with $X_0$ at the root $\rho_{X_{0}}$, $X_1$ at
level~$1$, and so on. In the canonical model $\Imc_{A,\Tmc_{\sma}}$ of $A$ and $\Tmc_{\sma}$,
the root is labelled by $A$ and $X_{0}$ and, \emph{in addition to} the binary tree, 
there is a path given by the sequence $\sma$ whose endpoint is marked by the concept name $M$.

One can use Angluin's strategy to show that TBoxes from the set $S_N$ of all such
TBoxes $\Tmc_{\sma}$ cannot be learned using polynomially many polynomial size 
membership queries only: notice that for no sequence $\sma'\neq\sma$
of length $n$, we have $\Tmc_\sma\models A\sqsubseteq \exists \sma'.M$.
Thus a membership query of the form $A\sqsubseteq \exists \sma.M$
eliminates at most one TBox from the set of TBoxes that the learner
cannot distinguish.  This observation can be generalised to arbitrary
membership queries $C \sqsubseteq D$ in $\mathcal{EL}$; however,
we instead observe that the TBoxes $\Tmc_{\sma}$ are formulated in
\ourDLLite and prove a stronger result.
The proof, given in the appendix, uses the canonical model construction introduced in Section~\ref{sect:prelim}.
\begin{lemma}\label{lem:dllite-non-learn}
For every $\ourDLLite$ CI $B \sqsubseteq D$ over the signature of $\Tmc_{\sma}$,
\begin{itemize}
\item either $\Tmc_{\sma} \models B \sqsubseteq D$ for every $\Tmc_{\sma}\in S_N$
\item or there is at most one $\Tmc_{\sma}\in S_N$ such that $\Tmc_{\sma}\models B \sqsubseteq D$.
\end{itemize}  
\end{lemma}
The argument outlined above immediately gives us the following side result.
\begin{theorem}\label{DLLiteequiv}
\ourDLLite TBoxes (even without inverse roles) are not polynomial query learnable using only membership queries.
\end{theorem}

We return now to our proof that $\EL$ TBoxes are not polynomial query learnable
using both membership and equivalence queries.  Notice that the set of TBoxes
$S_N$ is not suitable as a single equivalence query is sufficient to learn any
TBox from $S_N$ in two steps: given the equivalence query $\{A\sqsubseteq
X_0\}\cup\Tmc_0$, the oracle has no other option but to reveal the target TBox
$\Tmc_\sma$ as $A\sqsubseteq \exists \sma.M$ can be found `inside' every
counterexample.

Our strategy to rule out equivalence queries with the `intersection TBox' is to
modify $\Tmc_1,\dots,\Tmc_N$ in such a way that although a TBox $\Tmc_\cap$
axiomatising the intersection over the set of consequences of each $\Tmc_{i}$,
$i\leq N$, exists, its size is superpolynomial and so it cannot be used as an
equivalence query by a polynomial query learning algorithm.

For every $n>0$ and every $n$-tuple $L = (\sma_1,\dots,\sma_n)$,
where every $\sma_i$ is a role sequence of length $n$ as above, we
define
an acyclic $\mathcal{EL}$ TBox $\Tmc_L$ as the union of $\Tmc_0$ and
the following CIs and CEs:\footnote{In fact, to prove non-polynomial query learnability, it
suffices to consider $\exists {\sma_1}. M\sqcap \dots \sqcap \exists {\sma_n}.
M\sqsubseteq A$ in place of the concept equivalence; however, CIs of this form are
not allowed in acyclic TBoxes. 
CIs with a complex left-hand
side or concept equivalences are essential for non-polynomial query learnability as any
acyclic TBox containing expressions of the form $A\sqsubseteq C$ only is
a \ourDLLite TBox and thus polynomially learnable with membership and equivalence queries (Section~\ref{sec:dllite}).
}
$$
\begin{array}{c}
\begin{array}{c}
A_1\sqsubseteq \exists {\sma_1}.M\sqcap X_0\\
B_1\sqsubseteq \exists {\sma_1}.M\sqcap X_0
\end{array}
\quad\dots\quad
\begin{array}{c}
A_n\sqsubseteq \exists {\sma_n}.M\sqcap X_0\\
B_n\sqsubseteq \exists {\sma_n}.M\sqcap X_0
\end{array}\\[1em]
A\equiv X_0\sqcap \exists {\sma_1}. M\sqcap \dots \sqcap \exists {\sma_n}. M.
\end{array}
$$
Observe that every $\Tmc_{L}$ contains the TBoxes $\Tmc_{\sma_{i}}$, $1 \leq i \leq n$, discussed above with
$A$ replaced by any of the three concept names $A,A_{i},B_{i}$. In addition, \emph{every} $\Tmc_L$ entails, among other CIs,
$\bigsqcap_{i=1}^n C_i\sqsubseteq A$, where every $C_i$ is either
$A_i$ or $B_i$. There are $2^n$ different such CIs, which indicates that every representation of the `intersection TBox'
requires superpolynomially many axioms.  It follows from
Lemma~\ref{lem:EQ} below that this is indeed the case.

Let $\Lmf_n$ be a set of $n$-tuples such that for $1\leq i\leq n$ and
every $L,L'\in\Lmf_n$ with $L = (\sma_1,\dots,\sma_n)$, $L' =
(\sma_1',\dots,\sma_n')$, if $\sma_i = \sma'_j$ then $L = L'$ and $i=j$.  
Then for any sequence $\sma$ of length $n$ there exists at most one 
$L\in\Lmf_n$ and at most one $i \leq n$ such that 
$\Tmc_L\models A_i\sqsubseteq \exists \sma.M$ and 
$\Tmc_L\models B_i\sqsubseteq \exists \sma.M$.
We can choose $\Lmf_n$ such that there are
$N = \lfloor 2^n/n \rfloor$ different tuples in $\Lmf_n$. Notice that the size of
each $\Tmc_{L}$ with $L\in \Lmf_{n}$ is polynomial in $n$ and so $N$
is superpolynomial in the size of each $\Tmc_L$ with $L\in\Lmf_n$.
Let the set of TBoxes that the
learner cannot distinguish 
initially be $S_\Lmf = \{\Tmc_L\mid L\in\Lmf_n\}$.
We use $\Sigma_{n}$ to denote the signature of~$\Tmc_L$.

For the proof of non-polynomial query learnability, we show that 
the oracle has a strategy to answer both membership and
equivalence queries without eliminating too many TBoxes from $S_\Lmf$.  We
start with the former.

Unlike the \ourDLLite case presented above, membership query can eliminate more
than one TBox from $S_\Lmf$.  Consider, for example, two TBoxes $\Tmc_{L}$ and
$\Tmc_{L'}$, where 
$\{L,L'\}\subseteq \Lmf_n$ with
$L =(\sma_1,\dots,\sma_n)$ and
$L'=(\sma'_1,\dots,\sma'_n)$. Then the CI 
$$
X_0\sqcap \exists \sma_1.M\sqcap\exists \sma'_1.M\sqcap A_2\sqcap \dots\sqcap
A_n\sqsubseteq A
$$
is entailed by both $\Tmc_L$ and $\Tmc_{L'}$ but not
by any other $\Tmc_{L''}$ with $L'' \in \Lmf_n$.
We prove, however, that the number of TBoxes eliminated from $S_\Lmf$ by a single
membership query can be linearly bounded by the size of the query.
\begin{lemma}\label{lem:MQ}
For all $\mathcal{EL}$ CIs $C\sqsubseteq D$ over~$\Sigma_{n}$:
\begin{itemize}
\item either $\Tmc_L\models C\sqsubseteq D$ for every $L\in\Lmf_n$ 
\item or the number of $L\in\Lmf_n$ such that $\Tmc_L\models C\sqsubseteq D$
does not exceed $|C|$.
\end{itemize}
\end{lemma}
The proof of Lemma~\ref{lem:MQ} is technical and is deferred  to the appendix.
To illustrate our proof method here we consider a particular case that deals
with membership queries of the form $C\sqsubseteq \exists \sma.M$ and is 
used in the proof of the general case.
Both proofs rely on the following lemma
from~\citep{DBLP:journals/jair/KonevL0W12} that characterises CIs entailed by
acyclic $\mathcal{EL}$ TBoxes.
\begin{lemma}[\citeauthor{DBLP:journals/jair/KonevL0W12}]\label{basic}
Let $\Tmc$ be an acyclic $\mathcal{EL}$ TBox, $r$ a role name and $D$ an
$\mathcal{EL}$ concept expression.  Suppose that $ \Tmc \models 
\bigsqcap_{1\leq i\leq n} A_{i} \sqcap \bigsqcap_{1\leq j\leq m} \exists
r_{j}.C_{j} \sqsubseteq D$,
where $A_i$ are concept names for $1\leq i \leq n$, $C_j$ are $\mathcal{EL}$
concept expressions for $1\leq j\leq m$, and $m, n \geq 0$. Then the following holds:
\begin{itemize}
\item if $D$ is a concept name such that $\Tmc$ does not contain any CE
$D\equiv C$ for any concept expression $C$, then there exists $A_i$, $1\leq
i\leq n$, such that $\Tmc \models A_i \sqsubseteq D$;
\item if $D$ is of the form $\exists r.D'$ then
either (i) there exists $A_i$, $1\leq i\leq n$, such that $\Tmc \models A_i
\sqsubseteq \exists r.D'$ or (ii) there exists $r_j$, $1\leq j \leq m$, such
that $r_{j}=r$ and $\Tmc \models C_j\sqsubseteq D'$.
\end{itemize}
\end{lemma}
The following lemma considers membership queries of the form $C\sqsubseteq \exists\sma.M$.
\begin{lemma}\label{lem:component}
For any $0\leq m\leq n$, any sequence of role names $\sma =
\sigma^1\dots\sigma^m \in \{r,s\}^{m}$,
and any $\mathcal{EL}$ concept expression $C$ over $\Sigma_{n}$:
\begin{itemize}
\item either  $\Tmc_{L} \models C \sqsubseteq \exists \sma.M$
for every $\Tmc_{L}$ with $L\in\Lmf_n$;
\item or there is at most one $\Tmc_{L}$ such that $\Tmc_{L}\models C \sqsubseteq \exists \sma.M$.
\end{itemize}
\end{lemma}
\begin{proof}
The lemma follows from the following claim.

\medskip
\noindent
{\bf Claim}.  
Let $L = (\sma_1,\dots,\sma_n)\in\Lmf_n$ be such that
$\Tmc_L\models C\sqsubseteq \exists\sma.M$. Then \emph{either}
(1) there exists $i\leq n$ such that $\sma=\sma_i$ and $C$ is of the 
form $A\sqcap C'$, $A_i\sqcap C'$ or $B_i\sqcap C'$, for some $\mathcal{EL}$ concept expression 
$C'$; \emph{or} (2) we have 
 $\emptyset \models C\sqsubseteq \exists \sma.M$.

\medskip
\noindent
\emph{Proof of Claim}.
We prove the claim by induction on $m$.
If $m=0$, by Lemma~\ref{basic}, the concept expression $C$ is of the form
$Z\sqcap C'$, for some concept name $Z$ and concept expression $C'$ such that 
$\Tmc_L\models Z\sqsubseteq M$. As $\Tmc_{L}\models Z \sqsubseteq M$ does not hold for any concept name
$Z$ distinct from $M$, we obtain $Z=M$. Thus, $\emptyset\models C\sqsubseteq M$ and Point~(2) follows.

\noindent Let $m>0$. By Lemma~\ref{basic} we have one of the following two cases:
\begin{itemize}
\item $C$ is of the form $X\sqcap C'$, for some concept name $X$ and
concept expression $C'$ such that $\Tmc_L\models X\sqsubseteq \exists \sma.M$. But then
there exists $i\leq n$ such that $\sma = \sma_i$ and $X\in \{A,A_{i},B_{i}\}$ and Point~(1) follows.
\item $C$ is of the form $\exists \sigma^1.C'\sqcap C''$, for some concept expressions $C'$ and $C''$,
and $\Tmc_L\models C'\sqsubseteq \exists \sigma^2.\cdots\exists
\sigma^m.M$. Notice that the length of the sequence 
$\sigma^2\dots\sigma^n$ is strictly less than $n$. Thus, by induction hypothesis,
$\emptyset\models C'\sqsubseteq \exists \sigma^2.\cdots\exists \sigma^m.M$. But then 
$\emptyset\models C\sqsubseteq \exists \sma.M$ and Point~(2) follows.
\end{itemize}
This finishes the proof of the claim.
To see that the claim entails the lemma observe that at most one $L\in \Lmf_{n}$ can satisfy Point~(1).
Point~(2) entails that $\Tmc_{L}\models  C\sqsubseteq \exists \sma.M$ for every $\Tmc_{L}$ with $L\in \Lmf_{n}$.
\end{proof}
We now show how the oracle can answer equivalence queries, aiming to
show that for any polynomial size equivalence query~$\Hmc$, the oracle
can return a counterexample $C\sqsubseteq D$ such that either (i)
$\Hmc\models C\sqsubseteq D$ and $\Tmc_{L} \models C \sqsubseteq D$
for at most one $L\in\Lmf_n$ or (ii) $\Hmc\not\models C\sqsubseteq D$
and $\Tmc_L\models C\sqsubseteq D$ for every $L\in\Lmf_n$. Thus, such a counterexample eliminates at most one $\Tmc_L$ from
the set $S_\Lmf$ of TBoxes that the learner cannot distinguish. In
addition, however, we have to take extra care of the size of
counterexamples as the learning algorithm is allowed to formulate queries
polynomial not only in the size of the target TBox but also in the
size of the counterexamples returned by the oracle.  For instance, if
the hypothesis TBox $\Hmc$ contains a CI $C\sqsubseteq D$
which is not entailed by any $\Tmc_L$,
one cannot simply return $C\sqsubseteq D$ as a counterexample since
the learner will be able to `pump up' its capacity by 
asking a sequence of equivalence queries $\Hmc_i =
\{C_i\sqsubseteq D_i\}$ such that the size of $C_{i+1}\sqsubseteq
D_{i+1}$ is twice the size of $C_{i}\sqsubseteq
D_{i}$. Then at every stage in a run of the learning
algorithm, the query size will be polynomial in the size of the
input and the size of the largest counterexample received so far, but
exponential size queries will become available to the learner. The
following lemma addresses this issue.
\begin{lemma}\label{lem:EQ}
For any $n>1$ %
and any $\mathcal{EL}$ TBox
$\Hmc$ in $\Sigma_{n}$ with $|\Hmc|< 2^n$, there exists an $\mathcal{EL}$ CI
$C\sqsubseteq D$ over $\Sigma_{n}$ such that 
the size of $C\sqsubseteq D$ does not exceed $6n$ and 
\begin{itemize}
\item if $\Hmc\models C\sqsubseteq D$, then $\Tmc_L\models C\sqsubseteq D$ for at most one $L\in\Lmf_n$;
\item if $\Hmc\not\models C\sqsubseteq D$, then $\Tmc_{L}\models C \sqsubseteq \ D$ for every $L\in\Lmf_n$.
\end{itemize}
\end{lemma}
\begin{proof}
We define an exponentially large TBox $\Tmc_{\cap}$ and use it to prove that
one can select the required $\mathcal{EL}$ CI $C\sqsubseteq D$ in such a
way that either $\Hmc\models C\sqsubseteq D$ and $\Tmc_\cap\not\models
C\sqsubseteq D$, or vice versa.

To define $\Tmc_\cap$, denote for any sequence ${\boldsymbol{b}} = b_1\dots b_n\in \{0,1\}^{n}$
by $C_{{\boldsymbol{b}}}$ the conjunction $\bigsqcap_{i\leq n} C_{i}$,
where $C_{i}= A_{i}$ if $b_{i}=1$ and $C_{i}=B_{i}$ if $b_{i}=0$.
Then we define
$$
\Tmc_{\cap} = \Tmc_{0} \cup \{ C_{{\boldsymbol{b}}}\sqsubseteq A\sqcap X_0 \mid {\boldsymbol{b}}\in \{0,1\}^{n}\}.
$$
Consider the following cases for $\Hmc$ and $\Tmc_\cap$.
\begin{enumerate}
\item Suppose $\mathcal{H} \not\models \Tmc_{\cap}$. Then there exists a CI $C\sqsubseteq D\in\Tmc_\cap$ such
that $\Hmc\not\models C\sqsubseteq D$. Clearly, 
$C\sqsubseteq D$ 
is entailed by every $\Tmc_L$, for $L\in\Lmf_n$, and
the size of 
$C\sqsubseteq D$ does not exceed $6n$. Thus $C\sqsubseteq D$ is as required.

\item Suppose there exist ${\boldsymbol{b}}\in\{0,1\}^n$ and a concept expression of the form 
$\exists t.D'$ such that  
$\Hmc\models C_{\boldsymbol{b}}\sqsubseteq \exists t.D'$ and $\Tmc_{0}\not\models C_{\boldsymbol{b}}\sqsubseteq \exists t.D'$.
It can be seen (Lemma~\ref{lem:claim} in the appendix),
that there exists a sequence of role names $t_1,\dots, t_l\in \{r,s\}^{l}$ with $0\leq l\leq
n+1$ and $Y\in \{\top\}\cup \NC$ such that 
$\emptyset\models \exists t.D'\sqsubseteq \exists t_1.\cdots\exists t_l.Y$. Thus,
$\Hmc\models C_{\boldsymbol{b}}\sqsubseteq \exists t_1.\cdots\exists t_l.Y$ 
and
$\Tmc_0\not\models X_0\sqsubseteq \exists t_1.\cdots\exists t_l.Y$.
We show that the inclusion $C_{\boldsymbol{b}}\sqsubseteq \exists t_1.\cdots\exists t_l.Y$
is as required. Clearly, the size of $C_{\boldsymbol{b}}\sqsubseteq \exists t_1.\cdots\exists t_l.Y$ 
does not exceed $6n$. It remains to prove that $\Tmc_L\models  C_{\boldsymbol{b}}\sqsubseteq
\exists t_1\cdots\exists t_l.Y$ for at most one $L\in\Lmf_n$.

Suppose there exists $L\in\Lmf_n$ such that 
$\Tmc_L\models C_{\boldsymbol{b}}\sqsubseteq \exists t_1.\cdots\exists t_l.Y$.
By Lemma~\ref{basic}, there exists $A_j$ or $B_j$ such that
$\Tmc_L\models A_j\sqsubseteq \exists t_1.\cdots\exists t_l.Y$ 
or $\Tmc_L\models B_j\sqsubseteq \exists t_1.\cdots\exists t_l.Y$, respectively.
As $\Tmc_0\not\models X_0\sqsubseteq \exists t_1.\cdots\exists t_l.Y$ it is easy to 
see that $l=n$, $t_1t_2\dots t_n = \sma_j$, and $Y=M$ follow. 
As $\Tmc_{L'}\not\models C_{\boldsymbol{b}}\sqsubseteq \exists\sma_j.M$
for any $L'\in\Lmf_n$ such that $L'\neq L$, it follows that
$\Tmc_L\models  C_{\boldsymbol{b}}\sqsubseteq \exists t_1\cdots\exists t_l.Y$ for at most one $L\in\Lmf_n$

\item Finally, suppose that neither Case~1 nor 2 above apply. Then $\mathcal{H}
\models \Tmc_{\cap}$ and for every ${\boldsymbol{b}}\in\{0,1\}^n$ and every
$\mathcal{EL}$ concept expression over $\Sigma_{n}$ of the form $\exists t. D'$: if $\Hmc\models
C_{\boldsymbol{b}}\sqsubseteq \exists t.D'$ then $\Tmc_0\models X_0\sqsubseteq
\exists t.D'$.
We show that unless there exists 
a CI $C\sqsubseteq D$ satisfying the conditions of the lemma,
$\Hmc$ contains at least $2^n$ different CIs (and thus derive a contradiction).

Fix some ${\boldsymbol{b}} = b_1\dots b_n\in \{0,1\}^{n}$.
From $\mathcal{H} \models \Tmc_{\cap}$ we obtain $\Hmc\models C_{\boldsymbol{b}}\sqsubseteq A$. 
Then there must exist at least one CI $C\sqsubseteq A\sqcap D\in\Hmc$
such that $\Hmc\models C_{\boldsymbol{b}}\sqsubseteq C$ and $\emptyset\not\models C\sqsubseteq A$.
Let $C = Z_1\sqcap\dots\sqcap Z_m\sqcap \exists t_1.C'_1\sqcap\dots\sqcap \exists t_l.C'_l$, 
where $Z_1$,\dots, $Z_m$ are different concept names.
As $\Hmc\models C_{\boldsymbol{b}}\sqsubseteq \exists t_j.C'_j$ we have
$\Tmc_0\models X_0\sqsubseteq \exists t_j.C'_j$, for $j=1,\dots l$. 
As $\Hmc\models\Tmc_\cap$ we have $\Hmc\models X_0\sqsubseteq \exists
t_j.C'_j$, for $j=1,\dots l$. So $\Hmc\models Z_1\sqcap\dots \sqcap Z_m\sqcap X_0\sqsubseteq A$.
\begin{itemize}
\item Suppose there exists $i$ such that there is no $Z_{j}\in \{A_i,B_i\}$. Then we have $\Tmc_L\not\models
Z_1\sqcap\dots \sqcap Z_m\sqcap X_0\sqsubseteq A$, for any $L\in\Lmf_n$. 
Notice that $Z_1\sqcap\dots\sqcap Z_m$ contains at most all concepts names in $\Sigma_n$, except $A_i$, $B_i$.
Thus, the size of $Z_1\sqcap\dots\sqcap Z_m\sqcap X_0\sqsubseteq A$ does not exceed $6n$, and
$Z_1\sqcap\dots \sqcap Z_m\sqcap X_0\sqsubseteq A$ is as required.

\item Assume that $Z_{0}\sqcap \dots \sqcap Z_{m} \sqcap X_{0}$ contains a conjunct $B_{i}$ 
such that $b_{i}\not=0$.
Then $\Hmc\models C_{\boldsymbol{b}}\sqsubseteq B_{i}$ and there is no $L\in\Lmf_n$ such that
$\Tmc_L\models C_{\boldsymbol{b}}\sqsubseteq B_{i}$. The size of $C_{\boldsymbol{b}}\sqsubseteq B_{i}$ 
does not exceed $6n$, so $C_{\boldsymbol{b}}\sqsubseteq B_{i}$ is as required.  

\item Assume that $Z_{0}\sqcap \dots \sqcap Z_{m} \sqcap X_{0}$ contains a conjunct $A_{i}$ 
such that $b_{i}\not=1$.
Then $\Hmc\models C_{\boldsymbol{b}}\sqsubseteq A_{i}$ and there is no $L\in\Lmf_n$ such that
$\Tmc_L\models C_{\boldsymbol{b}}\sqsubseteq A_{i}$. The size of $C_{\boldsymbol{b}}\sqsubseteq A_{i}$ 
does not exceed $6n$, so $C_{\boldsymbol{b}}\sqsubseteq A_{i}$ is as required.  

\item If none of the above applies, then $Z_{1}\sqcap\dots \sqcap Z_m\sqcap X_0$ contains exactly the $A_{i}$
with $b_{i}=1$ and exactly the $B_{i}$ with $b_{i}=0$.
\end{itemize}
This argument applies to arbitrary ${\boldsymbol{b}}\in\{0,1\}^n$. Thus, if there exists no CI 
$C\sqsubseteq D$ satisfying the conditions of the lemma then, by the final case,  
$\Hmc$ contains at least $2^n$ CIs.
\end{enumerate}
\end{proof}
Now we have all the ingredients to prove 
that $\EL$ TBoxes 
are not polynomial query learnable 
using membership and equivalence queries.
\begin{theorem}\label{th:non-poly-learn-EL}
$\mathcal{EL}$ TBoxes are not polynomial query learnable using
membership and equivalence queries.
\end{theorem}
\begin{proof}
  Assume that TBoxes are polynomial query
  learnable. Then there exists a learning algorithm whose query complexity
  (the sum of the sizes of the inputs to membership and equivalence queries
  made by the algorithm up to a computation step) is bounded at any stage by a
  polynomial $p(n,m)$.  Choose $n$ such that $\lfloor
  2^n/n\rfloor>(p(n,6n))^2$ and let $S_\Lmf
  = \{ \Tmc_L
  \mid L \in \Lmf_n \}$. We follow Angluin's strategy of letting the
  oracle remove TBoxes from $S_\Lmf$ in such a way that the learner cannot
  distinguish between any of the remaining TBoxes.  Given a 
  membership query $C\sqsubseteq D$, if $\Tmc_L\models C\sqsubseteq D$
  for every $L\in\Lmf_n$, then the answer is `yes'; otherwise the
  answer is `no' and all $\Tmc_L$ with $\Tmc_L\models C\sqsubseteq D$
  are removed from $S_\Lmf$ (by Lemma~\ref{lem:MQ}, there are at most $|C|$
  such TBoxes).
  Given an equivalence query $\Hmc$, the answer is `no', a
  counterexample $C\sqsubseteq D$ guaranteed by Lemma~\ref{lem:EQ} is
  produced, and (at most one) $\Tmc_L$ such that $\Tmc_L\models
  C\sqsubseteq D$ is removed from $S_\Lmf$.

  As all counterexamples produced are smaller than $6n$, the overall
  query complexity of the algorithm is bounded by $p(n,6n)$.  Hence, the
  learner asks no more than $p(n,6n)$ queries and the size of every
  query does not exceed $p(n,6n)$. By Lemmas~\ref{lem:MQ} and
  \ref{lem:EQ}, at most $(p(n,6n))^2$ TBoxes are removed from $S_\Lmf$
  during the run of the algorithm.  But then, the algorithm cannot
  distinguish between any remaining TBoxes and we have derived a
  contradiction.
\end{proof}

We conclude this section by  showing that \ourDLLite TBoxes cannot be learned
using polynomially many polynomial size equivalence queries only.  We use the
following result on non-polynomial query learnability of
monotone DNF formulas, that is, DNF formulas that do not use negation, using
equivalence queries due to \citet{DBLP:journals/ml/Angluin90}.
Here, equivalence queries take a hypothesis $\psi$ in the form of a monotone DNF formula
and return as a counterexample either a truth assignment that satisfies $\psi$
but not the target formula $\phi$ or vice versa.
Let $M(n,t,s)$ denote the set of all monotone DNF formulas whose variables are
$x_1,\dots,x_n$, that have exactly $t$ conjunctions, and where each conjunction
contains exactly $s$ variables.

\begin{theorem}[Angluin]
\label{lem:angluin}
For any polynomial $q(\cdot)$ there exist constants $t_0$ and $s_0$ and a
strategy 
\footnote{
The existence of this strategy is a direct consequence of Theorem~8
of~\citep{DBLP:journals/ml/Angluin90}, which states that the class of DNF
formulae has the \emph{approximate fingerprint} property, and the proof of
Theorem~1 of~\citep{DBLP:journals/ml/Angluin90}, where such a strategy is
explicitly constructed for any class having approximate fingerprints.}
for the oracle $\Omf$ to answer equivalence queries posed by a learning algorithm in
such a way that for sufficiently large $n$ any learning algorithm that asks at
most $q(n)$ equivalence queries, each bounded in size by $q(n)$, cannot exactly
identify elements of $M(n,t_0,s_0)$.
\end{theorem}
To employ Theorem~\ref{lem:angluin}, we associate with every
monotone DNF formula 
$$
    \phi= \bigvee_{i=1}^t(x_1^i\land\dots\land x_{s^i}^i),
$$
where $\{x_1^i,\dots, x_{s^i}^i\}\subseteq \{x_{1},\ldots,x_{n}\}$,
a $\ourDLLite$ TBox
$\Tmc_\phi$ as follows. With each conjunct 
$x_1^i\land\dots \land x_{s^i}^i$ we associate a concept expression 
$$C_i :=\exists \rho_1^i.\exists \rho_2^i.\dots\exists \rho_n^i.\top,$$ 
where $\rho_j^i= r$ if
$x_j$ occurs in $x_1^i\land\dots \land x_{s^i}^i$ and $\rho_j^i=\bar r$
otherwise ($r$ and $\bar{r}$ are role names). Let $A$ be a concept name and set
\begin{center}
$
\Tmc_\phi = \{A \sqsubseteq \bigsqcap_{i=1}^t C_{i}, \quad \bar r\sqsubseteq r\}.
$
\end{center}
For example, for $n=4$ and $\phi = (x_1\land x_4)\lor x_2$ we have
$$
\Tmc_\phi = \{
   A  \sqsubseteq  \exists r.\exists \bar r.\exists \bar r.\exists r.\top, \;
   A  \sqsubseteq  \exists \bar r.\exists r.\exists \bar r.\exists \bar r.\top,\;
\bar r\sqsubseteq r\}.
$$
We say that a TBox \Tmc \emph{has a DNF-representation for $n$} if it is obtained by the translation of a monotone DNF-formula
with $n$ variables; that is, if \Tmc is of the following form, for some
$\Gamma \subseteq \{r,\bar r\}^n$:
$$
\{A \sqsubseteq \bigsqcap_{\rho_1 \cdots\rho_n \in \Gamma} \exists
\rho_1.\exists \rho_2.\dots\exists \rho_n.\top, \quad
\bar r\sqsubseteq r\}.
$$

A truth assignment $I$ (for the variables $x_1 \dots,x_n$) also corresponds to a concept expression 
$$C_I :=\exists \rho_1^i.\exists \rho_2^i.\dots\exists \rho_n^i.\top,$$
where $\rho^i_j = r$ if $I$ makes $x_j$ true and $\rho^i_j=\bar r$
otherwise. Then
\begin{center}
$I\models \phi$ if, and only if, $\Tmc_\phi\models A\sqsubseteq C_I$
\end{center}
holds
for all truth assignments~$I$.

Note that $\bar r$ represents that a variable is false and $r$ that a
variable is true. Thus, the RI $\bar r \sqsubseteq r$ captures
the monotonicity of the DNF formulas considered. 
For any fixed values $n$, $s$ and $t$, we set 
$$
T(n,t,s)= \{\Tmc_{\phi} \mid \phi\in M(n,t,s)\}.
$$
Note that the TBoxes in $T(n,t,s)$ are exactly those TBoxes that have a DNF-representation for $n$ and
satisfy additionally the conditions that the DNF represented by $\Tmc_{\phi}$ has exactly $t$ conjunctions
each conjunction of which has exactly $s$ variables. 

\bigskip

We describe now the strategy for the oracle $\Omf'$ to answer equivalence queries so that 
no learning algorithm is able to exactly identify members of $T(n,t,s)$ based on the
answers to polynomially many equivalence queries of polynomial size.
If the TBox in the equivalence query is `obviously' not within
the class $T(n,t,s)$, then we will explicitly produce a counterexample that the oracle can
return.  If, on the other hand, the TBox $\Hmc$ from the equivalence query is
`similar' to TBoxes that have a DNF-representation for $n$, 
then we approximate $\Hmc$ by a TBox $\Hmc'$ that has a DNF-representation for $n$
and return the counterexample $A \sqsubseteq C_{I}$ corresponding to the truth
assignment $I$ that the oracle $\Omf$ from Theorem~\ref{lem:angluin} would
return when given $\psi$. 

In detail the strategy is as follows. Assume $q$ is the given polynomial in
Theorem~\ref{lem:angluin} and that $t_{0}$, $s_{0}$ and the strategy of the
oracle $\Omf$ are chosen so that for sufficiently large $n$ no learning
algorithm for DNF formulas that asks at most $q(n)$ equivalence queries, each
bounded in size by $q(n)$, can distinguish all members of $M(n,t_{0},s_{0})$.
Choose a sufficiently large $n$.  Let $\Hmc$ be an equivalence TBox query
issued by a learning algorithm. Then $\Omf'$ does the following:
\begin{enumerate}

\item If $\Hmc$ entails some $A\sqsubseteq \exists
\rho_1.\exists \rho_2.\dots\exists \rho_{n+1}.\top$ with $\rho_{i}\in \{r,\bar{r}\}$ for $1\leq i \leq n+1$, 
then return this CI as a negative counterexample;

\item If $\Hmc$ entails some $\exists \rho_{1}.\top \sqsubseteq \exists \rho_{2}.\top$ such that $\{\rho_{1},\rho_{2}\}
\subseteq \{r,\bar{r},r^-,\bar{r}^-\}$
and $\{\bar{r}\sqsubseteq r\} \not\models \exists \rho_{1}.\top \sqsubseteq \exists \rho_{2}.\top$,
then return this CI as a negative counterexample;

\item If $\Hmc\models \exists \rho_{1}.\top \sqsubseteq \exists \rho_{2}.\exists \rho_{3}.\top$ such that
$\{\rho_{1},\rho_{2},\rho_{3}\}\subseteq \{r,\bar{r}\}$, then return this CI as a negative counterexample;
 
\item If there exists no $\rho_1,\dots, \rho_n\in\{r,\bar r\}^n$ such that  
  $\Hmc \models A \sqsubseteq \exists \rho_1 . \cdots \exists \rho_n. \top$ 
  then return 
$
A \sqsubseteq \underbrace{\exists {r} \cdots \exists {r}}_{n}.\top
$
as a positive counterexample.

\item Suppose now that none of the above applies. We say that a sequence $\rho_1,\dots,\rho_n\in\{r,\bar r\}^n$
 is \emph{$r$-minimal for $\Hmc$} if $\Hmc\models A\sqsubseteq \exists \rho_1.\cdots \exists \rho_n.\top$ and
 whenever $\rho_i = r$, for $1\leq i\leq n$, we have 
 $\Hmc\not\models \exists \rho_1.\cdots \exists \rho_{i-1}.\exists\bar r.\exists \rho_{i+1}.\cdots\exists \rho_n.\top$.
 We obtain a TBox $\Hmc'$ with a DNF representation by setting
  $$
  \Hmc' = \{A \sqsubseteq \bigsqcap_{\substack{\rho_1,\dots,\rho_n \textrm{ is} \\ \textrm{$r$-minimal for $\Hmc$}}} 
               \exists \rho_1.\cdots\exists \rho_n.\top, \quad \bar r\sqsubseteq r\}.
  $$
  Observe that for any sequence $\rho_1, \dots, \rho_n\in\{r,\bar r\}^n$ we have
  $\Hmc \models A\sqsubseteq \exists \rho_1.\cdots\exists \rho_n.\top$ if, and only if, 
  $\Hmc'\models A\sqsubseteq \exists \rho_1.\cdots\exists \rho_n.\top$.
  We convert $\Hmc'$ into its corresponding monotone DNF formula 
  $\phi_{\Hmc'}$ by reversing the translation from monotone DNF formulas into 
  $\ourDLLite$ TBoxes of the above form in the obvious way. 
  Note that the size of $\phi_{\Hmc'}$ is linear in the size of $\Hmc'$.
  Given $\phi_{\Hmc'}$ the oracle $\Omf$ returns a (positive or negative) counterexample (a truth assignment)
  $I$. Then return the counterexample in the form of the CI $A \sqsubseteq C_I$.

\end{enumerate}
Observe that the answers given in Points~1 to 3 are correct in the sense that if an inclusion
$\alpha$ is returned as a negative example then $\Tmc\not\models \alpha$ for any $\Tmc\in T(n,t,s)$.
Point~4 is trivially correct, since any monotone DNF is satisfied by the truth assignment that makes every
variable true. We analyse the size of the TBox $\Hmc'$ computed in Point~5.
\begin{lemma}\label{lem:sizest}
Assume that Points~1 to 4 do not apply to $\Hmc$. Then the number of sequences $\rho_1,\dots,\rho_n\in\{r,\bar r\}^n$
which are $r$-minimal for $\Hmc$ is bounded by $|\Hmc|$.
\end{lemma}
\begin{proof}
We first show that if $\rho_1,\dots,\rho_n\in\{r,\bar r\}^n$ is $r$-minimal for $\Hmc$,
then there exists a CI $A \sqsubseteq C\in \Hmc$ such that 
\begin{itemize}
\item[$(\ast)$] there are concept expressions $C_{0},\ldots,C_{n}$ with $C_{0}=C$ and
$\exists \rho_{i+1}.C_{i+1}$ a top-level conjunct of $C_{i}$, for all $i<n$.
\end{itemize}
For the proof we require the canonical model $\Imc_{A,\Hmc}$ of $A$ and $\Hmc$ (Lemma~\ref{lem:can1}).
Denote the root of $\Imc_{A,\Hmc}$ by $\rho_{A}$. 
Let $\rho_1,\dots,\rho_n\in\{r,\bar r\}^n$ be $r$-minimal for $\Hmc$.
Then there are $d_{0},\ldots,d_{n}\in \Delta^{\Imc_{A,\Hmc}}$ with $d_{0}=\rho_{A}$ such that
$(d_{i},d_{i+1})\in \rho_{i}^{\Imc_{A,\Hmc}}$ for all $i<n$. 
By the canonical model construction and the assumption that Points~2 and 3 do not hold, there either 
exists $d_{i}\in A^{\Imc_{A,\Hmc}}$ or there is a CI $A \sqsubseteq C\in \Hmc$ such 
that $(\ast)$ holds. We show that the first condition does not hold. Assume for a prove by contradiction 
that $d_{i}\in A^{\Imc_{A,\Hmc}}$. By Lemma~\ref{lem:can1},
$\Hmc\models A \sqsubseteq \exists \rho_{1}\cdots \exists \rho_{i}.A$.
But then $\Hmc\models A \sqsubseteq \exists (\rho_{1}\cdots \rho_{i})^{n}.\top$ for all $n>0$ which contradicts 
the assumption that Point~1 does not apply to $\Hmc$. 

It follows that the number of distinct $r$-minimal sequences is bounded by the number of distinct sequences 
$C_{0},\ldots,C_{n}$ with $A \sqsubseteq C_{0}\in \Hmc$ and $\exists \rho_{i+1}.C_{i+1}$ a top-level conjunct of 
$C_{i}$ for all $i<n$. Thus, the number of distinct $r$-minimal sequences is bounded by $|\Hmc|$.
\end{proof}
It follows from Lemma~\ref{lem:sizest} that the size of the TBox $\Hmc'$ computed in Point~5 is bounded by
$4n|\Hmc|+2$.

\begin{theorem}\label{th:DLLiteEQ}
DL-Lite$_{\mathcal{R}}^{\exists}$ TBoxes (even without inverse roles) are not polynomial query
learnable using only equivalence queries.
\end{theorem}
\begin{proof}
  Suppose that the query complexity of a learning algorithm $\Amf$ for DL-Lite$_{\mathcal{R}}^{\exists}$ TBoxes in
  $\Sigma=\{A,r,\bar{r}\}$  
  is bounded at every stage of computation by a polynomial $p(x,y)$, where $x$ is the size of
  the target TBox, and $y$ is the maximal size of a counterexample returned by the oracle 
  up to the current stage of computation.
  Let $q(n) = (p(n^2, 4n+6))^2$, and let constants $t_0$ and $s_0$ be as guaranteed by Lemma~\ref{lem:angluin}.
  We claim that, for sufficiently large $n$, $\Amf$ cannot distinguish some $\Tmc_\phi$ and $\Tmc_\psi$ for 
  $\phi,\psi\in M(n,t_0,s_0)$.
  
  Assuming that $n>11$ (the maximal size of counterexamples given under Point~2 and 3), 
  the largest counterexample returned by our strategy described above is of the form
  $A\sqsubseteq \exists \rho_1.\cdots\exists \rho_{n+1}.\top$, so for
  sufficiently large $n$ the maximal size of any counterexample in any run of
  $\Amf$ is bounded by $4n+6=4(n+1)+2$.
  Similarly, the size of every potential target TBox $\Tmc_\phi\in T(n,t_0,s_0)$ does not exceed $t_0\cdot(4n+2)$
  and, as $t_0$ is a constant, for sufficiently large $n$ it is bounded by 
  $n^{2}$. Thus, for sufficiently large $n$ the total query complexity of $\Amf$
  on any input from $T(n, t_0, s_0)$ is bounded by $p(n^{2},4n+6)$.
  Obviously, the size of each query is bounded by the query complexity of the learning algorithm.
  So, the size of a DNF equivalence query forwarded to the strategy $\Omf$ guaranteed
  by Lemma~\ref{lem:angluin} is bounded by $4n\times p(n^{2},4n+6)+2 \leq q(n)$, 
  and there will be at most $q(n)$ queries forwarded. But then $\Omf$ can return answers such that some $\phi$ and 
  $\psi$ from $M(n,t_0,s_0)$ cannot be distinguished. It remains to observe that $\Amf$ cannot distinguish 
  $\Tmc_\phi$ and $\Tmc_\psi$.
\end{proof}

\section{Related Work}\label{sect:related}
Some related work has already been discussed in the introduction to
this paper. Here we discuss in more detail related work from
\emph{ontology learning} in general and \emph{exact learning of
  ontologies} in particular. We start with the former. %

\paragraph{Ontology Learning.} Research in ontology learning has a
rich history that we cannot discuss here in full detail. The
collection~\citep{lehmann2014perspectives} and
surveys~\citep{DBLP:reference/nlp/CimianoVB10,Wong:2012:OLT:2333112.2333115}
provide an excellent introduction to the state of the art in this
field. The techniques applied in ontology learning range from
information extraction and text mining to interactive learning and
inductive logic programming (ILP). Of particular relevance for this
paper are the approaches to learning logical expressions (rather than
subsumption hierarchies between concept names). For example, the work
in~\citep{lehmann2009ideal,DBLP:journals/ml/LehmannH10,DBLP:conf/ekaw/BuhmannFLMV14}
applies techniques from ILP to learn description logic concept
expressions. ILP is applied as well
in~\citep{DBLP:journals/ijswis/Lisi11} for learning logical rules for
ontologies. The learning of fuzzy DLs has been considered
in~\citep{DBLP:journals/fuin/LisiS15}.  Other machine learning methods
which have been applied to learn %
ontology axioms include Association Rule Mining
(ARM)~\citep{volker2011statistical,DBLP:conf/otm/FleischhackerVS12,DBLP:journals/ws/VolkerFS15}
and Formal Concept Analysis
(FCA)~\citep{DBLP:conf/iccs/Rudolph04,DBLP:conf/ijcai/BaaderGSS07,DBLP:phd/de/Distel2011,borchmann2014learning,ganter2016conceptual}.
Recently, learnability of lightweight DL TBoxes from finite sets of interpretations has been investigated in \citep{klarmanontology}.

\paragraph{Exact Learning of Description Logic Concept Expressions.}
Rather than aiming to learn a TBox here one is interested in learning a
target concept expression~$C_{\ast}$. This was first studied in
\citep{Cohen1994learnability,Cohen94learningthe,DBLP:journals/ml/FrazierP96}. The
standard learning protocol is as follows:
\begin{itemize}
\item a membership query asks whether a concept expression $C$ is subsumed by the target concept expression $C_{\ast}$ 
(in symbols, $\emptyset \models C\sqsubseteq C_{\ast}$?);
\item an equivalence query asks whether a concept expression $C$ is equivalent to the target concept expression $C_\ast$ (in symbols,
$\emptyset \models C \equiv C_{\ast}$?). If $C$ and $C_\ast$ are not equivalent then the oracle gives a counterexample, that is, 
a concept expression $C'$ such that either $\emptyset \models C' \sqsubseteq C_{\ast}$ and $\emptyset\not\models C'\sqsubseteq C$
or $\emptyset \not\models C' \sqsubseteq C_{\ast}$ and $\emptyset\models C'\sqsubseteq C$. 
\end{itemize}
\citep{Cohen1994learnability,Cohen94learningthe,DBLP:journals/ml/FrazierP96}
consider concept expressions in (variations of) the now largely historic
description logic \Classic
\citep{DBLP:conf/sigmod/BorgidaBMR89,DBLP:journals/sigart/Patel-SchneiderMB91,borgida1994semantics}.
The expressive power of \Classic and its variants is incomparable to
the expressive power of modern lightweight description logics.
\Classic only shares conjunction and unqualified existential
restrictions of the form $\exists r.\top$ with the DLs considered in
this paper. It additionally admits \emph{value restrictions}
$\forall r.C$ whose interpretation is given as
$$
(\forall r.C)^{\Imc}= \{ d \in \Delta^{\Imc} | d'\in C^{\Imc} \mbox{ for all $d'$ with } (d,d')\in r^{\Imc}\}
$$
and \emph{unqualified number restrictions} $(\leq n\; r)$ and $(\geq n \;r)$ interpreted as 
\begin{eqnarray*}
(\leq n r)^{\Imc} & = & \{ d \in \Delta^{\Imc} \mid \,| \{d' \mid (d,d')\in r^{\Imc}\}|\leq n\}\\
(\geq n r)^{\Imc} & = & \{ d \in \Delta^{\Imc} \mid \,|  \{d' \mid (d,d')\in r^{\Imc}\}|\geq n\}
\end{eqnarray*} 
as well as various constructors using individual names. For example, if $a_{1},\ldots, a_{n}$ are names for individual objects,
then ONE-OF$(a_{1},\ldots,a_{n})$ is a \Classic concept denoting the set $\{a_{1}^{\Imc},\ldots,a_{n}^{\Imc}\}$, 
where $a_{i}^{\Imc}$ denotes the individual with name $a_{i}$ in interpretation $\Imc$. It is proved 
in~\citep{Cohen1994learnability,Cohen94learningthe,DBLP:journals/ml/FrazierP96} that 
in many fragments of \Classic concept expressions cannot be learned polynomially using only membership or equivalence queries 
but that they can be learned in polynomial time using both. Exact learning of concept expressions 
in modern lightweight description logics has not yet been investigated.   

\paragraph{Exact Learning of TBoxes using Concept Inclusions as Queries.} First results on exact learning of description logic TBoxes using 
concept inclusions as queries were presented in~\citep{DBLP:conf/dlog/KonevLW13,DBLP:conf/kr/KonevLOW14}, of which this paper is 
an extension. In contrast to~\citep{DBLP:conf/dlog/KonevLW13,DBLP:conf/kr/KonevLOW14}, we make the distinction between 
polynomial time and polynomial query learnability which enables us to
formulate and prove results on a more fine grained level.
TBoxes in \ourDLLitehorn, for which we prove polynomial query
learnability, were not considered
in~\citep{DBLP:conf/dlog/KonevLW13,DBLP:conf/kr/KonevLOW14}. The current paper is also closely related to the PhD thesis of the third author~\citep{AOzaki}.
In addition to the results presented here, it is shown there that
even in the extension of $\mathcal{EL}_{\sf lhs}$ with role inclusions,
TBoxes 
can be learned in polynomial time. The learning algorithm is a non-trivial extension of the algorithm presented here for
$\mathcal{EL}_{\sf lhs}$ TBoxes.

\paragraph{Exact Learning of TBoxes using Certain Answers.} 
In recent years, data access mediated by ontologies has become one of
the most important applications of DLs,
see~\citep{PLCD*08,DBLP:journals/tods/BienvenuCLW14,DBLP:conf/rweb/KontchakovZ14,DBLP:conf/rweb/BienvenuO15}
and references therein. The idea is to use a TBox to specify semantics
and background knowledge for the data and use it for deriving more
complete answers to queries over the data. In this context, the data
is stored in an \emph{ABox} consisting of a finite set of assertions
of the form $A(a)$ or $r(a,b)$, where $A$ is a concept names, $r$ a
role name, and $a,b$ are individual names. Given a query $q(\vec{x})$
(typically a conjunctive query), a TBox $\Tmc$, and an ABox $\Amc$, a
tuple of individual names $\vec{a}$ from $\Amc$ and of the same length
as $\vec{x}$ is called a certain answer to $q(\vec{x})$ over $\Amc$
w.r.t.~$\Tmc$, in symbols $\Tmc,\Amc\models q(\vec{a})$, if every
model $\Imc$ of $\Tmc$ and $\Amc$ satisfies $q(\vec{a})$.  Motivated
by this setup, \citep{DBLP:conf/aaai/KonevOW16,AOzaki} study
polynomial learnability of TBoxes using membership queries that ask
whether a tuple of individuals names is a certain answer to a query
over an ABox w.r.t.\ the target TBox.  This is a natural
alternative to learning using concept inclusions since domain experts
are often more familiar with querying data in a particular domain than
with the logical notion of subsumption between concept expressions. In
detail, the learning protocol is as follows:
\begin{itemize}
\item a membership query takes the form $(\Amc,q(\vec{a}))$ and asks whether the tuple $\vec{a}$ of individual names is a certain answer 
to the query $q(\vec{x})$ over the ABox \Amc w.r.t.~the target TBox~\Tmc; 
\item an equivalence query asks whether a TBox \Hmc is equivalent to the target TBox $\Tmc$. 
If \Tmc and \Hmc are not equivalent then a counterexample of the form $(\Amc,q(\vec{a}))$ is given such that $\Tmc,\Amc\models q(\vec{a})$
and $\Hmc,\Amc\not\models q(\vec{a})$ (a positive counterexample) or $\Tmc,\Amc\not\models q(\vec{a})$
and $\Hmc,\Amc\models q(\vec{a})$ (a negative counterexample). 
\end{itemize}
In the learning protocol above we have not yet specified the class of queries from which the $q(\vec{x})$ are drawn and 
which strongly influences the classes of TBoxes that can be learned. In the context of data access using TBoxes the two 
most popular classes of queries are:
\begin{itemize}
\item conjunctive queries (CQs), that is, existentially quantified conjunctions of atoms; and 
\item instance queries (IQs), which take the form $C(x)$ or $r(x,y)$ with $C$ a concept expression from the DL under consideration and $r$ a 
role name. 
\end{itemize}
In~\citep{DBLP:conf/aaai/KonevOW16,AOzaki}, exact learning of TBoxes in the languages $\mathcal{EL}$, 
$\mathcal{EL}_{\sf lhs}$ and \ourDLLite 
is studied for both IQs and CQs in queries. The positive learnability results are proved by polynomial reductions to the 
learnability results presented in this paper and~\citep{AOzaki}. The basic link between learning using concept inclusions as 
queries and learning by certain answers is as follows: if \Tmc is a TBox and  
$C,D$ are concept expressions in any of the DLs discussed above then one can regard the labelled tree $T_{C}$ corresponding to $C$ 
as an ABox $\Amc_C$ with root $\rho_C$ and it holds that $\Tmc \models C\sqsubseteq D $ if, and only if, $\Tmc,\Amc_{C} \models D(\rho_C)$. 
The converse direction (obtaining a concept expression from an ABox) is more involved since 
ABoxes are not tree-shaped and an additional unfolding step is needed to compute a corresponding concept expression. Using this link 
it is proved in~\citep{DBLP:conf/aaai/KonevOW16,AOzaki} that \ourDLLite and $\mathcal{EL}_{\sf lhs}$ 
TBoxes with role inclusions can be learned 
with polynomially many queries using certain answers to IQs. 
It is also proved that \EL is still not learnable with polynomially many queries using certain answers 
with neither IQs nor CQs as the query language and that \ourDLLite TBoxes cannot be learned with 
polynomially many queries using certain answers with CQs as the query language. 

\paragraph{Exact Learning in (other) Fragments of FO Horn.} 
We discuss results on exact learning of finite sets of FO Horn clauses
or fragments of this logic, where a \emph{FO Horn clause} is a
universally quantified clause with at most one positive literal
\citep{page1993anti,arimura1997learning,reddy1998learning,DBLP:journals/iandc/AriasK02,DBLP:journals/jmlr/AriasKM07,DBLP:conf/pkdd/SelmanF11}.
Depending on what is used as membership queries and as counterexamples to
equivalence queries, one can distinguish between exact learning FO Horn clauses using interpretations and using entailments.
As learning using entailments is closer to our approach we focus on
that setting.
The exact learning protocol is
then as follows:
\begin{itemize}
\item a membership query asks whether an FO Horn clause is entailed by the target set $T$ of FO Horn clauses;  
\item an equivalence query asks whether a set $H$ of FO Horn clauses is equivalent to the target set $T$.
If $H$ and $T$ are not equivalent then a counterexample is given, that is, an FO Horn clause entailed by $T$ but not by $H$
(a positive counterexample) or vice versa. 
\end{itemize}   
Considering how terms (with function symbols allowed) 
can appear in an FO Horn clause, two main restrictions have been studied in the literature: 
\begin{enumerate}
\item %
\emph{Range restricted clauses}: when the set of terms in the positive literal (if existent) 
is a subset of the terms in the negative literals and their subterms; and 
\item %
\emph{Constrained clauses}: when the set of terms and subterms in the positive literal (if existent) 
is a superset of the terms in the negative literals. 
\end{enumerate}
For example, the FO Horn clause $\forall x(\neg P(f(x)) \vee P(x))$
is range restricted 
but not constrained %
and the FO Horn  clause $\forall x(\neg P(x) \vee P(f (x)))$  
is constrained  
but not %
range restricted, where $P$ is a predicate symbol and $f$ a function symbol.
In~\citep{reddy1998learning} and~\citep{arimura1997learning}, it is shown that under certain 
acyclicity conditions FO Horn with range restricted clauses and, respectively, constrained clauses
are polynomial time learnable from entailments if the arity of predicates is bounded by a constant.
A learning algorithm for a fragment of FO Horn (called closed FO Horn) that subsumes 
the two languages defined above is presented in~\citep{DBLP:journals/iandc/AriasK02}.  
The algorithm is polynomial in the number of clauses, terms and
predicates and the size of the counterexamples, but exponential not
only in the arity of predicates but also in the number of variables
per clause. In fact, it is an open question whether there exists a
learning algorithm for closed FO Horn that is polynomial in the number
of variables per clause.

We relate the learnability results for FO Horn to the learnability
results for lightweight description logics presented in this paper.
Observe that most DLs (and in particular all DLs investigated in this paper)  
can be translated into FO~\citep{baader2003description}.
For example, a translation of the $\EL_{\sf lhs}$ CI %
$\exists r.A \sqsubseteq B$ is $\forall x\forall y(\neg r(x,y)\vee\neg A(y)\vee B(x))$
and  a translation of the \ourDLLite CI $A \sqsubseteq \exists r.A$ 
is $\forall x(A(x) \rightarrow \exists y. (r(x,y) \wedge A(y)))$. 
Under this translation, every $\EL_{\sf lhs}$ TBox
can be regarded as a set of range restricted FO Horn clauses, where the arity of predicates is 
bounded by $2$.   
In contrast, since in \ourDLLite existential quantifiers can be nested in the right 
side of CIs, \ourDLLite CIs cannot be translated into FO Horn clauses.
We can now summarise the relationship between our learnability results for $\EL_{\sf lhs}$, \ourDLLite and \ourDLLitehorn 
and the results on exact learnability of FO Horn from entailments as follows: 
since the arity of DL predicates is at most~$2$ and since no function symbols are admitted in DLs, 
none of the DLs considered in this paper can express the fragments of FO Horn discussed above.
On the other hand, we do not impose an acyclicity condition on the TBoxes (in contrast to~\citep{reddy1998learning,arimura1997learning}) 
and our algorithms are polynomial in the number of variables permitted in any clause (in contrast to~\citep{DBLP:journals/iandc/AriasK02}). 
Thus, the results discussed above for FO Horn do not translate into polynomial learning algorithms for $\mathcal{EL}_{\sf lhs}$
and are not applicable to \ourDLLite nor \ourDLLitehorn.
Our results thus cover new fragments of FO that have not yet been considered for exact learning. This is not surprising,
given the fact that the fragments of FO considered previously were not motivated by applications in ontology learning.

Also related to exact learning of Horn FO is recent work on exact learning of schema mappings in data
exchange~\citep{DBLP:conf/icdt/CateDK12}. Schema mappings are tuples $(S,T,M)$ where $S$ is a source schema (a finite set of predicates),
$T$ is a target schema (a finite set of predicates), and $M$ is a
finite set of sentences of the form $\forall \vec{x} (\varphi(\vec{x}) \rightarrow \exists \vec{y}\psi(\vec{x},\vec{y}))$
where $\varphi(\vec{x})$ and $\psi(\vec{x},\vec{y})$ are conjunctions of atoms over $S$ and $T$, respectively~\citep{fagin2005data}. 
$(S,T,M)$ is a GAV schema mapping if $\vec{y}$ is empty and $\psi(\vec{x},\vec{y})$ is an atom. In \citep{DBLP:conf/icdt/CateDK12},
the authors study exact learnability of GAV schema mappings from data examples $(I,J)$ consisting of a database $I$ over the
source schema $S$ and a database $J$ over the target schema $T$. Such a data example satisfies $M$ if $I\cup J \models M$.
The authors present both polynomial query learnability results for protocols using membership and equivalence queries
and non-polynomial query learnability results if either only
membership or only equivalence queries are allowed. The results presented in~\citep{DBLP:conf/icdt/CateDK12} are not applicable
to the setting considered in this paper since the learning protocol uses data examples instead of entailments. 

\section{Conclusion}
We have presented the first study of learnability of DL ontologies in
Angluin et al's framework of exact learning, obtaining both positive
and negative results. 
Several research questions remain to be
explored. One immediate question is whether acyclic $\EL$ TBoxes can
be learned in polynomial time using queries and counterexamples of the
form $A \equiv C$ and $A \sqsubseteq C$ only. Note that our
non-polynomial query learnability result for acyclic $\EL$ TBoxes relies
heavily on counterexamples that are not of this form.  Another
immediate question is whether the extension of
$\mathcal{EL}_{\mn{lhs}}$ with inverse roles (which is a better
approximation of OWL2 RL than $\mathcal{EL}_{\mn{lhs}}$ itself) can
still be learned in polynomial time, or at least with polynomially
many queries of polynomial size. Other interesting research directions are non-polynomial
time learning algorithms for $\mathcal{EL}$ TBoxes and the admission
of different types of membership queries and counterexamples in the
learning protocol. For example, one could replace CIs as
counterexamples with interpretations.

\smallskip

\noindent
{\bf Acknowledgements} 
Lutz was supported by the DFG project Prob-DL (LU1417/1-1).
Konev and Wolter were supported by the EPSRC project EP/H043594/1.
Ozaki was supported by the Science without Borders scholarship programme.
\bibliography{ourbib}

\newpage

\appendix

\section{Proofs for Section~\ref{sec:EL_not_learnable}}
We supply proofs for Lemma~\ref{lem:dllite-non-learn} and Lemma~\ref{lem:MQ}. In addition, we prove
a claim used in the proof of Lemma~\ref{lem:EQ}. 
We start by giving the proof of Lemma~\ref{lem:dllite-non-learn}.

\medskip
\noindent
{\bf Lemma~\ref{lem:dllite-non-learn}}
{\it For every $\ourDLLite$ CI $B \sqsubseteq D$ over the signature of $\Tmc_{\sma}$,
\begin{itemize}
\item either $\Tmc_{\sma} \models B \sqsubseteq D$ for every $\Tmc_{\sma}\in S_N$
\item or there is at most one $\Tmc_{\sma}\in S_N$ such that $\Tmc_{\sma}\models B \sqsubseteq D$.
\end{itemize} 
}
\medskip
\noindent
\begin{proof}
Assume the CI $B \sqsubseteq D$ is given. If $B\not=A$ or $M$ does not occur in $D$, then the claim can be
readily checked. Thus, we assume that $B=A$ and $M$ occurs in $D$. 
Assume there exists $\sma_{0}$ such that $\Tmc_{\sma_{0}}\models A \sqsubseteq D$ (if no such $\sma_{0}$ exists,
we are done). For any $\sma$, let $\Imc_{A,\Tmc_{\sma}}$ be the canonical model of $A$ and $\Tmc_{\sma}$
(Lemma~\ref{lem:can1}). 
Apply the following restricted form of parent/child merging exhaustively to the concept expression $D$:
\begin{itemize}
\item if there are nodes $d,d_{1},d_{2}\in T_{D}$ with $l(d_{1},d)=\sigma$ and $l(d,d_{2})=\sigma^{-}$ for some $\sigma\in \{r,s\}$, 
then replace $D$ by the resulting concept expression after $d_{1}$ and $d_{2}$ are merged in $D$.
\end{itemize}
Let $D'$ be the resulting concept expression. Recall from Lemma~\ref{lem:can1} that 
$\Tmc_{\sma}\models A\sqsubseteq D$
iff there is a homomorphism from $T_{D}$ to $\Imc_{A,\Tmc_{\sma}}$ mapping $\rho_{D}$ to $\rho_{A}$.
Using the fact that $\Imc_{A,\Tmc_{\sma}}$ is a ditree interpretation, one can readily check that 
any homomorphism $h$ from $T_{D}$ to $\Imc_{A,\Tmc_{\sma}}$ mapping $\rho_{D}$ to $\rho_{A}$ factors through 
$T_{D'}$ and that $D'$ is an $\EL$ concept expression. Thus, if there is an additional
$\sma'\not=\sma_{0}$ such that $\Tmc_{\sma'} \models A \sqsubseteq D$, then there are two homomorphisms 
$h_{\sma_{0}}$ and $h_{\sma'}$
with the same domain $T_{D'}$ into $\Imc_{A,\Tmc_{\sma_{0}}}$ and $\Imc_{A,\Tmc_{\sma'}}$ and
mapping the root of $T_{D'}$ to the roots of $\Imc_{A,\Tmc_{\sma_{0}}}$ and $\Imc_{A,\Tmc_{\sma'}}$, respectively. 
Since $M$ occurs in $D'$ and $D'$ is an $\EL$ concept expression
we find a sequence $D_{0},\ldots,D_{m}$ with $D_{0}=D'$ and $D_{m}=M$ such that 
$\exists s_{i+1}.D_{i+1}$ is a top-level conjunct of $D_{i}$ for $s_{i}\in \{r,s\}$
and all $i<m$. But then $s_{1}\cdots s_{m}=\sma_{0}$ and $s_{1}\cdots s_{m}=\sma'$ and we have derived a 
contradiction to the assumption that $\sma_{0}$ and $\sma'$ are distinct.
\end{proof}
To prove Lemma~\ref{lem:MQ} we require the following observation.
\begin{lemma}\label{lem:rhs3}
For any acyclic $\mathcal{EL}$ TBox $\Tmc$, any CI $A\sqsubseteq C\in\Tmc$
and any concept expression of the form $\exists t. D$ we have
$\Tmc\models A\sqsubseteq \exists t.D$ if, and only if,
$\Tmc\models C\sqsubseteq \exists t.D$.
\end{lemma}

We are now ready to prove Lemma~\ref{lem:MQ}.

\smallskip
\noindent\textbf{Lemma~\ref{lem:MQ}\ }
{\it
For every $\mathcal{EL}$ CI $C\sqsubseteq D$ over~$\Sigma_{n}$:
\begin{itemize}
\item either $\Tmc_L\models C\sqsubseteq D$ for every $L\in\Lmf_n$ 
\item or the number of $L\in\Lmf_n$ such that $\Tmc_L\models C\sqsubseteq D$
does not exceed $|C|$.
\end{itemize}
}

\smallskip
\noindent\begin{proof}
We prove the lemma by induction on the structure of $D$.
We assume throughout the proof that there exists some $L_0\in\Lmf_n$ such that
$\Tmc_{L_0}\models C\sqsubseteq D$.

\smallskip

\noindent \textsl{Base case}: $D$ is a concept name.
We make the following case distinction.
\begin{itemize}
\item 
$D\in \{X_i, A_i, B_i \mid 1\leq i \leq n\}$ or $D=M$. 
By Lemma~\ref{basic}, $C$ is of the
form $Z\sqcap C'$, for some concept name $Z$, and $\Tmc_{L_0}\models Z\sqsubseteq D$.
But then $Z=D$ and it follows that $\Tmc_{L}\models C\sqsubseteq D$ for every $L\in\Lmf_n$.

\item $D=X_0$. By Lemma~\ref{basic}, $C$ is of the
form $Z\sqcap C'$, for some concept name $Z$, and $\Tmc_{L_0}\models Z\sqsubseteq X_0$.
This is the case if either $Z=X_{0}$, or $Z\in \{A,A_1,B_1,\dots,A_n,B_n\}$.
In either case, $\Tmc_{L}\models C\sqsubseteq X_0$ for every $L\in\Lmf_n$.

\item 
$D = A$. If $C$ is of the form $A\sqcap C'$ or for all $i$ such that $1\leq i \leq n$,
$A_i$ or $B_i$ is a conjunct of $C$, then $\Tmc_L\models C\sqsubseteq A$ for
every $L\in\Lmf_n$. 
Assume now that $C$ is not of this form. Then for some $j$ such that $1\leq j \leq n$, 
$C$ is neither of the form $A\sqcap C'$ nor of the form $A_j\sqcap C'$ nor of the form $B_j\sqcap C'$.
Let $L = (\sma_1,\dots,\sma_n)\in\Lmf_n$ be such that  
$\Tmc_L\models C\sqsubseteq A$.  
Notice that $\Tmc_L\models C\sqsubseteq A$, for $L=(\sma_1,\dots,\sma_n)\in\Lmf_n$, if, and only if, 
$\Tmc_L\models C\sqsubseteq X_0\sqcap\exists \sma_1.M\sqcap\dots\sqcap\exists \sma_n.M$.
By the claim in the proof of Lemma~\ref{lem:component}, for such a $\Tmc_L$ we must have
$\emptyset\models C\sqsubseteq \exists {\sma_j}.M$. 
Clearly, the number of $L=(\sma_1,\dots,\sma_n)\in\Lmf_n$
with $\emptyset\models C\sqsubseteq \exists\sma_j.M$ does not exceed $|C|$.

Thus, either $\Tmc_L\models C\sqsubseteq A$ for every $L\in\Lmf_n$ or the number of 
$L\in\Lmf_n$ such that $\Tmc_L\models C\sqsubseteq A$ does
not exceed $|C|$.
\end{itemize}

\smallskip

\noindent \textsl{Induction step}. 
If $D = D_1\sqcap D_2$, then $\Tmc_L\models C\sqsubseteq D$ if, and only if, $\Tmc\models
C\sqsubseteq D_i$, $i=1,2$.
By induction hypothesis, for $i=1,2$
either $\Tmc_L\models C\sqsubseteq D_i$ for every $L\in\Lmf_n$ or
there exist at most $|C|$ different $L\in\Lmf_n$
such that $\Tmc_L\models C\sqsubseteq D_i$. Thus
either $\Tmc_L\models C\sqsubseteq D$ for every $L\in\Lmf_n$ or
the number of $L\in\Lmf_n$ 
such that $\Tmc_L\models C\sqsubseteq D$ also does not exceed $|C|$.

\medskip

\noindent 
Now assume that $D = \exists t.D'$.
Suppose that $\Tmc_L\models C\sqsubseteq D$ for some $L\in\Lmf_n$.
Then, by Lemma~\ref{basic}, either there exists a
conjunct $Z$ of $C$, $Z$ a concept name, such that 
$\Tmc_{L} \models Z\sqsubseteq \exists t.D'$ or
there exists a conjunct $\exists t.C'$ of $C$ with
$\Tmc_{L} \models C' \sqsubseteq D'$. 
We analyse for every conjunct of $C$ of the form $Z$ or $\exists t.C'$ the number of
$L\in\Lmf_n$ such that $\Tmc_L\models Z\sqsubseteq \exists t.D'$
or $\Tmc_L\models \exists t.C'\sqsubseteq \exists t.D'$, respectively.
\begin{itemize}
\item[(i)] Let $Z$ be a conjunct of $C$ such that 
$Z$ is a concept name and $\Tmc_{L} \models Z\sqsubseteq \exists t.D'$. 
Notice that $Z\not=M$ as there is no $L\in\Lmf_n$ such that
$\Tmc_L\models M\sqsubseteq \exists t.D'$.
We consider the remaining cases.
\begin{itemize}
\item $Z = X_i$, for some $i\geq0$.
It is easy to see that for $L,L'\in\Lmf_n$ we have 
$\Tmc_L\models X_i\sqsubseteq \exists t.D'$
if, and only if
$\Tmc_{L'}\models X_i\sqsubseteq \exists t.D'$.
Thus, $\Tmc_L\models Z\sqsubseteq \exists t.D'$ for every $L\in\Lmf_n$.
\item $Z\in \{A_i, B_i \mid 1\leq i \leq n\}$.
By Lemma~\ref{lem:rhs3}, 
$\Tmc_{L}\models Z\sqsubseteq \exists t.D'$ if, and only if,
$\Tmc_{L}\models X_0\sqcap \exists \sma_i.M\sqsubseteq \exists t.D'$. By Lemma~\ref{basic},
either $\Tmc_{L}\models X_0\sqsubseteq \exists t.D'$
or $\Tmc_{L}\models \exists \sma_i.M\sqsubseteq \exists t.D'$.
If $\Tmc_{L}\models X_0\sqsubseteq \exists t.D'$ then, as above,
for $\Tmc_{L}\models C\sqsubseteq \exists t.D'$ every $L\in\Lmf_n$. 
Suppose that $\exists t.D'$ is such that 
$\Tmc_{L}\not \models X_0\sqsubseteq \exists t.D'$ and
$\Tmc_{L}\models \exists \sma_i.M \sqsubseteq \exists t.D'$.
By inductive applications of Lemma~\ref{basic},
this is only possible if $\exists t.D'  = \exists \sma_i.M$.
Thus, there is exactly one $L\in\Lmf_n$ (namely, $L= L_0$) such that 
$\Tmc_L\models Z\sqsubseteq \exists \sma_i.M$.
\item $Z = A$. Suppose that for some $L=(\sma_1,\dots,\sma_n)\in\Lmf_n$ we have
$\Tmc_{L}\models A \sqsubseteq \exists t.D'$. Equivalently, 
$\Tmc_{L}\models X_0\sqcap \exists \sma_1.M\sqcap\dots\sma_n.M\sqsubseteq \exists t.D'$. 
By Lemma~\ref{basic}, 
either $\Tmc_{L}\models X_0\sqsubseteq \exists t.D'$
or $\Tmc_{L}\models \exists \sma_i.M\sqsubseteq \exists t.D'$ for some $i$ with $1\leq i \leq n$.
Thus, as above, unless $\Tmc_L\models X_0\sqsubseteq \exists t.D'$ we have
$\exists t.D'$ is $\exists \sma_i.M$. But then $L=L_0$.
\end{itemize}

\item[(ii)] Let $\exists t.C'$ be a conjunct of $C$ with
$\Tmc_{L} \models C' \sqsubseteq D'$. 
The induction hypothesis implies that the number of $L\in\Lmf_n$ such that $\Tmc_L\models C'\sqsubseteq D'$
does not exceed $|C'|$.
\end{itemize}
To summarise,
either $\Tmc_L\models C\sqsubseteq \exists t.D'$ for every $L\in\Lmf_n$ or
for every conjunct $C_0$ of $C$ of the form $Z$ or $\exists t.C'$,
the number of $L\in\Lmf_n$ such that $\Tmc_L\models C_0\sqsubseteq \exists t.D'$
does not exceed $|C_0|$. Hence
the number of $L\in\Lmf_n$ such that $\Tmc_L\models C\sqsubseteq \exists t.D'$
does not exceed $|C|$.
\end{proof}
The next result is used in the proof of Lemma~\ref{lem:EQ}.
\begin{lemma}\label{lem:claim}
For any $0\leq i\leq n$ and concept expression $D$ over ${\Sigma_n}$, if
$\Tmc_0\not\models X_i\sqsubseteq D$ 
then there exists a sequence of role names $t_1,\dots t_l\in \{r,s\}^{l}$ such that  
$\emptyset\models D\sqsubseteq \exists t_1.\cdots\exists t_l.Y$ and 
$\Tmc_0\not\models X_i\sqsubseteq \exists t_1.\cdots\exists t_l.Y$,
where $Y$ is either $\top$ or a concept name and $0\leq l\leq n-i+1$.
\end{lemma}
\begin{proof} We prove the lemma by induction on $i$ from $i=n$ to $i=0$.
If $i=n$, then $\Tmc_0\not\models X_i\sqsubseteq D$ if either $\emptyset\models
D\sqsubseteq \exists t.\top$, for some role name $t$, or
$\emptyset\models D\sqsubseteq Y$, for some concept name $Y\neq X_i$.

Suppose that the lemma is proved for $0<j\leq n$ and let $i=j-1$.
We proceed by induction on the structure of $D$. If $D$ is 
a concept name, we are done as $\Tmc_0\models X_i\sqsubseteq Z$ does not hold for any concept name $Z\neq X_i$.
If $D$ is of the form $\exists t.D'$, where $t\in \{r,s\}$, then 
$\Tmc_0\not\models X_{i+1}\sqsubseteq D'$, and so, by induction hypothesis,
there exists a sequence of role names $t_1,\dots, t_l$, with $l\leq n-i$, such that
$\Tmc_0\not\models X_{i+1}\sqsubseteq \exists t_1.\cdots\exists t_l.Y$ and 
$\emptyset\models D'\sqsubseteq \exists t_1.\cdots\exists t_l.Y$. 
But then, by Lemma~\ref{lem:rhs3} and Lemma~\ref{basic}, 
$\Tmc_0\not\models X_{i}\sqsubseteq \exists t.\exists t_1.\cdots\exists t_l.Y$ and
$\emptyset\models \exists t.D'\sqsubseteq \exists t.\exists t_1.\cdots\exists t_l.Y$.
If $D$ is of the form $D=D_1\sqcap D_2$, there there exists $D_{i}$, $i=1,2$,
such that $\Tmc_0\not\models X_i\sqsubseteq D_i$ and the lemma holds
by induction hypothesis.
\end{proof}

\medskip

\end{document}